%% file: example_paper.tex
\documentclass{article}

\usepackage{microtype}
\usepackage{graphicx}
\usepackage{subfigure}
\usepackage{booktabs} 

\usepackage{breakurl}

\usepackage{hyperref}

\usepackage[accepted]{icml2025}

\usepackage{amsmath}
\usepackage{amssymb}
\usepackage{mathtools}
\usepackage{amsthm}
\usepackage{dsfont}

\usepackage[capitalize,noabbrev]{cleveref}

\DeclareMathOperator{\Hom}{Hom}

\usepackage{mathdots}
\usepackage{bm}
\usepackage{nicematrix}
\usepackage{xcolor}

\usepackage{tabularray}
\usepackage{comment}

\usepackage[title]{appendix}
\usepackage{appendix}
\usepackage{xurl}

\usepackage{ytableau}
\usepackage{tikzit}
\input{mystyle.tikzstyles}
\usetikzlibrary{decorations.pathreplacing}
\usepackage{graphicx}
\usepackage[most]{tcolorbox}
\usepackage{enumitem}

\definecolor{melon}{RGB}{227, 168, 105} 
\definecolor{peach}{RGB}{255, 218, 185}
\definecolor{babypink}{RGB}{244, 194, 194}
\definecolor{lplum}{RGB}{191, 128, 191}
\definecolor{sblue}{RGB}{135, 206, 250}
\definecolor{ygreen}{RGB}{154, 205, 50}
\definecolor{lpurple}{RGB}{210, 180, 222}
\definecolor{terracotta}{RGB}{204, 78, 63}
\definecolor{lorange}{RGB}{255, 209, 153}
\definecolor{sage}{RGB}{188, 184, 138}   
\definecolor{dustyblue}{RGB}{96, 130, 182}

\theoremstyle{plain}
\newtheorem{theorem}{Theorem}[section]
\newtheorem{proposition}[theorem]{Proposition}
\newtheorem{lemma}[theorem]{Lemma}
\newtheorem{corollary}[theorem]{Corollary}
\theoremstyle{definition}
\newtheorem{definition}[theorem]{Definition}

\theoremstyle{remark}
\newtheorem{remark}[theorem]{Remark}
\newtheorem{example}[theorem]{Example}

\usepackage[textsize=tiny]{todonotes}

\icmltitlerunning{Permutation Equivariant Neural Networks for Symmetric Tensors}

\begin{document}

\twocolumn[
\icmltitle{
	Permutation Equivariant Neural Networks for Symmetric Tensors
	}




\begin{icmlauthorlist}
\icmlauthor{Edward Pearce--Crump}{yyy}
\end{icmlauthorlist}

\icmlaffiliation{yyy}{Department of Mathematics, Imperial College London, United Kingdom}

\icmlcorrespondingauthor{Edward Pearce--Crump}{ep1011@ic.ac.uk}

\icmlkeywords{Machine Learning, ICML}

\vskip 0.3in
]



\printAffiliationsAndNotice{}  

\begin{abstract}
	Incorporating permutation equivariance into neural networks 
	has proven to be useful 
	in ensuring that models respect 
	symmetries 
	that exist 
	in data.
	Symmetric tensors, which 
	naturally 
	appear
	in statistics, machine learning, and graph theory,
	are essential for many
	applications in
	physics, chemistry, and materials science,
	amongst others.
	However,
	existing research on permutation equivariant models
	has not explored symmetric tensors as inputs,
	and most prior work on learning from these
	tensors 
	has focused on equivariance to Euclidean groups.
	In this paper, 
	we present two different characterisations of all linear 
	permutation equivariant functions between symmetric power spaces of $\mathbb{R}^n$.
	We show on two tasks that these functions 
	are highly data efficient compared to standard MLPs
	and 
	have potential
	to generalise well to symmetric tensors of different sizes.
\end{abstract}


\section{Introduction}


Equivariance has proven to be an effective approach
for encoding 
known structure
about a problem 
directly into a neural network's architecture, making learning more efficient
and interpretable
\citep{
cohenc16, cohen2017steerable,
qi2017, ravanbakhsh17a, zaheer2017, 
cohen2018, esteves2018, kondor18a, clebschgordan, thomas2018, weiler2018, maron2019, 
finzi2021, satorras21a, villar2021scalars,
yarotsky2022, pearcecrumpB}.
Equivariant neural networks have been applied successfully in a wide range of problems, 
including, but not restricted to,
particle physics \citep{bogatskiy2020, villar2021scalars, gong2022},
biological and medical imaging \citep{bekkers2018, lafarge2021, suk24},
molecular and quantum chemistry 
\citep{thomas2018, finzi20a, fuchs2020, miller2020, schutt21a, batatia2022, 
batzner2022, hoogeboom22a, liao2023},
reinforcement learning \citep{vanderpol, wang2022r, wang2022, wang22j},
computer vision \citep{marcos2017, worrall2017, esteves2019, deng2021, chatzipantazis2023, kaba23a}, and
modelling compositional structures in natural language \citep{gordon2020, petrache2024}.

In this paper, we consider linear permutation ($S_n$) equivariant functions $f$ on symmetric tensors 
$T \in (\mathbb{R}^{n})^{\otimes k}$.
In particular, these functions must satisfy
\begin{equation}
	f(e_{i_1} \otimes e_{i_2} \otimes \dots \otimes e_{i_k}) 
	= 
	f(e_{i_{\pi(1)}} \otimes e_{i_{\pi(2)}} \otimes \dots \otimes e_{i_{\pi(k)}}) 
\end{equation}
for all $\pi \in S_k$,
where $e_{i_j}$ are standard basis vectors in $\mathbb{R}^{n}$,
since the coefficients of $T$, when expressed in this basis,
satisfy
$T_{i_1, i_2, \dots, i_k} = T_{i_{\pi(1)}, i_{\pi(2)}, \dots, i_{\pi(k)}}$
for all $\pi \in S_k$.

Symmetric tensors and linear permutation equivariant functions on them play a fundamental role 
in many scientific domains owing to their ability to capture invariant and equivariant relationships 
in high-dimensional data. 
In statistics, these functions preserve the defining 
properties of covariance matrices, which are symmetric and equivariant under
permutations of their coordinates \citep{mccullagh2018}. 
These functions and tensors are also important in machine learning, 
where they enable robust parameter estimation in latent variable models 
\citep{anandkumar2014, jaffe2018, goulart2022}.
They also appear in clustering algorithms, which must treat all data points equivalently \citep{khouja2022}. 
However, the most notable example is in graph theory, where
symmetric adjacency matrices are formed from undirected graphs \citep{maron2019}
and high-order adjacency tensors are used to identify affinity relations in hypergraphs 
\citep{shashua2005, georgii2011, ghoshdastidar2017}.
This framework is essential for many real-world applications: for example,
in fluid dynamics simulations, particle interactions must be invariant under exchange \citep{gao2022};
in materials science, stress-strain relationships must maintain their physical meaning 
under permutation symmetry \citep{garanger2024, wen2024};
and in neuroscience, brain connectome mapping requires consistent representation 
of neural pathways \citep{guha2020}.

The range and nature of these applications highlight the importance of achieving 
a complete and exact characterisation of permutation equivariant linear functions 
applied to symmetric tensors.
However, most studies in the existing machine learning literature on symmetric tensors
focus only on equivariance to Euclidean groups,
and no prior work on permutation equivariant neural networks
has considered the situation where the only inputs are symmetric tensors.

To address this gap, we make the following contributions.
We obtain a complete and exact characterisation of all
permutation equivariant linear functions
between any two symmetric powers of $\mathbb{R}^{n}$, 
deriving this result using two different bases.
In particular, we show that, by embedding symmetric powers 
of 
$\mathbb{R}^{n}$ into tensor power space, our characterisation
includes all linear scalar-valued and vector-valued
permutation equivariant functions that are defined on symmetric tensors.
As a corollary, we also recover the Deep Sets characterisation \citep{zaheer2017}.
Crucially, we address the implementation challenges that are associated 
with storing large weight matrices in memory by introducing what we call 
\textit{map label notation}. 
This notation makes it possible for us to express the transformation 
of an input symmetric tensor by a permutation equivariant weight matrix 
as a series of equations in the input tensor, eliminating the need to 
store the weight matrices explicitly in memory.
In particular, 
we can then use the same set of equations 
to compute the transformation 
between symmetric tensors of fixed orders $k$ and $l$
for any dimension $n$, 
eliminating the need to generate separate weight matrices
for each choice of $n$.
We validate our approach on two toy problems 
and show that these functions exhibit high data efficiency and 
strong potential for generalising well to symmetric tensors of different sizes.






\section{Related Work} \label{relatedwork}

We describe two important classes of related works.

\textbf{Equivariant Machine Learning on Symmetric Tensors.}
Most prior research on learning from
symmetric tensors has only considered equivariance to Euclidean groups.
\citet{gao2022} applied tensor contractions to achieve $SO(n)$-equivariance 
on high-order symmetric tensors,
using their network to simulate fluid systems.
\citet{lou2024} and \citet{wen2024} each introduced $SO(3)$-equivariant graph neural networks 
to respectively discover highly anisotropic dielectric crystals and to
learn from elasticity tensors,
which are fourth-order symmetric tensors that fully characterise a material's elastic behavior.
\citet{heilman2024}
designed an $SE(3)$-equivariant graph neural network architecture 
to predict material property tensors directly from crystalline structures.
\citet{garanger2024}
constructed a neural network that is equivariant
to subgroups of the orthogonal group $O(n)$
in order to learn from symmetric $n \times n$ matrices,
and applied it to the constitutive modelling of materials with symmetries.

A few prior works have characterised equivariant functions on symmetric tensors, 
but none have provided an exact characterisation of permutation equivariant weight matrices 
for these tensors. 
\citet{kunisky2024} used graph moments to characterise $O(n)$-invariant polynomials on symmetric tensors, 
while \citet{blumsmith2024} introduced a set of $\binom{n+1}{2}+ 1$ functions 
that can universally approximate $S_n$-invariant functions 
on almost all real symmetric $n \times n$ matrices.

\textbf{Permutation Equivariant Neural Networks.}
These neural networks have been constructed for a number
of different types of data.
They include 
permutation equivariant 
neural networks to learn
from set-structured data \citep{qi2017, zaheer2017, sverdlov2024};
to learn mappings between different sized sets \citep{ravanbakhsh17a};
to learn the interaction between sets \citep{hartford2018};
to learn from graphs and hypergraphs 
\citep{maron2019, thiede2020, finzi2021, morris22a, godfrey2023, 
puny2023, pearcecrump, pearcecrumpG};
and to learn from sets of symmetric elements \citep{maron20a}.
\citet{pan22} studied how to organise the computations that are involved
in high-order tensor power networks.
Other works have studied the expressive power of permutation equivariant neural networks
\citep{keriven2019, maron19a, segol2019, yarotsky2022}.


\section{Symmetric Powers of $\mathbb{R}^{n}$}

We present the key definitions and concepts 
that are needed to study linear permutation
equivariant functions on symmetric tensors. 
We first introduce an action of $S_n$ on a set of $k$-length indices $S[n]^k$
which will then index
a basis of a real vector space, $S^k(\mathbb{R}^{n})$, 
known as the $k^{\text{th}}$ symmetric power of $\mathbb{R}^{n}$.
We let $[n] \coloneqq \{1, \dots, n\}$ throughout.

Define the set $S[n]^k$ to consist of all tuples of the form
\begin{equation}
	I 
	=
		(
		\underbrace{i_1, \dots, i_1}_{m_1}, 
		\underbrace{i_2, \dots, i_2}_{m_2}, 
		\dots,
		\underbrace{i_p, \dots, i_p}_{m_p}
		)
\end{equation}
such that
$p \in [n]$,
$1 \leq i_1 < i_2 < \ldots < i_p \leq n$, 
and $\sum m_i = k$.
By a ``stars and bars'' argument, 
we see that the cardinality of $S[n]^k$ is
$\binom{k+n-1}{k}$.

We obtain an action of the symmetric group $S_n$ on $S[n]^k$
by applying the permutation to each element in the tuple
and then reordering the result (via a permutation in $S_k$) 
to restore ascending numerical order.

\begin{example}
	Choosing the element $(1,1,1,2,3,3)$ of $S[3]^6$ and applying
	the permutation $(1 3 2)$ in $S_3$ to it, we first obtain
	$(3,3,3,1,2,2)$ and then reorder it to $(1,2,2,3,3,3)$
	to give another element of $S[3]^6$.
\end{example}

Consequently, we define the
\textbf{$\bm{k^\text{th}}$ symmetric power of $\mathbb{R}^{n}$},
$S^k(\mathbb{R}^{n})$,
to be the vector space
that has a \textbf{standard basis}
$\{e_I\}$ indexed by the elements $I$ of 
$S[n]^k$. 
Moreover, we can raise the action of $S_n$ on the set $S[n]^k$
to a representation of $S_n$ on $S^k(\mathbb{R}^{n})$ in the usual way:
we call the representation itself $\rho_{k}^{S}$.

We wish to characterise in full the linear
permutation equivariant functions between 
any two symmetric power spaces of $\mathbb{R}^{n}$.
They are defined as follows.
\begin{definition} \label{equivariance}
	A linear map $f: 
	S^k(\mathbb{R}^{n})
	\rightarrow
	S^l(\mathbb{R}^{n})$
	is said to be \textbf{permutation equivariant} if,
	for all $\sigma \in S_n$ and $v \in 
	S^k(\mathbb{R}^{n})$
	\begin{equation} \label{equivmapdefn}
		f(\rho_{k}^{S}(\sigma)[v]) = \rho_{l}^{S}(\sigma)[f(v)]
	\end{equation}
	We denote the vector space of all such maps by
	$\Hom_{S_n}(S^k(\mathbb{R}^{n}),S^l(\mathbb{R}^{n}))$.
\end{definition}

Since our goal is to calculate all of the permutation equivariant weight matrices
that can appear between any two symmetric power spaces of $\mathbb{R}^{n}$,
it is enough to construct a basis of matrices for 
$\Hom_{S_n}(S^k(\mathbb{R}^{n}), S^l(\mathbb{R}^{n}))$,
by viewing it as a subspace of 
$\Hom(S^k(\mathbb{R}^{n}), S^l(\mathbb{R}^{n}))$
and choosing the standard basis for each symmetric power space,
since any weight matrix will be a weighted linear combination of these basis matrices.
This characterisation includes the following functions: 

\begin{proposition} \label{linearsymmchar}
	Linear permutation equivariant scalar-valued and vector-valued functions 
	on symmetric tensors in $(\mathbb{R}^{n})^{\otimes k}$ 
	are elements of $\Hom_{S_n}(S^k(\mathbb{R}^{n}),S^0(\mathbb{R}^{n}))$
	and
	$\Hom_{S_n}(S^k(\mathbb{R}^{n}),S^1(\mathbb{R}^{n}))$, respectively.
\end{proposition}











\section{Orbit Basis of
	$\Hom_{S_n}(S^k(\mathbb{R}^{n}), S^l(\mathbb{R}^{n}))$
	}

We construct our first basis of 
$\Hom_{S_n}(S^k(\mathbb{R}^{n}), S^l(\mathbb{R}^{n}))$
by 
associating the basis with
the orbits of an action of $S_n$ on
the Cartesian product set $S[n]^{l} \times S[n]^{k}$
and understanding each orbit as a diagram.
We begin with the following result.

\begin{proposition} \label{basisorbits}
	There is a basis of
	$\Hom_{S_n}(S^k(\mathbb{R}^{n}), S^l(\mathbb{R}^{n}))$
	that corresponds bijectively
	with the orbits coming from the action of 
	$S_n$ on the set $S[n]^{l} \times S[n]^{k}$.
\end{proposition}


\begin{figure*}[tb]
	\begin{tcolorbox}[colback=white!02, colframe=black]
	\begin{center}
		\scalebox{0.8}{\input{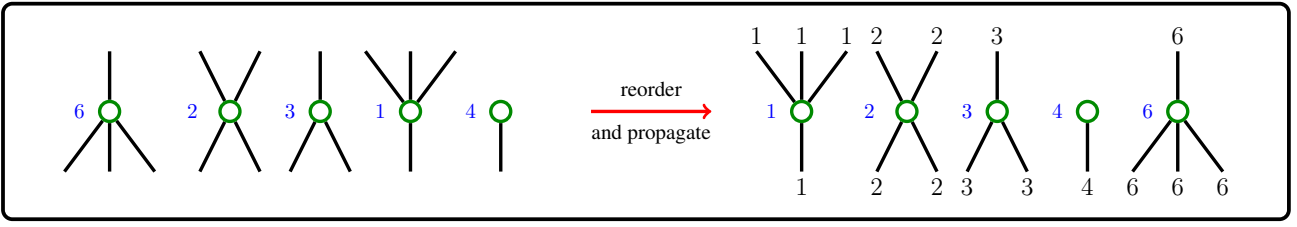}}
	\end{center}
	\end{tcolorbox}
	\caption{
For the $(9,7)$--bipartition
$\{
	[3,1],
	[2,2],
	[2,1],
	[1,3],
	[1,0]
\}$,
we show how to obtain
the element in the corresponding $S_6$ orbit of
$S[6]^{7} \times S[6]^{9}$
that comes from labelling its blocks
with the $5$-length tuple 
$\{6,2,3,1,4\}$.
	By reordering the labelled green nodes of the
	$(9,7)$--orbit bipartition diagram and propagating
these values to the ends of the wires, we see that this element is
$\binom{1112236}{122334666}$.
		}
  	\label{propagateorbit}
\end{figure*}


\begin{proof}
	We see that the vector space
	$\Hom(S^k(\mathbb{R}^{n}), S^l(\mathbb{R}^{n}))$
	has a standard basis of matrix units
	\begin{equation}
		\{E_{I,J}\}_{I \in S[n]^l, J \in S[n]^k}
	\end{equation}
	where $E_{I,J}$ has a $1$ in the $(I,J)$ position and is $0$ elsewhere.

	Hence, for any standard basis element 
	$e_P \in S^k(\mathbb{R}^{n})$,
	we have that
	\begin{equation}
		E_{I,J}e_P = \delta_{J,P}e_I
	\end{equation}
	and so, for any linear map 
	$f: S^k(\mathbb{R}^{n}) \rightarrow S^l(\mathbb{R}^{n})$,
	expressing $f$ in the basis of matrix units as
	\begin{equation}
		f = \sum_{I \in S[n]^l}\sum_{J \in S[n]^k} f_{I,J}E_{I,J}
	\end{equation}
	we get that
	\begin{equation}
		f(e_P) = \sum_{I \in S[n]^l} f_{I,P}e_I
	\end{equation}

	Given that 
	$\Hom_{S_n}(S^k(\mathbb{R}^{n}), S^l(\mathbb{R}^{n}))$
	is a subspace of
	$\Hom(S^k(\mathbb{R}^{n}), S^l(\mathbb{R}^{n}))$,
	we have that
	$f$ is an $S_n$-equivariant linear map
	if and only if, 
	for all $\sigma \in S_n$ and standard basis elements
	$e_J \in S^k(\mathbb{R}^{n})$,
	\begin{equation} \label{equivarg}
		f(\rho_{k}^{S}(\sigma)[e_J]) = \rho_{l}^{S}(\sigma)[f(e_J)] 
	\end{equation}
	(\ref{equivarg}) holds
	if and only if
	\begin{equation} \label{centeq2}
		\sum_{I \in S[n]^l} f_{I,\sigma(J)\pi_{\sigma(J), k}}e_I 
		= 
		\sum_{I \in S[n]^l} f_{I,J}e_{\sigma(I)\pi_{\sigma(I), l}} 
	\end{equation}
	for some $\pi_{\sigma(J), k} \in S_k$ and $\pi_{\sigma(I), l} \in S_l$
	that restores 
	ascending numerical order
	in $\sigma(J)$ and $\sigma(I)$
	respectively.

	(\ref{centeq2}) now holds if and only if
	\begin{equation} \label{centeq1}
		f_{\sigma(I)\pi_{\sigma(I), l},\sigma(J)\pi_{\sigma(J), k}} 
		= 
		f_{I,J} 
	\end{equation}
	for all $\sigma \in S_n$, $I \in S[n]^{l}$ and $J \in S[n]^{k}$.

	Therefore, taking the Cartesian product set $S[n]^{l} \times S[n]^{k}$
	and writing its elements as $\binom{I}{J}$,
	we define an action of $S_n$ on $S[n]^{l} \times S[n]^{k}$ by
	\begin{equation} \label{sigmapairaction}
		\sigma\binom{I}{J} 
		\coloneqq 
		\binom{\sigma(I)\pi_{\sigma(I), l}}{\sigma(J)\pi_{\sigma(J), k}}
	\end{equation}	

	Consequently, (\ref{centeq1}) tells us that the basis elements of
	$\Hom_{S_n}(S^k(\mathbb{R}^{n}), S^l(\mathbb{R}^{n}))$
	correspond bijectively
	with the orbits coming from the 
	action of $S_n$ on $S[n]^{l} \times S[n]^{k}$.
\end{proof}

Now, consider an orbit coming from the action of 
$S_n$ on $S[n]^{l} \times S[n]^{k}$.
Suppose that $\binom{I}{J}$ is a class representative.
We form a diagram in the following way:

\begin{tcolorbox}[colback=terracotta!10, colframe=terracotta!40]
	We place the values of $I$ in one row and the values of 
	$J$ in a second row directly below it. For each value $x \in [n]$, 
	if $x$ appears in either $I$ or $J$, then we insert a central green node 
	between the two rows and connect it to every occurrence of $x$:
	drawing edges downward from entries in $I$ 
	and upward from entries in $J$.
	If $x$ does not appear in either row, then no central green node is added.
\end{tcolorbox}

\begin{example} \label{orbitdiagex}
	Suppose that $l = 6$, $k = 4$, and $n = 3$.
	Consider the $S_3$ orbit of 
	$S[3]^{6} \times S[3]^{4}$
	that contains the element 
	$\binom{111233}{1123}$.
	
	Then, by the construction given above, 
	we obtain
	\begin{equation} \label{orbitdiagex1}
		\begin{aligned}
			\scalebox{0.7}{\input{orbitbasis1.tikz}}
		\end{aligned}
	\end{equation}
	Removing the labels gives the diagram that corresponds to the orbit.
\end{example}

This diagram is, in fact, a diagrammatic representation of a 
\textbf{$\bm{(k,l)}$--bipartition} having at most $n$ \textbf{blocks}
that for now we call a \textbf{$\bm{(k,l)}$--orbit bipartition diagram}.
We define bipartitions below.

\begin{definition}
	A $(k,l)$--bipartition $\pi$ having some $t$ blocks is the set 
	\begin{equation} \label{bipartition}
		\{
			[x_1, y_1],
			[x_2, y_2],
			\dots,
			[x_t, y_t]
		\}
	\end{equation}
	of $t$ pairs $[x_i, y_i]$ such that
	$x_i, y_i \geq 0$ for all $i$, not both zero,
	such that 
	$\sum_{i = 1}^{t} x_i = k$ and
	$\sum_{i = 1}^{t} y_i = l$.
	These were first introduced and studied by Major Percy A. 
	\citet{macmahon1893, macmahon1896, macmahon1899, 
	macmahon1900, macmahon1906, macmahon1908, macmahon1912, macmahon1918}.
	For $(k,l)$--orbit bipartition diagrams, we call the blocks 
	\textbf{spiders}
	and the black edges \textbf{wires},
	referring to them as \textbf{input wires} or \textbf{output wires},
	as appropriate.

	The number of $(k,l)$--bipartitions having at most $n$ blocks
	is the value $p_n(k,l)$.
	We provide a novel procedure for how to generate all
	$(k,l)$--bipartitions having at most $n$ blocks, and hence
	determine $p_n(k,l)$,
	in Appendix \ref{generatebipart}.


\end{definition}

Note that the $(k,l)$--bipartition that is obtained 
from an $S_n$ orbit
is independent of
the choice of class representative: 
if we choose 
$\binom{\sigma(I)\pi_{\sigma(I), l}}{\sigma(J)\pi_{\sigma(J), k}}$ instead
and form a diagram using the same procedure given 
in the red-coloured box
above,
then the spiders that appear in each of the diagrams 
(viewed without their labels)
will be the same, just 
potentially 
in a different order.
However, both diagrams give the \textit{same} $(k,l)$--bipartition.

\begin{example} \label{orbitdiagex2}
	Continuing on from Example \ref{orbitdiagex}, we see that
	$\binom{122333}{1233}$ is in the same orbit as
	$\binom{111233}{1123}$, since it is obtained by applying the permutation
	$(132)$ to $\binom{111233}{1123}$,
	then applying permutations in $S_6$ and $S_4$
	to reorder each component into ascending numerical order.
	Hence, by applying the procedure given above to $\binom{122333}{1233}$, 
	we obtain 
	the diagram
	\begin{equation}	
		\begin{aligned}
			\scalebox{0.7}{\input{orbitbasis2.tikz}}
		\end{aligned}
	\end{equation}
	which is the same $(4,6)$--orbit bipartition diagram as (\ref{orbitdiagex1}),
	except the spiders are in a different order.
	However, the underlying $(4,6)$--bipartitions are the same.
\end{example}

\begin{remark}
	Example \ref{orbitdiagex2}
	highlights an important point: 
	in drawing a $(k,l)$--orbit bipartition diagram for a $(k,l)$--bipartition,
	the order of the spiders corresponding to the blocks does not matter.
	Hence each $(k,l)$--orbit bipartition diagram 
	represents an entire equivalence class of diagrams that 
	correspond to a $(k,l)$--bipartition.
\end{remark}

Now, we also see that all such $(k,l)$--bipartitions having at most $n$ blocks must appear 
in this process:
representing each $(k,l)$--bipartition having some $t \in [n]$ blocks 
(of the form (\ref{bipartition}))
as a $(k,l)$--orbit bipartition diagram
by connecting, for each $i \in [t]$, 
$x_i$ wires from below and $y_i$ wires from above to a central green node,
we obtain its orbit by following the six steps in the blue-coloured box below.

\begin{tcolorbox}[colback=blue!10, colframe=blue!40]
\begin{enumerate}[leftmargin=5pt,label=\arabic*.]
	\item Form all possible $t$-length tuples with elements in $[n]$, 
		not allowing for repetitions amongst the elements.
	\item Then, for each $t$-length tuple, label the central green nodes 
		in the diagram
		from left to right with the elements of the tuple.
	\item If necessary, reorder the spiders 
		so that 
		they appear from left to right in increasing order of their labels.
	\item Then, for each spider, propagate the block
		label to the end of each wire. 
	\item This produces the element $\binom{I}{J}$,
		where
		the top row's labels are $I$
		and the bottom row's labels are $J$.
	\item Doing this for each $t$-length tuple gives the entire $S_n$ orbit.
\end{enumerate}
\end{tcolorbox}

\begin{example} \label{orbittuple1}
In Figure \ref{propagateorbit}, 
	we give an example of how to obtain,
for the $(9,7)$--bipartition
$\{
	[3,1],
	[2,2],
	[2,1],
	[1,3],
	[1,0]
\}$,
the element in the corresponding $S_6$ orbit of
$S[6]^{7} \times S[6]^{9}$
that comes from labelling its blocks
with the $5$-length tuple 
$\{6,2,3,1,4\}$.
In particular, by reordering the labelled green nodes of the
corresponding $(9,7)$--orbit bipartition diagram and propagating
these values to the ends of the wires, we see that this element is
$\binom{1112236}{122334666}$.
\end{example}

Hence, overall, we have shown the following result:

\begin{proposition} \label{orbitsdiags}
	The orbits that come from the action of $S_n$ on 
	$S[n]^{l} \times S[n]^{k}$
	correspond bijectively with 
	the $(k,l)$--bipartitions (or $(k,l)$--orbit bipartition diagrams) 
	having at most $n$ blocks.
\end{proposition}

Combining Proposition \ref{basisorbits} and Proposition \ref{orbitsdiags},
we obtain

\begin{theorem} \label{orbitbasisSn}
	There is a basis of
	$\Hom_{S_n}(S^k(\mathbb{R}^{n}), S^l(\mathbb{R}^{n}))$
	that corresponds bijectively with
	the $(k,l)$--bipartitions having at most $n$ blocks.
	Consequently, its dimension is $p_n(k,l)$.
	Moreover, we call this basis of 
	$\Hom_{S_n}(S^k(\mathbb{R}^{n}), S^l(\mathbb{R}^{n}))$
	the \textbf{orbit basis}. 
\end{theorem}

To find the orbit basis element $X_\pi$ in 
$\Hom_{S_n}(S^k(\mathbb{R}^{n}), S^l(\mathbb{R}^{n}))$
that corresponds 
to a $(k,l)$--bipartition $\pi$ having at most $n$ blocks,
we represent $\pi$ as a $(k,l)$--orbit bipartition diagram $x_\pi$ 
and follow the same steps $1$ through $5$ inclusive as before, 
except we now form the matrix unit $E_{I,J}$ from each
element $\binom{I}{J}$ that comes from a $t$-length tuple. 
Adding together all of these matrix units (over all of the tuples) gives $X_\pi$.
Hence we have the following corollary:

\begin{corollary}
	The $S_n$-equivariant weight matrix from 
	$S^k(\mathbb{R}^{n})$
	to
	$S^l(\mathbb{R}^{n})$
	is $\sum_{\pi} \lambda_\pi{X_\pi}$ 
	for weights $\lambda_\pi$,
	where the sum
	is over all $(k,l)$--bipartitions $\pi$ having at most $n$ blocks.
\end{corollary}

\begin{example}
	Continuing on from Example \ref{orbittuple1}, we see that
	the $(1112236,122334666)$--entry of the basis matrix $X_\pi$ in 
	$\Hom_{S_6}(S^9(\mathbb{R}^{6}), S^7(\mathbb{R}^{6}))$
	corresponding to the $(9,7)$--bipartition $\pi =
\{
	[3,1],
	[2,2],
	[2,1],
	[1,3],
	[1,0]
\}$ is $1$.
\end{example}

In the Appendix, we provide a general procedure
summarising the results above for calculating the $S_n$-equivariant weight matrix 
from $S^k(\mathbb{R}^{n})$ to $S^l(\mathbb{R}^{n})$ from all of the $(k,l)$--orbit
bipartition diagrams having at most $n$ blocks.

\begin{example} \label{orbitexample}
	Suppose that we would like to find
	the $S_3$-equivariant weight matrix
	from
	$S^2(\mathbb{R}^{3})$
	to
	$S^1(\mathbb{R}^{3}) = \mathbb{R}^{3}$.

	We need to consider the $(2,1)$--orbit bipartition diagrams having at most $n = 3$ blocks.
	They are 
	\begin{equation}	
		\begin{aligned}
			\scalebox{0.7}{\input{orbit21diags.tikz}}
		\end{aligned}
	\end{equation}
	
	Each of these four diagrams corresponds to an orbit basis matrix in 
	$\Hom_{S_3}(S^2(\mathbb{R}^{3}), S^1(\mathbb{R}^{3}))$
	of size $3 \times 6$.
	Example \ref{orbitexampleapp} gives the calculations of these matrices.


	Hence, the $S_3$-equivariant weight matrix
	from
	$S^2(\mathbb{R}^{3})$
	to
	$S^1(\mathbb{R}^{3})$
	is of the form 
\begin{equation}
	\NiceMatrixOptions{code-for-first-row = \scriptstyle \color{blue},
                   	   code-for-first-col = \scriptstyle \color{blue}
	}
	\renewcommand{\arraystretch}{1}
	\begin{bNiceArray}{*{6}{c}}[first-row,first-col]
				& 1,1 	& 1,2	& 1,3	& 2,2 	& 2,3	& 3,3	\\
		1		& \lambda_1	& \lambda_3	& \lambda_3	& \lambda_2	& \lambda_4	& \lambda_2	\\
		2		& \lambda_2	& \lambda_3	& \lambda_4	& \lambda_1	& \lambda_3	& \lambda_2	\\
		3		& \lambda_2	& \lambda_4	& \lambda_3	& \lambda_2	& \lambda_3	& \lambda_1
	\end{bNiceArray}
\end{equation}
for weights $\lambda_1, \lambda_2, \lambda_3, \lambda_4 \in \mathbb{R}$.
\end{example}


\section{Diagram Basis of
	$\Hom_{S_n}(S^k(\mathbb{R}^{n}), S^l(\mathbb{R}^{n}))$
	}

We are able to obtain a second basis of 
$\Hom_{S_n}(S^k(\mathbb{R}^{n}), S^l(\mathbb{R}^{n}))$
that we will show is more efficient for performing 
linear transformations
in the following section.
To obtain this second basis, we first need to define the following vector space
and reformulate the construction given in the previous section
as a linear map. 

\begin{definition}
	We define the formal $\mathbb{R}$-linear span of the 
	set of all $(k,l)$--orbit bipartition diagrams
	to be the vector space $SP_k^l(n)$.
	We call this vector space the 
	\textbf{spherical partition vector space} 
	since it is an adaptation of the spherical partition algebra
	that was first discovered by \citet{bastias2024}.
	We also call the generating basis the \textbf{orbit basis} of $SP_k^l(n)$.
\end{definition}

\begin{figure*}[tb]
	\begin{tcolorbox}[colback=melon!10, colframe=melon!40, coltitle=black, title={\bfseries A General Procedure for Calculating 
the $S_n$-Equivariant Weight Matrix from 
		$S^k(\mathbb{R}^{n})$
		to
		$S^l(\mathbb{R}^{n})$
		using the Diagram Basis of $SP_k^l(n)$.},
	fonttitle=\bfseries]
		Perform the following steps:
	\begin{enumerate}
		\item Calculate all of the $(k,l)$--bipartitions $\pi$ 
			that have at most $n$ blocks, and
		express each $(k,l)$--bipartition $\pi$ as a $(k,l)$--bipartition
		diagram $d_\pi$ in $SP_k^l(n)$.
	\item Then, if $t$ is the number of blocks in $\pi$,
		form all possible $t$-length tuples with elements in $[n]$, 
		allowing for repetitions amongst the elements.
	\item For each $t$-length tuple, label the central green nodes in $d_\pi$ 
		from left to right with the elements of the tuple.
		Reorder the spiders in $d_\pi$ such that they are in increasing
		order from left to right. 
		Fuse together any spiders whose central green nodes are labelled
		with the same value.
		Then, for each spider, 
		propagate the block label to the end of each wire. 
		We obtain the matrix unit $E_{I,J}$ from this diagram
		by setting $I$ to be the labels in the top row and
		$J$ to be the labels in the bottom row.
	\item Since each tuple corresponds to a matrix unit, add together
		all of the resulting matrix units from all of the tuples
		to obtain the basis matrix $D_\pi$ that is associated with $d_\pi$.
		Attach a weight $\lambda_\pi \in \mathbb{R}$ to each matrix $D_\pi$.
		Finally, calculate $\sum \lambda_\pi D_\pi$ to give the overall weight matrix.
	\end{enumerate}
	\end{tcolorbox}
  	\label{diagrambasissummaryprocedure}
\end{figure*}

\begin{definition} \label{partmatrixmap}
	For all non-negative integers $k, l$ and positive integers $n$, 
	we define a surjective map
	\begin{equation} \label{partmatrixmapexp}
		\Gamma_{k,n}^l : SP_k^l(n) \rightarrow 
		\Hom_{S_n}(S^k(\mathbb{R}^{n}), S^l(\mathbb{R}^{n}))
	\end{equation}
	on the orbit basis 
	of $SP_k^l(n)$ as follows, and extend linearly:
	\begin{equation} \label{partcentsurj}
		\Gamma_{k,n}^l(x_\pi) \coloneqq
		\begin{cases}
			X_\pi & \text{if $\pi$ has $n$ or fewer blocks} \\
			0     & \text{if $\pi$ has more than $n$ blocks}
		\end{cases}
	\end{equation}
\end{definition}

We now introduce a new diagram $d_\pi$ for each $(k,l)$--bipartition $\pi$:
it is formed in exactly the same way as a $(k,l)$--orbit bipartition diagram,
except the central green nodes are now filled in.
We call such a diagram a \textbf{$\bm{(k,l)}$--bipartition diagram}.

\begin{example}
	Let $\pi$ be the $(9,7)$--bipartition
	\begin{equation}
		\pi
		\coloneqq
		\{
			[3,1],
			[2,2],
			[2,1],
			[1,3],
			[1,0]
		\}
	\end{equation}
	Then its corresponding $(9,7)$--bipartition diagram, $d_\pi$, is
	\begin{equation}	
		\begin{aligned}
			\scalebox{0.7}{\input{diagrambasis1.tikz}}
		\end{aligned}
	\end{equation}
	As with the orbit bipartition diagrams, the order of the spiders
	corresponding to the blocks does not matter in drawing the bipartition diagram.
\end{example}

We define $d_\pi$ to be an element of $SP_k^l(n)$ as follows:
\begin{equation} \label{orbitbasis}
	d_\pi 
	\coloneqq
	\sum_{\pi \preceq \theta} x_\theta
\end{equation}
where $\preceq$ is the partial order defined on all $(k,l)$--bipartitions as follows:
$\pi \preceq \theta$ if every pair in $\theta$ is the sum of pairs in $\pi$
such that each pair in $\pi$ appears in exactly one sum.

\begin{example}
	We see that if 	
	$\pi
		\coloneqq
		\{
			[1,1],
			[1,0]
		\}
	$
	then, by (\ref{orbitbasis}), we have that
	\begin{equation}	
		\begin{aligned}
			\scalebox{0.7}{\input{partialorder.tikz}}
		\end{aligned}
	\end{equation}
\end{example}

\begin{lemma} \label{Homdiagbasis}
	The set of $(k,l)$--bipartition diagrams $\{d_\pi\}$
	forms a basis of $SP_k^l(n)$.
	We call this basis the \textbf{diagram basis} of $SP_k^l(n)$.
\end{lemma}

We can combine the results in this section to define another map
\begin{equation}
	\Theta_{k,n}^l: SP_k^l(n) \rightarrow 
	\Hom_{S_n}(S^k(\mathbb{R}^{n}), S^l(\mathbb{R}^{n}))
\end{equation}
this time on the diagram basis of $SP_k^l(n)$.

\begin{definition} \label{mapthetakl}
We define $\Theta_{k,n}^l(d_\pi)$ in the following way.
First change the basis in $SP_k^l(n)$ 
from the diagram basis to the orbit basis
using (\ref{orbitbasis}),
and then apply $\Gamma_{k,n}^l$ to each of the orbit basis diagrams
that appear in (\ref{orbitbasis}) for $d_\pi$.
We get that $\Theta_{k,n}^l(d_\pi)$ is defined to be
\begin{equation} \label{diagbasiscalcs}
	\Gamma_{k,n}^l\Bigl(\sum_{\pi \preceq \theta} x_\theta\Bigr)
	= 
	\sum_{\pi \preceq \theta} \Gamma_{k,n}^l(x_\theta)
	= 
	\sum_{\pi \preceq  \theta, \; b(\theta) \leq n} X_\theta 
\end{equation}
where $b(\theta)$ is the number of blocks in $\theta$.
Now define
\begin{equation} \label{diagbasisdefn}
	D_\pi \coloneqq
	\sum_{\pi \preceq  \theta, \; b(\theta) \leq n} X_\theta 
\end{equation}
Clearly,
$D_\pi$ is an element of 
$\Hom_{S_n}(S^k(\mathbb{R}^{n}), S^l(\mathbb{R}^{n}))$.
Hence 
we have defined $\Theta_{k,n}^l(d_\pi) \coloneqq D_\pi$.
$\Theta_{k,n}^l$ becomes a 
map on $SP_k^l(n)$ by extending the definition
by linearity to the entire diagram basis of $SP_k^l(n)$.
We claim the following. 
\end{definition}

\begin{theorem} \label{diagbasisSn}
	For all non-negative integers $k, l$ and positive integers $n$,
	the set 
		$\{D_\pi \mid d_\pi \text{ has at most } n \text{ blocks}\}$
	forms a basis of
	$\Hom_{S_n}(S^k(\mathbb{R}^{n}), S^l(\mathbb{R}^{n}))$.
	We call this basis the \textbf{diagram basis} of
	$\Hom_{S_n}(S^k(\mathbb{R}^{n}), S^l(\mathbb{R}^{n}))$.
	In particular, the map $\Theta_{k,n}^l$ is surjective.
\end{theorem}


\begin{figure*}[tb]
	\begin{tcolorbox}[colback=white!02, colframe=black]
	\begin{center}
		\scalebox{0.8}{\input{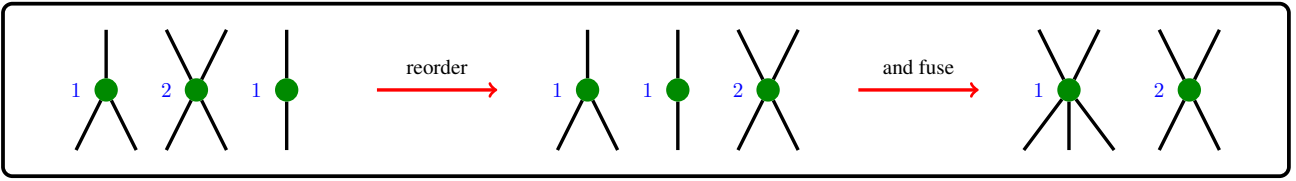}}
	\end{center}
	\end{tcolorbox}
		\caption{After labelling the central green nodes of a $(5,4)$--bipartition diagram
	with the tuple $(1,2,1)$ and reordering the spiders into 
	ascending numerical order,
	we fuse 
	together any spiders that are labelled with the same value.}

  	\label{fusespiders}
\end{figure*}


Note that whilst $\Theta_{k,n}^l$ is defined on all diagram basis elements in
$SP_k^l(n)$, it is only those that have at most $n$ blocks whose image under 
$\Theta_{k,n}^l$ creates a basis of 
$\Hom_{S_n}(S^k(\mathbb{R}^{n}), S^l(\mathbb{R}^{n}))$.
In particular, we have
the following corollary.

\begin{corollary}
	The $S_n$-equivariant weight matrix from 
	$S^k(\mathbb{R}^{n})$
	to
	$S^l(\mathbb{R}^{n})$
	is $\sum_{\pi} \lambda_\pi{D_\pi}$ 
	for weights $\lambda_\pi$,
	where the sum
	is over all $(k,l)$--bipartitions $\pi$ having at most $n$ blocks.
\end{corollary}

We would like to be able to calculate the diagram basis of 
$\Hom_{S_n}(S^k(\mathbb{R}^{n}), S^l(\mathbb{R}^{n}))$
directly from the diagram basis of $SP_k^l(n)$, 
without having to go via the orbit basis of $SP_k^l(n)$.
We claim the following:

\begin{proposition} \label{diagbasisdirect}
	The diagram basis element $D_\pi$ in 
	$\Hom_{S_n}(S^k(\mathbb{R}^{n}), S^l(\mathbb{R}^{n}))$
	that corresponds 
	to a $(k,l)$--bipartition diagram $d_\pi$ having some $t \leq n$ blocks
	can be found as follows:
	we follow the same steps $1$ through $5$ inclusive in the blue-coloured box as before, 
	except in Step $1$ we now allow repetitions amongst the elements 
	in each $t$-length tuple,
	and in Step $3$, after reordering the spiders, we fuse together any central green nodes
	that are labelled with the same value. See Figure \ref{fusespiders} for an example.
	Adding together all of the resulting matrix units gives $D_\pi$.
\end{proposition}

In the orange-coloured box, we provide a general procedure
for calculating the $S_n$-equivariant weight matrix 
from $S^k(\mathbb{R}^{n})$ to $S^l(\mathbb{R}^{n})$ using 
the diagram basis of $SP_k^l(n)$.

\begin{example} \label{diagramexample}
	We return to Example \ref{orbitexample} and calculate the
	$S_3$-equivariant weight matrix
	from
	$S^2(\mathbb{R}^{3})$
	to
	$S^1(\mathbb{R}^{3})$ using the diagram basis of $SP_2^1(3)$ instead.

	We now consider the $(2,1)$--bipartition diagrams having at most $n = 3$ blocks:
	\begin{equation} \label{bipartdiagsex}
		\begin{aligned}
			\scalebox{0.7}{\input{diagram21diags.tikz}}
		\end{aligned}
	\end{equation}
	As before, the calculations of the corresponding matrices are given
	in Example \ref{diagramexampleapp}.

	Hence, the $S_3$-equivariant weight matrix
	from
	$S^2(\mathbb{R}^{3})$
	to
	$S^1(\mathbb{R}^{3})$
	using the diagram basis
	is of the form 
\begin{equation} \label{diagweightmatex}
	\NiceMatrixOptions{code-for-first-row = \scriptstyle \color{blue},
                   	   code-for-first-col = \scriptstyle \color{blue}
	}
	\renewcommand{\arraystretch}{1}
	\begin{bNiceArray}{*{6}{c}}[first-row,first-col]
				& 1,1 	& 1,2	& 1,3	& 2,2 	& 2,3	& 3,3	\\
		1		& \lambda_{1,2,3,4}	& \lambda_{3,4}	& \lambda_{3,4}	& \lambda_{2,4}	& \lambda_4	& \lambda_{2,4}	\\
		2		& \lambda_{2,4}	& \lambda_{3,4}	& \lambda_4	& \lambda_{1,2,3,4}	& \lambda_{3,4}	& \lambda_{2,4}	\\
		3		& \lambda_{2,4}	& \lambda_4	& \lambda_{3,4}	& \lambda_{2,4}	& \lambda_{3,4}	& \lambda_{1,2,3,4}
	\end{bNiceArray}
\end{equation}
for weights $\lambda_1, \lambda_2, \lambda_3, \lambda_4 \in \mathbb{R}$,
	where $\lambda_A \coloneqq \sum_{i \in A} \lambda_i$.
\end{example}

\begin{remark}
	In Example \ref{deepsetsexample} in the Appendix,
	we show that we can recover the Deep Sets \citep{zaheer2017}
	characterisation of any $S_n$-equivariant weight matrix	
	from
	$\mathbb{R}^{n}$
	to 
	$\mathbb{R}^{n}$
	using the diagram basis of $SP_1^1(n)$. 
\end{remark}

\section{Embedding into Tensor Power Space}

There is a useful way of representing any $S_n$-equivariant weight matrix $W$
from 
$S^k(\mathbb{R}^{n})$
to 
$S^l(\mathbb{R}^{n})$
as a matrix from
$(\mathbb{R}^{n})^{\otimes k}$
to 
$(\mathbb{R}^{n})^{\otimes l}$,
namely by embedding $S^k(\mathbb{R}^{n})$ into $(\mathbb{R}^{n})^{\otimes k}$
and similarly for $S^l(\mathbb{R}^{n})$.
We call the latter matrix the \textbf{unrolled} $S_n$-equivariant weight matrix,
and the process of representing the original matrix in this way \textbf{unrolling}.

The $(I,J)$--entry in the unrolled weight matrix is the same as
the $(I\pi_{I,l}, J\pi_{J,k})$--entry in $W$, where 
$\pi_{I,l}$, $\pi_{J,k}$ are any permutations in $S_l$ and $S_k$, respectively,
that restore 
ascending numerical order
in $I$ and $J$, respectively.

This is because the $k^{\text{th}}$ symmetric power of $\mathbb{R}^{n}$ 
can be thought of as the quotient space of $(\mathbb{R}^{n})^{\otimes k}$ by 
the ideal generated by elements of the form $x \otimes y - y \otimes x$.

The benefit of unrolling an $S_n$-equivariant weight matrix to become a map
from
$(\mathbb{R}^{n})^{\otimes k}$
to 
$(\mathbb{R}^{n})^{\otimes l}$
is that we can now operate on symmetric tensors in $(\mathbb{R}^{n})^{\otimes k}$ whilst
retaining the original properties of the matrix.
Going forward, we now consider any weight matrix to be in its unrolled form.

\begin{example}
	Continuing on from Example \ref{orbitexample}, 
	we see that the unrolled $S_3$-equivariant weight matrix	
	from
	$(\mathbb{R}^{3})^{\otimes 2}$
	to 
	$\mathbb{R}^{3}$
	is
\begin{equation}
	\NiceMatrixOptions{code-for-first-row = \scriptstyle \color{blue},
                   	   code-for-first-col = \scriptstyle \color{blue}
	}
	\renewcommand{\arraystretch}{1}
	\begin{bNiceArray}{*{9}{c}}[first-row,first-col]
		& 1,1 	& 1,2	& 1,3	& 2,1 	& 2,2 	& 2,3	& 3,1	& 3,2	& 3,3	\\
		1		& \lambda_1	& \lambda_3	& \lambda_3	& \lambda_3	& \lambda_2	& \lambda_4	& \lambda_3	& \lambda_4	& \lambda_2	\\
		2		& \lambda_2	& \lambda_3	& \lambda_4	& \lambda_3	& \lambda_1	& \lambda_3	& \lambda_4	& \lambda_3	& \lambda_2	\\
		3		& \lambda_2	& \lambda_4	& \lambda_3	& \lambda_4	& \lambda_2	& \lambda_3	& \lambda_3	& \lambda_3	& \lambda_1
	\end{bNiceArray}
\end{equation}
\end{example}


\section{Map Label Notation}


In practice, when it comes to implementing the transformation of an input tensor $T$
by an unrolled weight matrix $W$, expressed in tensor form,
restrictions on memory make it impractical for the entire weight matrix to be stored in memory.
Instead we can use what we term \textbf{map label notation} to describe,
for each basis matrix $D_{\pi}$ that appears in $W$, 
the transformation of an input $T$ into its output $D_{\pi}(T)$.
It is here where the diagram basis shines over the orbit basis,
in that it is much easier to vectorize
the map label for each diagram basis element,
since it includes, as a result of (\ref{orbitbasis}),
all labels that come from all of the possible fusings 
of the spiders in the bipartition diagram $d_{\pi}$
that corresponds to $D_{\pi}$.
In fact, instead of calculating a separate weight matrix explicitly for each value of $n$,
when $n \geq l+k$, we will see that a single set of map labels 
can be used to obtain the transformation for any value of $n$,
and, when $n < l+k$, we only need to retain map labels 
that come from $(k,l)$--bipartition diagrams having at most $n$ blocks.
Consequently, the only requirements on memory come from needing to store 
the input tensor $T$
and the output tensor $W(T)$ in memory.

\begin{definition} \label{maplabeldefn}
	Let $d_\pi$ be a $(k,l)$--bipartition diagram. 
	We associate to $d_\pi$ a set of \textbf{map labels},
	where each map label consists of an arrow $\leftarrow$,
	a fixed pattern of indices $I$ on its left hand side,
	and, on its right hand side, 
	a formal linear combination of
	tuples of indices
	that are associated with $I$.

	We immediately use each map label for $d_\pi$
	to obtain the values in $D_{\pi}(T)$, for unrolled $D_{\pi}$
	and input tensor $T \in (\mathbb{R}^{n})^{\otimes k}$, by 
	replacing, in each map label, $I$ by ${D_{\pi}(T)}_I$
	and each tuple $J$ that appears on the right hand side
	by $T_J$, and treating the result as a map that describes
	how to find ${D_{\pi}(T)}_I$.
	The resulting \textbf{transformation map labels} coming from all of the map labels ultimately determine $D_{\pi}(T)$.

\end{definition}



We have provided a full procedure for how to obtain the map labels
and the resulting transformation for any $(k,l)$--bipartition diagram $d_\pi$
and for any value of $n$ in the green-coloured box. 
The implementations of Subprocedures I to V,
along with illustrative examples for each,
are provided in Appendix \ref{maplabelimpl}.
We give definitions of the important terms that appear in the procedure below.

\begin{definition}
	Let $d_\pi$ be a $(k,l)$--bipartition diagram.
	A \textbf{grouped output} for $d_\pi$ is any tuple of labels that can be obtained by labelling all
	the spiders that have output wires, allowing repeats, sorting the spiders in decreasing order of output wires,
	propagating the labels to the ends of the output wires, and reading off the resulting tuple.
	A \textbf{partially labelled diagram} of $d_\pi$ is identical to $d_\pi$,
	except all spiders that have output wires are assigned a label, allowing repeats.
	A \textbf{fully labelled diagram} of $d_\pi$ is identical to $d_\pi$,
	except now all spiders are assigned a label, allowing repeats.
	A \textbf{grouped map label} is obtained from a fully labelled diagram of $d_\pi$
	by propagating the labels to the ends of the wires, placing the tuple on the output wires
	to the left of an arrow $\xleftarrow{g}$ and placing the tuple on the input wires to the right of the arrow.
	A \textbf{left grouped map label} is obtained from a grouped map label by unrolling only the right hand side
	of a grouped map label, and replacing $\xleftarrow{g}$ with $\xleftarrow{lg}$.
\end{definition}

\begin{example} \label{mainexmaplabels}
	Continuing on from Example \ref{diagramexample}, 
	and considering once again the $(2,1)$--bipartition diagrams having at most $n = 3$ blocks
	given in (\ref{bipartdiagsex}), 
	we show in Example \ref{mainexmaplabelscalcs}
	that the transformations corresponding to each $(2,1)$--bipartition diagram, for $n = 3$,
	are fully determined by the following transformation map labels ---
	in this case, only one for each $(2,1)$--bipartition diagram:
	\begin{equation} \label{tensormaplabelex1}
		D_{\pi_1}(T)_i = T_{i,i}
	\end{equation}
	\begin{equation} \label{tensormaplabelex2}
		D_{\pi_2}(T)_i = \sum_{j = 1}^{3} T_{j,j}
	\end{equation}

\begin{figure}[thb]
	\begin{center}
		\includegraphics[scale=0.5]{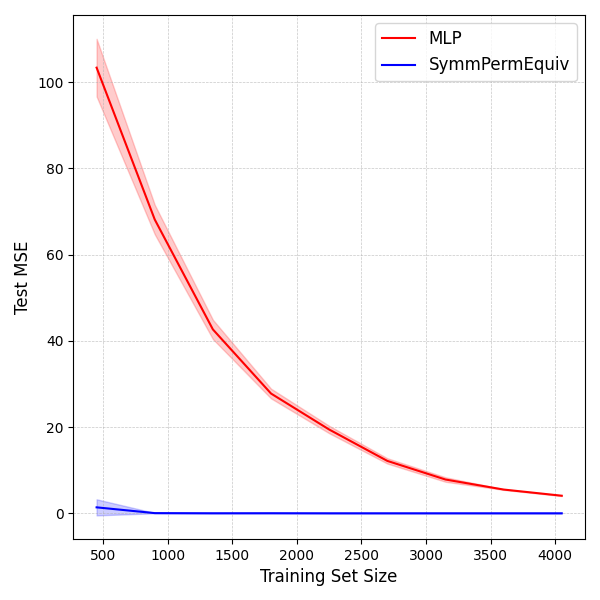} 
	\end{center}
		\caption{Data efficiency for the synthetic $S_{12}$-invariant task.
		The shaded regions depict 95\% confidence intervals taken over 3 runs.}
  	\label{fig:sumiji}
\end{figure}

	\begin{equation} \label{tensormaplabelex3}
		D_{\pi_3}(T)_i = \sum_{j = 1}^{3} [T_{i,j} + T_{j,i}] - T_{i,i}
	\end{equation}

	\begin{equation} \label{tensormaplabelex4}
		D_{\pi_4}(T)_i = \sum_{j = 1}^{3} \sum_{k = 1}^{3} T_{j,k}
	\end{equation}
	where $i \in [3]$.

	However, the power of this method lies in the fact that if $n \geq 3$, 
	then we only need to change the value $3$ in each sum to $n$,
	and if $n = 2$, then we remove $\pi_4$ and change the 
	value $3$ in each sum to $2$.
\end{example}


\begin{figure*}[thb!]
	\begin{tcolorbox}[colback=ygreen!10, colframe=ygreen!40, coltitle=black, 
		title={\bfseries 
		Procedure: Generation of the Transformation Map Labels That Describe the Transformation Corresponding
		to a $(k,l)$--Bipartition Diagram
		of a Symmetric Tensor in $(\mathbb{R}^{n})^{\otimes k}$
		to a Symmetric Tensor in $(\mathbb{R}^{n})^{\otimes l}$.
		},
		fonttitle=\bfseries
		]
		\textbf{Input:} 
		A $(k,l)$--bipartition diagram $d_{\pi}$ and a value of $n$.
	\begin{enumerate}
		\item Apply Subprocedure I to obtain all possible grouped outputs for $d_{\pi}$, and for each one,
			its associated set of partially labelled diagrams, where only spiders that have output wires are labelled in these diagrams.
		\item
			For each grouped output for $d_{\pi}$:
			\begin{enumerate}
				\item 
					Apply Subprocedure II to each partially labelled diagram in the grouped output's associated set.
					Take the set union of all fully labelled diagrams that are returned from each call to the subprocedure,
					giving a set of fully labelled diagrams that are associated with the grouped output.
				\item
					Remove all fully labelled diagrams in the set with more than $n$ spiders.
					If the set becomes empty, delete the grouped output and proceed to the next iteration of the loop.
				\item
					Apply Subprocedure III to each fully labelled diagram to convert it into a grouped map label.
				\item
					For each grouped map label, apply Subprocedure IV to unroll its right hand side,
					returning a left grouped map label.
				\item
					Formally add together the right hand side of the resulting left grouped map labels,
					creating one left grouped map label that is associated with the grouped output.
				\item	
					Unroll the left hand side of the left grouped map label 
					according to Subprocedure V
					to obtain a set of map labels
					that are associated with the grouped output. 
			\end{enumerate}
		\item	\label{step:maplabel}
			The union of the resulting map label sets from all grouped outputs
			are the map labels that correspond to the $(k,l)$--bipartition diagram $d_{\pi}$.
		\item
			Finally, for each map label 
			in Step \ref{step:maplabel},
			convert each index $J$ on the right hand side to be $T_J$, and convert
			the index $I$ on the left hand side to be ${D_{\pi}(T)}_{I}$ 
			to obtain its corresponding transformation map label.
	\end{enumerate}
		\textbf{Output:} A set of transformation map labels describing the transformation $D_{\pi}$
		of a symmetric tensor $T \in (\mathbb{R}^{n})^{\otimes k}$
		to a symmetric tensor $D_{\pi}(T) \in (\mathbb{R}^{n})^{\otimes l}$
		for the given value of $n$, where
		$D_{\pi}$ corresponds to the $(k,l)$--bipartition diagram $d_{\pi}$.
	\end{tcolorbox}
  	\label{maplabelprocedure}
\end{figure*}


\section{Numerical Experiments}

We demonstrate our characterisation on the following toy experiments.

\textbf{$S_{12}$-Invariant Task:}
We evaluate our model on a synthetic $S_{12}$-invariant task
given by the function $f(T) \coloneqq \sum_{i,j}^{12} T_{i,j,i}$,
where $T$ is a $3$-order symmetric tensor.
We demonstrate the high data efficiency 
of our model
compared with a standard MLP for this task,
as shown in Figure \ref{fig:sumiji}.
We attribute this superior performance to our model's strong inductive bias.

\textbf{$S_8$-Equivariant Task:}
We evaluate our model on a synthetic $S_8$-equivariant task
from $(\mathbb{R}^{8})^{\otimes 3}$ to $\mathbb{R}^{8}$: 
namely, to extract the diagonal from $8 \times 8 \times 8$ symmetric tensors.
We evaluate our model against a standard MLP and a standard
$S_8$-equivariant model 
from $(\mathbb{R}^{8})^{\otimes 3}$ to $\mathbb{R}^{8}$.
We show the Test Mean Squared Error 
(MSE) for each of these models in Table \ref{tab:diag}.
Furthermore, Table \ref{tab:generalisation} shows that the network appears to generalise well 
to symmetric tensors having different sizes, even though it was trained on 
symmetric tensors of a particular size.
\begin{table}[h!]
\centering
\begin{tabular}{|l|c|}
\hline
\textbf{Method} & \textbf{Test MSE} \\ \hline
MLP            & $0.6486$            \\ \hline
PermEquiv      & $0.0447$           \\ \hline
SymmPermEquiv       & $\bm{0.0035}$            \\ \hline
\end{tabular}
	\caption{Diagonal Extraction}
\label{tab:diag}
\end{table}

\begin{table}[h!]
\centering
\begin{tabular}{|c|c|c|c|}
\hline
	\textbf{$\bm{n}$} & $\bm{8}$ & $\bm{16}$ & $\bm{32}$ \\ \hline
\textbf{Test MSE}            & $0.0035$      & $0.0048$     & $0.0088$      \\ \hline
\end{tabular}
\caption{Generalisation}
\label{tab:generalisation}
\end{table}



\section{Conclusion}

We obtained a complete characterisation of 
permutation equivariant linear functions between symmetric powers of 
$\mathbb{R}^{n}$ using two different bases. 
This characterisation includes all scalar-valued 
and vector-valued permutation equivariant 
functions on symmetric tensors in 
$(\mathbb{R}^{n})^{\otimes k}$.
We introduced map label notation to address limitations on memory, 
enabling the efficient transformation of input tensors without explicitly 
storing large sized weight matrices in memory. 
We validated our approach through two toy examples
and demonstrated the potential for transfer learning 
to symmetric tensors of different sizes.


\section*{Acknowledgements}
The author 
is deeply grateful to Chen Lin for his valuable insights 
and constructive discussions 
on designing experiments.
This work is supported by the UKRI–EPSRC grant EP/Y028872/1, 
\textit{Mathematical Foundations of Intelligence: An ``Erlangen Programme” for AI}.

\section*{Impact Statement}

This paper presents work whose goal is to advance the field of 
Machine Learning. There are many potential societal consequences 
of our work, none which we feel must be specifically highlighted here.


\nocite{*}

\bibliography{references}
\bibliographystyle{icml2025}

\newpage
\appendix
\onecolumn

\section{Supplementary Proofs and Definitions}

\begin{proof}[Proof of Proposition \ref{linearsymmchar}]
	Suppose that $f$ is a linear permutation equivariant function
	on a symmetric tensor $T$ in $(\mathbb{R}^{n})^{\otimes k}$ to $(\mathbb{R}^{n})^{\otimes l}$,
	where $l$ is either $0$ or $1$.
	
	Then, for all tuples of indices $J = (j_1, \dots, j_k)$ that index the standard
	basis $e_J \in (\mathbb{R}^{n})^{\otimes k}$, we can partition these
	indices into disjoint sets such that $I$ and $J$ are in the same set
	if and only if there is a permutation in $S_k$ such that $\sigma(I) = J$.

	Moreover, for each such set, $f$ must map each element in the set to the same
	value, by the definition of $f$.
	Hence, for each set, we can select the unique tuple in the set
	that is in 
	ascending numerical
	order.

	The union of these tuples across all of the sets gives $S[n]^k$,
	which indexes the standard basis of $S^k(\mathbb{R}^{n})$, by construction.
	Moreover, since $l$ is either $0$ or $1$, we see that
	$\mathbb{R} = S^0(\mathbb{R}^{n})$ and $\mathbb{R}^{n} = S^1(\mathbb{R}^{n})$,
	giving the result.
\end{proof}

\begin{proof}[Proof of Lemma \ref{Homdiagbasis}]
	Let 
	$B_O \coloneqq \{x_\pi\}$ be the set of $(k,l)$--orbit bipartition diagrams
	and let
	$B_D \coloneqq \{d_\pi\}$ be the set of $(k,l)$--bipartition diagrams.

	To see why the set $B_D$ forms a basis of $SP_k^l(n)$,
	first we form an ordered set of $(k,l)$--bipartitions by ordering 
	the bipartitions by the number of blocks that they have from smallest to largest, 
	with any arbitrary ordering allowed for a pair of bipartitions that have 
	the same number of blocks.
	Call this set $S_{{l+k}}$.
	
	Then, because the square matrix that maps elements of $B_D$ to 
	linear combinations of the set $B_O$ --- 
	whose rows and columns are indexed (in order) by the ordered set $S_{{l+k}}$ 
	--- is unitriangular by (\ref{orbitbasis}), it is therefore invertible.
	Since $B_O$ is the generating basis of $SP_k^l(n)$,
	we get that $B_D$ also forms a basis of $SP_k^l(n)$, as required.
\end{proof}

\begin{proof}[Proof of Theorem \ref{diagbasisSn}]
Consider the two sets
	$B_D \coloneqq \{D_\pi\}$
and
	$B_O \coloneqq \{X_\theta\}$ 
	where each set is over all $(k,l)$--bipartitions that have at most $n$ blocks.
	We see that, 
	by the definition of $D_\pi$ in (\ref{diagbasisdefn}), 
	the transition matrix from $B_D$ to $B_O$ 
	is unitriangular.
As $B_O$
forms a basis of 
$\Hom_{S_n}(S^k(\mathbb{R}^{n}), S^l(\mathbb{R}^{n}))$
by Theorem \ref{orbitbasisSn},
this implies that $B_D$
is also a basis of
$\Hom_{S_n}(S^k(\mathbb{R}^{n}), S^l(\mathbb{R}^{n}))$.
Since $B_D$ is a basis, where each element in $B_D$ 
is the image under $\Theta_{k,n}^l$ of a diagram basis element in $SP_k^l(n)$,
and $\Theta_{k,n}^l$ is linear, the surjectivity of $\Theta_{k,n}^l$ is immediate.
\end{proof}

\begin{proof}[Proof of Proposition \ref{diagbasisdirect}]
	Suppose that $d_\pi$ is a diagram basis element 
	in $SP_k^l(n)$.
	Assume that we have placed the indices $I$ 
	on the top row of $d_\pi$ and the indices $J$ 
	on the bottom row of $d_\pi$.

	Consider the following cases.

	\textbf{Case 1}: $(I,J)$ does not label the central nodes of $d_\pi$ consistently.

	Then, by (\ref{orbitbasis}) and
	the definition of the partial ordering $\preceq$ on $(k,l)$--bipartitions,
	$(I,J)$ does not label the central nodes of any of the $x_\theta$ consistently,
	since such $x_\theta$ are formed diagrammatically 
	by fusing together some of the spiders of $x_\pi$ at the central nodes.
	Hence the $(I,J)$--entry of each $X_\theta$ that is the image of 
	$x_\theta$ under $\Gamma_{k,n}^l$ is $0$, and so 
	the $(I,J)$--entry of $D_\pi$ is also $0$, by (\ref{diagbasisdefn}).

	\textbf{Case 2}: $(I,J)$ labels the central nodes of $d_\pi$ consistently.

	Then, by (\ref{orbitbasis}),
	we see that $(I,J)$ labels the central nodes of exactly one $x_\theta$ in
	(\ref{orbitbasis})
	consistently \textit{and distinctly}.
	Indeed,
	this $x_\theta$ is obtained by fusing together the central nodes of all of 
	the spiders of the labelled diagram $x_\pi$ whose labels are the same.
	Hence the $(I,J)$--entry for $X_\theta$ that is the image of this $x_\theta$ 
	under $\Gamma_{k,n}^l$ is $1$, and is $0$ for the images 
	of all of the other orbit basis diagrams in (\ref{orbitbasis}) 
	that are not $x_\theta$.
	Hence the $(I,J)$--entry of $D_\pi$ is $1$, by (\ref{diagbasisdefn}).
\end{proof}

\begin{figure*}[tb]
	\begin{tcolorbox}[colback=teal!10, colframe=teal!40, coltitle=black, title={\bfseries A General Procedure for Calculating 
the $S_n$-Equivariant Weight Matrix from 
		$S^k(\mathbb{R}^{n})$
		to
		$S^l(\mathbb{R}^{n})$
		using the Orbit Basis of $SP_k^l(n)$.},
	fonttitle=\bfseries]
		Perform the following steps:
	\begin{enumerate}
		\item Calculate all of the $(k,l)$--bipartitions $\pi$ 
			that have at most $n$ blocks, and
		express each $(k,l)$--bipartition $\pi$ as a $(k,l)$--orbit bipartition
		diagram $x_\pi$ in $SP_k^l(n)$.
	\item Then, if $t$ is the number of blocks in $\pi$,
		form all possible $t$-length tuples with elements in $[n]$, 
		\textbf{not} allowing for repetitions amongst the elements.
	\item For each $t$-length tuple, label the central green nodes in $x_\pi$ 
		from left to right with the elements of the tuple.
		Reorder the spiders in $x_\pi$ such that they are in increasing
		order from left to right. Then, for each spider, 
		propagate the block label to the end of each wire. 
		We obtain the matrix unit $E_{I,J}$ from this diagram
		by setting $I$ to be the labels in the top row and
		$J$ to be the labels in the bottom row.
	\item Since each tuple corresponds to a matrix unit, add together
		all of the resulting matrix units from all of the tuples
		to obtain the basis matrix $X_\pi$ that is associated with $x_\pi$.
		Attach a weight $\lambda_\pi \in \mathbb{R}$ to each matrix $X_\pi$.
		Finally, calculate $\sum \lambda_\pi X_\pi$ to give the overall weight matrix.
	\end{enumerate}
	\end{tcolorbox}
  	\label{orbitbasissummaryprocedure}
\end{figure*}


\section{How to Generate All $(k,l)$--Bipartition Diagrams Having at Most $n$ Blocks.}
\label{generatebipart}

In Theorem \ref{orbitbasisSn} and Theorem \ref{diagbasisSn}, we constructed two different bases for
$\Hom_{S_n}(S^k(\mathbb{R}^{n}), S^l(\mathbb{R}^{n}))$
using the set of all $(k,l)$--bipartitions having at most $n$ blocks. 
However, we have yet to outline a method for how to generate all of these bipartitions.
To the best of our knowledge, there have only been a few studies on 
$(k,l)$--bipartitions restricted to a certain number of blocks
\citep{wright1964, landman1992, kim1997}.
In particular, no-one has given an explicit procedure for generating
all $(k,l)$--bipartitions having at most $n$ blocks: at best,
\citet{kim1997} suggest that one
can obtain them from the set of all $(k+l,0)$--bipartitions having at most $n$ blocks,
but they do not explain how to do this explicitly.

In the green- and orange-coloured boxes below, we present a procedure for 
generating these bipartitions, notably without creating any duplicates. 
While the procedure is described for $(k,l)$--bipartition diagrams, 
it applies equally to $(k,l)$--orbit bipartition diagrams. 
Only the choice of diagram that is specified in Step 6 of the green-coloured box changes.

To explain why this procedure works, we need to introduce some definitions.

\begin{definition}
	An \textbf{integer partition} $\lambda$ of $m$ 
	is defined to be a tuple 
	$(\lambda_1, \lambda_2, \dots, \lambda_t)$ 
	of positive integers $\lambda_i$
	such that $\sum_{i = 1}^{t} \lambda_i = m$ and
			$\lambda_1 \geq \lambda_2 \geq \ldots \geq \lambda_p > 0$.
	We say that $\lambda$ consists of $t$ \textbf{parts}, 
	and we define its length to be~$t$.

	We often extend the tuple of any integer partition
	$\lambda$ of $m$ that has some $t < n$ parts
	to have exactly $n$ parts by appending $n-t$ zeros to $\lambda$.
	Hence $\lambda = (\lambda_1, \lambda_2, \dots, \lambda_n)$
	of non-negative integers $\lambda_i$
	such that $\sum_{i = 1}^{n} \lambda_i = m$ and
	$\lambda_1 \geq \lambda_2 \geq \ldots \geq \lambda_n \geq 0$.	
\end{definition}

\begin{example}
	We see that
	$\lambda = (4, 2, 2)$
	is an integer partition of $8$ into exactly $3$ parts.
	If $n = 5$, then we write $\lambda$ as $(4, 2, 2, 0, 0)$ instead.
\end{example}

Integer partitions must be written in non-decreasing order.
The following definition relaxes this requirement.

\begin{definition}
	A \textbf{weak composition} $\mu$ of $m$ into $n$ parts is an $n$-length tuple
	$(\mu_1, \mu_2, \dots, \mu_n)$
	of non-negative integers $\mu_i$
	such that
	$\sum_{i = 1}^{m} \mu_i = m$.

	We say that a weak composition $\mu$ is in \textbf{partition order} if it is also
	an integer partition.
\end{definition}

\begin{example}
	$(2, 0, 4, 0, 2)$ is a weak composition of $8$. It is not in partition order.
	However, sorting these entries into non-decreasing order, $(4, 2, 2, 0, 0)$,
	gives a weak composition of $8$ that is in partition order.
\end{example}

\begin{definition}
	We can represent integer partitions and weak compositions as \textbf{Young diagrams}.
	For each tuple $(a_1, a_2, \dots, a_n)$, 
	we create a diagram consisting of $n$ rows, where
	row $i$ consists of $a_i$ boxes 
	(by convention, the row number increases downwards).

	Note that if $a_i$ is $0$ for some $i$, we leave a row's worth of space for it and move onto the next
	value in the tuple.
\end{definition}

\begin{example}
	We represent the integer partition $(4, 2, 2, 0, 0)$ 
	and the weak composition $(2, 0, 4, 0, 2)$ 
	as Young diagrams:
	\begin{equation}
		\begin{aligned}
			\scalebox{0.95}{\input{youngframe.tikz}}
		\end{aligned}
	\end{equation}
\end{example}

\begin{definition}
	Let $\lambda$ be an integer partition of $q$ having exactly $n$ parts,
	where we allow zero entry parts.
	Let $\mu$ be a weak composition of $m$ into $n$ parts.
	We say that $\mu$ \textbf{fits inside} $\lambda$ if
	$\lambda_i \geq \mu_i$ for all $i \in [n]$.
\end{definition}
We can see this diagrammatically by considering the Young diagrams 
for each of $\lambda$ and $\mu$:
if we can shade $\mu$ in $\lambda$ without going outside the boundary of $\lambda$, 
then $\mu$ fits inside $\lambda$, otherwise it does not.
\begin{example}
	If $\lambda$ is the integer partition $(6, 4, 4, 2, 0, 0)$
	and $\mu$ is the weak composition $(2, 3, 0, 2, 0, 0)$,
	then we see that $\mu$ fits inside $\lambda$:
	\begin{equation}
		\begin{aligned}
			\scalebox{0.95}{\input{skewyoungframe.tikz}}
		\end{aligned}
	\end{equation}
	However, the weak composition $(3, 5, 3, 2, 0, 1)$ does not fit inside
	$\lambda$, as indicated by the red boxes in the diagram below, 
	which fall outside the boundary of $\lambda$:
	\begin{equation}
		\begin{aligned}
			\scalebox{0.95}{\input{skewyoungframe2.tikz}}
		\end{aligned}
	\end{equation}
\end{example}
The key insight behind why our procedure generates all $(k,l)$--bipartition
diagrams having at most $n$ blocks can be stated as follows:
every integer partition of $k+l$ having exactly $n$ parts (allowing for zero element parts)
can be expressed as a $(k+l, 0)$--bipartition, and every 
$(k, l)$--bipartition having some $t \in [n]$ blocks must come from a integer partition of $k+l$
having exactly $t$ parts.
In particular, we can then represent each $(k+l, 0)$--bipartition as a
$(k+l, 0)$--bipartition diagram (or as a $(k+l, 0)$--orbit bipartition diagram).

Indeed,
if $\lambda = (\lambda_1, \lambda_2, \dots, \lambda_n)$
is an integer partition of $k + l$ having exactly $n$ parts, with zero element parts allowed, 
then, by first removing all zero element parts to give the integer partition
$\lambda = (\lambda_1, \lambda_2, \dots, \lambda_t)$ of $k + l$
for some $t \in [n]$
such that $\lambda_i$ is positive for all $i \in [t]$,
we obtain the 
$(k+l, 0)$--bipartition
$\lambda = \{[\lambda_1, 0], [\lambda_2, 0], \dots, [\lambda_t, 0]\}$
having precisely $t$ blocks.
(Remark: we remove all zero element parts in $\lambda$ first since the definition
of a bipartition does not allow parts $[x_i, y_i]$ where $x_i = 0 = y_i$.)
We can then represent the $(k+l, 0)$--bipartition $\lambda$ as a one-row 
$(k+l, 0)$--bipartition diagram
\begin{equation} \label{onerowbipart}
	\begin{aligned}
		\scalebox{0.95}{\input{onerowbipartition.tikz}}
	\end{aligned}
\end{equation}
or with open green circles for a $(k+l, 0)$--orbit bipartition diagram.
We retain the order of the integer partition in the diagram; that is,
the number of legs is decreasing from left to right in each spider.

To see why every $(k, l)$--bipartition having $t$ blocks 
must come from a integer partition of $k+l$ having exactly $t$ parts,
let $\pi = \{[x_1, y_1], [x_2, y_2], \dots, [x_t, y_t]\}$
be a $(k, l)$--bipartition having $t$ blocks.
We form a $(k+l, 0)$--bipartition by subtracting, for each part $i$, 
$y_i$ from the second element of the bracket and then adding it to the first element of the bracket.
This gives the $(k+l, 0)$--bipartition
$\{[x_1 + y_1, 0], [x_2 + y_2, 0], \dots, [x_t + y_t, 0]\}$.
Without loss of generality, we can assume that the parts 
in this $(k+l, 0)$--bipartition
have been reordered 
in decreasing size order from left to right.
Hence, this $(k+l, 0)$--bipartition
is precisely the integer partition
$(x_1 + y_1, x_2 + y_2, \dots, x_t + y_t)$ of $k+l$ into exactly $t$ parts.

Hence, to generate all $(k,l)$--bipartitions having at most $n$ blocks, we first need
to generate all integer partitions of $k+l$ having exactly $n$ parts, with zero element
parts allowed,
and express them as $(k+l,0)$--bipartitions
(or as bipartition diagrams).
We then need to work out how to create all possible $(k, l)$--bipartitions (or diagrams)
from each such $(k+l,0)$--bipartition (diagram), which is the reverse procedure 
of the process described in the previous paragraph.

We see that, for a given $(k+l,0)$--bipartition diagram having $t$ blocks, where $t \in [n]$, 
to obtain a $(k,l)$--bipartition diagram, we need to turn $l$ legs up across all 
of the $t$ spiders.
This is the same as creating a weak composition $\mu$ of $l$ into exactly $n$ parts,
where we set the last $n - t$ parts in $\mu$ to $0$.
Hence, calculating all weak compositions of $l$ into exactly $n$ parts  
will be enough to determine all of the $(k,l)$--bipartition diagrams
from a given $(k+l,0)$--bipartition diagram having at most $n$ blocks.

In effect, this produces a matrix where the integer partitions $\lambda$ of $k+l$ into exactly
$n$ parts, with zero element parts allowed, index the rows, 
and the weak compositions $\mu$ of $l$ into exactly $n$ parts index the columns.
We then need to see, for each $(\lambda, \mu)$ pair, whether the elements of $\mu$
are a well-fitting combination for turning up the legs of $\lambda$.
We call such a pair \textbf{valid}.
A simple way to determine if each $(\lambda, \mu)$ pair is valid is to see whether
$\mu$ fits inside $\lambda$. 
If it does, then the pair is valid, otherwise it is not.
Then, for each valid pair $(\lambda, \mu)$, we can create a 
$(k,l)$--bipartition diagram by turning up $\mu_i$ legs 
in each spider $i$ of $\lambda$, viewed as a $(k+l,0)$--bipartition diagram.

Hence, if $\mu = (\mu_1, \mu_2, \dots, \mu_t, 0, \dots, 0)$ is a weak composition of $l$
into exactly $n$ parts such that 
$(\lambda, \mu)$ is a valid pair, 
where $\lambda$ is the $(k+l,0)$--bipartition diagram given in (\ref{onerowbipart}),
then, for each spider $i$ in (\ref{onerowbipart}), we can turn up $\mu_i$ legs to obtain
the $(k,l)$--bipartition diagram 
\begin{equation} \label{tworowbipart}
	\begin{aligned}
		\scalebox{0.95}{\input{tworowbipartition.tikz}}
	\end{aligned}
\end{equation}
At this stage, this would be enough to generate all 
$(k,l)$--bipartition diagrams having at most $n$ blocks.
However, this procedure in its current state would create potentially 
multiple duplicate $(k,l)$--bipartition diagrams
from any $(k+l,0)$--bipartition diagram where spiders have the same number of legs;
that is, this procedure will overgenerate the $(k,l)$--bipartition diagrams
having at most $n$ blocks that we need.
Indeed, if block $i$ and block $j$ have the same number of legs, that is, $\lambda_i = \lambda_j$,
then, if $\mu^1$ is the weak composition with entry $\mu_i$ in position $i$
and entry $\mu_j$ in position $j$ such that $(\lambda, \mu^1)$ is valid,
and $\mu^2$ is the weak composition with entry $\mu_j$ in position $i$
and entry $\mu_i$ in position $j$ such that $(\lambda, \mu^2)$ is valid,
with the entries of $\mu^1$ and $\mu^2$ being equal at all other indices in the tuples,
we see that using $\mu^1$ to turn up the legs of $\lambda$ gives
the same $(k,l)$--bipartition diagram as using $\mu^2$ to turn up the legs of $\lambda$,
even though the spiders might be in a different order.

Hence, we can improve this procedure so that only the exact number of 
$(k,l)$--bipartition diagrams having at most $n$ blocks are generated, in the following way:
after the matrix that is indexed by $(\lambda, \mu)$ pairings is created,
perform what we have termed the $(\lambda, \mu)$ Duplication Test to eliminate
all pairings that fail the test first before checking the rest for validity.

The $(\lambda, \mu)$ Duplication Test does the following: to avoid duplicating
$(k,l)$--bipartition diagrams from a given $(k+l,0)$--bipartition diagram $\lambda$,
for spiders whose legs have the same size, we only need to 
turn up the legs on these spiders in decreasing order,
since any other order will produce duplicate diagrams.
That is, in the subtuple of a weak composition formed by taking those elements of the 
weak composition at the same indices as the blocks of the spiders,
we only need to consider the subtuple that is in partition order.
Hence we can set any $(\lambda, \mu)$ pairing that fails this Duplication Test
to be invalid.

We can also speed up the process of determining whether a $(\lambda, \mu)$ pair is valid or not as follows.
Suppose that $\lambda$ has non-zero entries only in its first $t$ parts.
If $\mu$ contains a non-zero entry in any index from $t+1$ to $n$, then this pair is automatically invalid.
Hence, for a given $\lambda$ having non-zero entries only in its first $t$ parts, 
we only need to consider the validity of $(\lambda, \mu)$ pairings where $\mu$ is
of the form $\mu = (\mu_1, \mu_2, \dots, \mu_t, 0, \dots, 0)$,
with the $\mu_i$ being non-negative integers.
We perform this before the $(\lambda, \mu)$ Duplication Test.

There is one further optimisation that we can make to the entire procedure.
As it stands, we chose to turn $l$ legs up in a $(k+l,0)$--bipartition diagram
and found all possible ways of doing this by considering weak combinations of $l$
into exactly $n$ parts.
In fact, we could choose instead to not turn up $k$ legs in total, which would
be the same as considering all weak combinations of $k$ into exactly $n$ parts.
As the number of weak combinations of a number $m$ into $n$ parts increases as $m$ increases,
by considering the sizes of $k$ and $l$, we can 
reduce the number of weak combinations that we need to consider
for each $(k+l,0)$--bipartition diagram
by choosing the smaller value between $k$ and $l$.

\begin{figure*}[tb!]
	\begin{tcolorbox}[colback=ygreen!10, colframe=ygreen!40, coltitle=black, 
		title={\bfseries 
		Procedure: How to Generate All $(k,l)$--Bipartition Diagrams
		Having at Most $n$ Blocks.},
		fonttitle=\bfseries
		]
	\begin{enumerate}
		\item 
			Calculate all integer partitions $\lambda$ of $k+l$ into at most
			$n$ parts, and express each integer partition as an $n$-length tuple
			whose elements are in decreasing order, allowing zero element parts.
		\item 
			Choose the smaller value between $k$ and $l$.
			Let $m = \min(k,l)$.
			Calculate all weak compositions $\mu$ of $m$ into exactly $n$ parts.
		\item
			Create a matrix where the integer partitions $\lambda$ index the rows
			and the weak compositions $\mu$ index the columns.
		\item
			Put $0$ in any $(\lambda, \mu)$--entry where $\mu$ has a non-zero entry
			in any index that is greater than the number of parts in $\lambda$.
		\item
			Now consider only those $\lambda$ that have parts that are equal.
			For each such $\lambda$, put $0$ in each remaining $(\lambda, \mu)$--entry
			that fails the $(\lambda, \mu)$ Duplication Test.
		\item
			For each remaining $(\lambda, \mu)$--entry, determine
			whether $\mu$ fits inside $\lambda$, viewing each as a Young diagram.
			If it does, put $1$ in that entry, and call the pair 
			$(\lambda, \mu)$ \textbf{valid}.
			Otherwise, put $0$ in that entry.
		\item 
			For each valid $(\lambda, \mu)$ pair, remove all zero element parts 
			of $\lambda$ and express it as a $(k+l,0)$--bipartition diagram.
			Hence, each block $i$ in $\lambda$ is a spider 
			that corresponds to the pair $[\lambda_i, 0]$, viewing $\lambda$
			as a $(k+l,0)$--bipartition.
			Then:
			\begin{itemize}
				\item If $m = l$, then, for each block $i$ in $\lambda$,
					turn $\mu_i$ legs up to the top row.
				\item Otherwise ($m = k$), for each block $i$ in $\lambda$,
					turn $\lambda_i - \mu_i$ legs up to the top row.
			\end{itemize}
	\end{enumerate}	
	Note that this procedure also works for generating 
	$(k,l)$--orbit bipartition diagrams; simply express each $\lambda$ in Step $6$
	as a $(k,l)$--orbit bipartition diagram instead.
	\end{tcolorbox}
  	\label{generationsummaryprocedure}
\end{figure*}


\begin{figure*}[tb!]
	\begin{tcolorbox}[colback=orange!10, colframe=orange!40, coltitle=black, 
		title={\bfseries 
		Procedure: The $(\lambda, \mu)$ Duplication Test.},
		fonttitle=\bfseries
		]
	\begin{enumerate}
		\item 
			Suppose that $\lambda$ has $t$ non-zero parts, where $t \in [n]$.
			Partition $[t]$ into blocks, where two elements $i, j$ are in
			the same block if and only if $\lambda_i = \lambda_j$.
		\item
			For each block in the partition whose size is greater than $1$,
			create a subtuple from $\mu$ by only considering those elements
			in $\mu$ that are indexed by the elements of the block.
		\item
			Consider each subtuple in turn, 
			If a subtuple is not in partition order,
			then we say that $(\lambda, \mu)$ has 
			\textbf{failed the duplication test}.
		\item 
			Otherwise, we say that $(\lambda, \mu)$ has 
			\textbf{passed the duplication test}.
	\end{enumerate}	
	\end{tcolorbox}
  	\label{duplicationsummaryprocedure}
\end{figure*}


\begin{example}
	Suppose that we would like to find all $(3,2)$--bipartition diagrams
	having at most $3$ blocks.
	We follow the procedure that is given in the green- and orange-coloured boxes below.

	\textbf{Step 1:} We calculate all integer partitions $\lambda$ of $3+2 = 5$
	into at most $n = 3$ parts, and express each as a $3$-length tuple whose elements
	are in decreasing order, allowing zero element parts.

	They are:
	$(5,0,0)$,
	$(4,1,0)$,
	$(3,2,0)$,
	$(3,1,1)$, and
	$(2,2,1)$.

	\textbf{Step 2:} As $l = 2 < 3 = k$, we calculate all weak compositions $\mu$
	of $m = \min(3,2) = 2$ into exactly $n = 3$ parts.

	They are:
	$(2,0,0)$,
	$(0,2,0)$,
	$(0,0,2)$,
	$(1,1,0)$,
	$(1,0,1)$, and
	$(0,1,1)$.

	\textbf{Steps 3, 4 and 5:} We create a matrix where 
	the integer partitions $\lambda$ index the rows
	and the weak compositions $\mu$ index the columns.
	We put $0$ in any $(\lambda, \mu)$--entry where $\mu$ has a non-zero entry
	in any index that is greater than the number of parts in $\lambda$.
	We then look to find
	any remaining $(\lambda, \mu)$--entries that fail the Duplication Test.
	For this test,
	we only need to consider those $\lambda$ that have repeating parts:
	they are $(3,1,1)$ and $(2,2,1)$.

	Consider $\lambda = (3,1,1)$ and $\mu = (1,0,1)$.
	We form the partition $\{1 \mid 2, 3\}$ of $[3]$ from $\lambda$,
	and form the subtuple $(0,1)$ of $\mu$ from the only block in the partition
	of size greater than $1$.
	As this subtuple is not in partition order, we see that this pair $(\lambda, \mu)$
	has failed the Duplication Test.

	Likewise, consider $\lambda = (2,2,1)$ and $\mu = (0,1,1)$.
	We form the partition $\{1, 2 \mid 3\}$ of $[3]$ from $\lambda$,
	and form the subtuple $(0,1)$ of $\mu$ from the only block in the partition
	of size greater than $1$.
	As this subtuple is not in partition order, we see that this pair $(\lambda, \mu)$
	has failed the Duplication Test.

	For these two pairs, we put $0$ in the $(\lambda, \mu)$--entry of the matrix.
	All other pairs pass the Duplication Test.

	\textbf{Step 6:} For each remaining $(\lambda, \mu)$--entry, we determine
	whether $\mu$ fits inside $\lambda$, viewing each as a Young diagram.

	At the end of this step, we see that the matrix contains the following entries:

	\begin{equation}
	\NiceMatrixOptions{code-for-first-row = \scriptstyle \color{blue},
                   	   code-for-first-col = \scriptstyle \color{blue}
	}
	\renewcommand{\arraystretch}{1.5}
	\begin{bNiceArray}{*{6}{c}}[first-row,first-col]
			& (2,0,0) & (0,2,0) & (0,0,2) & (1,1,0) & (1,0,1) & (0,1,1) \\
		(5,0,0)		& 1	& 0	& 0	& 0	& 0	& 0	\\
		(4,1,0)		& 1	& 0	& 0 	& 1	& 0	& 0	\\
		(3,2,0)		& 1	& 1	& 0	& 1	& 0	& 0 	\\
		(3,1,1)		& 1	& 0	& 0	& 1	& 0	& 1 	\\
		(2,2,1)		& 1	& 1	& 0	& 1	& 1	& 0 	\\
	\end{bNiceArray}
	\end{equation}
	
	At this point, we can see that $p_3(3,2) = 13$.

	\textbf{Step 7:} We form the $(3,2)$--bipartition diagram that is associated
	with each valid $(\lambda, \mu)$ pair.

	They are:
	\begin{equation}
	\NiceMatrixOptions{code-for-first-row = \scriptstyle \color{blue},
                   	   code-for-first-col = \scriptstyle \color{blue}
	}
	\renewcommand{\arraystretch}{4}
	\begin{bNiceArray}{*{6}{c}}[first-row,first-col]
			& (2,0,0) & (0,2,0) & (0,0,2) & (1,1,0) & (1,0,1) & (0,1,1) \\
		(5,0,0)		& \scalebox{0.65}{\input{diagram31ex1.tikz}}	& 0	& 0	& 0	& 0	& 0	\\
		(4,1,0)		& \scalebox{0.65}{\input{diagram31ex2.tikz}}	& 0	& 0 	& \scalebox{0.65}{\input{diagram31ex8.tikz}}	& 0	& 0	\\
		(3,2,0)		& \scalebox{0.65}{\input{diagram31ex3.tikz}}	& \scalebox{0.65}{\input{diagram31ex6.tikz}}	& 0	& \scalebox{0.65}{\input{diagram31ex9.tikz}}	& 0	& 0 	\\
		(3,1,1)		& \scalebox{0.65}{\input{diagram31ex4.tikz}}	& 0	& 0	& \scalebox{0.65}{\input{diagram31ex10.tikz}}	& 0	& \scalebox{0.65}{\input{diagram31ex13.tikz}} 	\\
		(2,2,1)		& \scalebox{0.65}{\input{diagram31ex5.tikz}}	& \scalebox{0.65}{\input{diagram31ex7.tikz}}	& 0	& \scalebox{0.65}{\input{diagram31ex11.tikz}}	& \scalebox{0.65}{\input{diagram31ex12.tikz}}	& 0 	\\
	\end{bNiceArray}
	\end{equation}
\end{example}

To finish this section, we note that there exists a more generic
partition function, $p(k,l)$, which counts the number of possible $(k,l)$--bipartitions without
any restriction on the number of blocks.
In particular, the sequence $p(k,k)$ (i.e $k = l$) is the OEIS sequence A002774.
Bipartitions without any restriction on the number of blocks
have been studied in many prior works
\citep{macmahon1893, macmahon1896, mathews1896, macmahon1899, macmahon1918, 
auluck1953, nanda1957, wright1957, wright1958, carlitz1963, carlitz1966, andrews1977, andrews1998}.
It is fair to say that 
most of these authors, 
with the exception of MacMahon and Mathews,
primarily focused on 
determining either generating functions or asymptotic formulas for $p(k,l)$. 

We see from our analysis that $p(k,l)$ is equal to $p_{k+l}(k,l)$, 
the number of $(k,l)$--bipartitions having at most $k+l$ parts.
We give the value of $p(k,l)$ up to $k = l = 5$ in Table \ref{p(k,l)values}.
Note that $p(k,l)$ is symmetric in $k$ and $l$, 
and that $p(k,0) = p(0,k)$ is precisely the number of integer partitions of $k$ (into at most $k$ parts).
\begin{table}[h]
    \centering
	\caption{The number of $(k,l)$--bipartitions: $p(k,l)$} 
	\label{p(k,l)values} 
\begin{equation} 
	\NiceMatrixOptions{code-for-first-row = \color{blue},
                   	   code-for-first-col = \color{blue}
	}
	\renewcommand{\arraystretch}{1.5} 
	\setlength{\tabcolsep}{15pt} 
	\begin{NiceArray}{|*{6}{c}}[first-row,first-col]
		\diagbox{\textcolor{red}{k}}{\textcolor{red}{l}} 	& 0	& 1 	& 2	& 3	& 4 	& 5	\\
		\cline{0-6}
		0		& 1 	& 1 	& 2	& 3	& 5 	& 7  	\\
		1		& 1	& 2 	& 4	& 7	& 12 	& 19  	\\
		2		& 2	& 4 	& 9	& 16	& 29 	& 47  	\\
		3		& 3	& 7 	& 16	& 31	& 57 	& 97  	\\
		4		& 5	& 12 	& 29	& 57	& 109	& 189  	\\
		5		& 7	& 19 	& 47	& 97	& 189 	& 339 	
	\end{NiceArray}
\end{equation}
\end{table}

\section{Supplementary Weight Matrix Calculations}

\begin{example} [Calculations for Example \ref{orbitexample}]
	\label{orbitexampleapp}
	Recall that
	we would like to find
	the $S_3$-equivariant weight matrix
	from
	$S^2(\mathbb{R}^{3})$
	to
	$S^1(\mathbb{R}^{3}) = \mathbb{R}^{3}$
	using the orbit basis of $SP_2^1(3)$.

	We need to consider the $(2,1)$--orbit bipartition diagrams 
	having at most $n = 3$ blocks.
	They are 
	\begin{equation} \label{orbitexdiagapp}
		\begin{aligned}
			\scalebox{0.7}{\input{orbit21diags.tikz}}
		\end{aligned}
	\end{equation}
	For example, for the second diagram which has two blocks, we see that, since $n=3$,
	there are only six possible $2$-length tuples having distinct elements in $[3]$:
	\begin{equation}
		(1,2);
		(1,3);
		(2,1);
		(2,3);
		(3,1);
		(3,2)
	\end{equation}
	Assigning these tuples to the central green nodes of the diagram and reordering the spiders where appropriate gives
	\begin{equation}	
		\begin{aligned}
			\scalebox{0.7}{\input{orbit21ex2ex.tikz}}
		\end{aligned}
	\end{equation}
	By propagating the central values to the ends of the wires in each diagram,
	then reading the $I$ tuple off the top row and the $J$ tuple off the bottom row
	to form the matrix unit $E_{I,J}$ for each diagram, and then adding these matrix units together
	gives the following orbit basis matrix $X_\pi$ corresponding to the second diagram
	in (\ref{orbitexdiagapp}):

	\begin{equation}
	\NiceMatrixOptions{code-for-first-row = \scriptstyle \color{blue},
                   	   code-for-first-col = \scriptstyle \color{blue}
	}
	\renewcommand{\arraystretch}{1}
	\begin{bNiceArray}{*{6}{c}}[first-row,first-col]
				& 1,1 	& 1,2	& 1,3	& 2,2 	& 2,3	& 3,3	\\
		1		& 0	& 0	& 0	& 1	& 0	& 1	\\
		2		& 1	& 0	& 0 	& 0	& 0	& 1	\\
		3		& 1	& 0	& 0	& 1	& 0	& 0 
	\end{bNiceArray}
	\end{equation}
	
	Hence, in full, from left to right, the four diagrams in (\ref{orbitexdiagapp}) correspond to the following basis matrices $X_\pi$
	of size $3 \times 6$:

	\begin{equation} 
	\NiceMatrixOptions{code-for-first-row = \scriptstyle \color{blue},
                   	   code-for-first-col = \scriptstyle \color{blue}
	}
	\renewcommand{\arraystretch}{1}
	\begin{bNiceArray}{*{6}{c}}[first-row,first-col]
				& 1,1 	& 1,2	& 1,3	& 2,2 	& 2,3	& 3,3	\\
		1		& 1	& 0	& 0	& 0	& 0	& 0	\\
		2		& 0	& 0	& 0 	& 1	& 0	& 0	\\
		3		& 0	& 0	& 0	& 0	& 0	& 1 
	\end{bNiceArray}
	\;
	;
	\;
	\NiceMatrixOptions{code-for-first-row = \scriptstyle \color{blue},
                   	   code-for-first-col = \scriptstyle \color{blue}
	}
	\renewcommand{\arraystretch}{1}
	\begin{bNiceArray}{*{6}{c}}[first-row,first-col]
				& 1,1 	& 1,2	& 1,3	& 2,2 	& 2,3	& 3,3	\\
		1		& 0	& 0	& 0	& 1	& 0	& 1	\\
		2		& 1	& 0	& 0 	& 0	& 0	& 1	\\
		3		& 1	& 0	& 0	& 1	& 0	& 0 
	\end{bNiceArray}
	\end{equation}
	
	\begin{equation}
	\NiceMatrixOptions{code-for-first-row = \scriptstyle \color{blue},
                   	   code-for-first-col = \scriptstyle \color{blue}
	}
	\renewcommand{\arraystretch}{1}
	\begin{bNiceArray}{*{6}{c}}[first-row,first-col]
				& 1,1 	& 1,2	& 1,3	& 2,2 	& 2,3	& 3,3	\\
		1		& 0	& 1	& 1	& 0	& 0	& 0	\\
		2		& 0	& 1	& 0 	& 0	& 1	& 0	\\
		3		& 0	& 0	& 1	& 0	& 1	& 0 
	\end{bNiceArray}
	\;
	;
	\; 
	\NiceMatrixOptions{code-for-first-row = \scriptstyle \color{blue},
                   	   code-for-first-col = \scriptstyle \color{blue}
	}
	\renewcommand{\arraystretch}{1}
	\begin{bNiceArray}{*{6}{c}}[first-row,first-col]
				& 1,1 	& 1,2	& 1,3	& 2,2 	& 2,3	& 3,3	\\
		1		& 0	& 0	& 0	& 0	& 1	& 0	\\
		2		& 0	& 0	& 1 	& 0	& 0	& 0	\\
		3		& 0	& 1	& 0	& 0	& 0	& 0 
	\end{bNiceArray}
	\end{equation}

	Hence, the $S_3$-equivariant weight matrix
	from
	$S^2(\mathbb{R}^{3})$
	to
	$S^1(\mathbb{R}^{3})$
	using the orbit basis of $SP_2^1(3)$
	is of the form 
\begin{equation}
	\NiceMatrixOptions{code-for-first-row = \scriptstyle \color{blue},
                   	   code-for-first-col = \scriptstyle \color{blue}
	}
	\renewcommand{\arraystretch}{1}
	\begin{bNiceArray}{*{6}{c}}[first-row,first-col]
				& 1,1 	& 1,2	& 1,3	& 2,2 	& 2,3	& 3,3	\\
		1		& \lambda_1	& \lambda_3	& \lambda_3	& \lambda_2	& \lambda_4	& \lambda_2	\\
		2		& \lambda_2	& \lambda_3	& \lambda_4	& \lambda_1	& \lambda_3	& \lambda_2	\\
		3		& \lambda_2	& \lambda_4	& \lambda_3	& \lambda_2	& \lambda_3	& \lambda_1
	\end{bNiceArray}
\end{equation}
for weights $\lambda_1, \lambda_2, \lambda_3, \lambda_4 \in \mathbb{R}$.
\end{example}

\begin{example} [Calculations for Example \ref{diagramexample}]
	\label{diagramexampleapp}
	Recall that
	we would like to find
	the $S_3$-equivariant weight matrix
	from
	$S^2(\mathbb{R}^{3})$
	to
	$S^1(\mathbb{R}^{3}) = \mathbb{R}^{3}$ 
	using the diagram basis of $SP_2^1(3)$ instead.

	We need to consider the $(2,1)$--bipartition diagrams 
	having at most $n = 3$ blocks.
	They are 
	\begin{equation} \label{diagramexdiagapp}
		\begin{aligned}
			\scalebox{0.7}{\input{diagram21diags.tikz}}
		\end{aligned}
	\end{equation}
	For example, for the third diagram which has two blocks, we see that, since $n=3$,
	there are now nine possible $2$-length tuples whose elements are in $[3]$, since
	the elements in each tuple no longer need to be distinct:	
	\begin{equation}
		(1,1);
		(1,2);
		(1,3);
		(2,1);
		(2,2);
		(2,3);
		(3,1);
		(3,2);
		(3,3)
	\end{equation}
	Assigning these tuples to the central green nodes of the diagram,
	reordering the spiders where appropriate, 
	and fusing together any spiders whose central green nodes 
	are labelled with the same value
	gives
	\begin{equation}	
		\begin{aligned}
			\scalebox{0.7}{\input{diagram21ex3exa.tikz}}
		\end{aligned}
	\end{equation}
	and
	\begin{equation}	
		\begin{aligned}
			\scalebox{0.7}{\input{diagram21ex3exb.tikz}}
		\end{aligned}
	\end{equation}
	By propagating the central values to the ends of the wires in each diagram,
	then reading the $I$ tuple off the top row and the $J$ tuple off the bottom row
	to form the matrix unit $E_{I,J}$ for each diagram, and then adding these matrix units together
	gives the following diagram basis matrix $D_\pi$ corresponding to the 
	third diagram
	in (\ref{diagramexdiagapp}):

	\begin{equation}
	\NiceMatrixOptions{code-for-first-row = \scriptstyle \color{blue},
                   	   code-for-first-col = \scriptstyle \color{blue}
	}
	\renewcommand{\arraystretch}{1}
	\begin{bNiceArray}{*{6}{c}}[first-row,first-col]
				& 1,1 	& 1,2	& 1,3	& 2,2 	& 2,3	& 3,3	\\
		1		& 1	& 1	& 1	& 0	& 0	& 0	\\
		2		& 0	& 1	& 0 	& 1	& 1	& 0	\\
		3		& 0	& 0	& 1	& 0	& 1	& 1 
	\end{bNiceArray}
	\end{equation}

	Hence, in full, from left to right, the four diagrams in (\ref{orbitexdiagapp}) correspond to the following basis matrices $D_\pi$
	of size $3 \times 6$:

	\begin{equation} 
	\NiceMatrixOptions{code-for-first-row = \scriptstyle \color{blue},
                   	   code-for-first-col = \scriptstyle \color{blue}
	}
	\renewcommand{\arraystretch}{1}
	\begin{bNiceArray}{*{6}{c}}[first-row,first-col]
				& 1,1 	& 1,2	& 1,3	& 2,2 	& 2,3	& 3,3	\\
		1		& 1	& 0	& 0	& 0	& 0	& 0	\\
		2		& 0	& 0	& 0 	& 1	& 0	& 0	\\
		3		& 0	& 0	& 0	& 0	& 0	& 1 
	\end{bNiceArray}
	\;
	;
	\;
	\NiceMatrixOptions{code-for-first-row = \scriptstyle \color{blue},
                   	   code-for-first-col = \scriptstyle \color{blue}
	}
	\renewcommand{\arraystretch}{1}
	\begin{bNiceArray}{*{6}{c}}[first-row,first-col]
				& 1,1 	& 1,2	& 1,3	& 2,2 	& 2,3	& 3,3	\\
		1		& 1	& 0	& 0	& 1	& 0	& 1	\\
		2		& 1	& 0	& 0 	& 1	& 0	& 1	\\
		3		& 1	& 0	& 0	& 1	& 0	& 1 
	\end{bNiceArray}
	\end{equation}
	
	\begin{equation}
	\NiceMatrixOptions{code-for-first-row = \scriptstyle \color{blue},
                   	   code-for-first-col = \scriptstyle \color{blue}
	}
	\renewcommand{\arraystretch}{1}
	\begin{bNiceArray}{*{6}{c}}[first-row,first-col]
				& 1,1 	& 1,2	& 1,3	& 2,2 	& 2,3	& 3,3	\\
		1		& 1	& 1	& 1	& 0	& 0	& 0	\\
		2		& 0	& 1	& 0 	& 1	& 1	& 0	\\
		3		& 0	& 0	& 1	& 0	& 1	& 1 
	\end{bNiceArray}
	\;
	;
	\; 
	\NiceMatrixOptions{code-for-first-row = \scriptstyle \color{blue},
                   	   code-for-first-col = \scriptstyle \color{blue}
	}
	\renewcommand{\arraystretch}{1}
	\begin{bNiceArray}{*{6}{c}}[first-row,first-col]
				& 1,1 	& 1,2	& 1,3	& 2,2 	& 2,3	& 3,3	\\
		1		& 1	& 1	& 1	& 1	& 1	& 1	\\
		2		& 1	& 1	& 1 	& 1	& 1	& 1	\\
		3		& 1	& 1	& 1	& 1	& 1	& 1 
	\end{bNiceArray}
	\end{equation}

	Hence, the $S_3$-equivariant weight matrix
	from
	$S^2(\mathbb{R}^{3})$
	to
	$S^1(\mathbb{R}^{3})$
	using the diagram basis of $SP_2^1(3)$
	is of the form 
\begin{equation}
	\NiceMatrixOptions{code-for-first-row = \scriptstyle \color{blue},
                   	   code-for-first-col = \scriptstyle \color{blue}
	}
	\renewcommand{\arraystretch}{1}
	\begin{bNiceArray}{*{6}{c}}[first-row,first-col]
				& 1,1 	& 1,2	& 1,3	& 2,2 	& 2,3	& 3,3	\\
		1		& \lambda_{1,2,3,4}	& \lambda_{3,4}	& \lambda_{3,4}	& \lambda_{2,4}	& \lambda_4	& \lambda_{2,4}	\\
		2		& \lambda_{2,4}	& \lambda_{3,4}	& \lambda_4	& \lambda_{1,2,3,4}	& \lambda_{3,4}	& \lambda_{2,4}	\\
		3		& \lambda_{2,4}	& \lambda_4	& \lambda_{3,4}	& \lambda_{2,4}	& \lambda_{3,4}	& \lambda_{1,2,3,4}
	\end{bNiceArray}
\end{equation}
for weights $\lambda_1, \lambda_2, \lambda_3, \lambda_4 \in \mathbb{R}$,
	where $\lambda_A \coloneqq \sum_{i \in A} \lambda_i$.
\end{example}

\begin{example} [Deep Sets $S_n$-Equivariant Weight Matrix 
	from 
	$\mathbb{R}^{n}$
	to
	$\mathbb{R}^{n}$]
	\label{deepsetsexample}
	To find the Deep Sets $S_n$-equivariant weight matrix from
	from 
	$\mathbb{R}^{n}$
	to
	$\mathbb{R}^{n}$,
	we need to find the diagram basis of 
	$\Hom_{S_n}(S^1(\mathbb{R}^{n}), S^1(\mathbb{R}^{n}))$
	from the diagram basis of $SP_1^1(n)$,
	since $S^1(\mathbb{R}^{n}) = \mathbb{R}^{n}$.
	
	We need to consider all $(1,1)$--bipartition diagrams 
	having at most $n$ blocks.
	Assuming that $n > 1$, they are 
	\begin{equation} \label{deepsetsexdiag}
		\begin{aligned}
			\scalebox{0.7}{\input{diagram11diags.tikz}}
		\end{aligned}
	\end{equation}
	For the first diagram, the only valid $1$-length tuples are the elements of $[n]$.
	Hence, the matrix that corresponds to this $(1,1)$--bipartition diagram is 
	the $n \times n$ identity matrix. 
	For the second diagram, all possible $2$-length tuples with elements from $[n]$ are
	valid, since the elements in each tuple do not need to be distinct.
	Hence, the matrix that corresponds to this $(1,1)$--bipartition diagram is 
	the $n \times n$ all ones matrix.

	Hence, the $(i,j)$--entry of the $S_n$-equivariant weight matrix from 
	$\mathbb{R}^{n}$
	to
	$\mathbb{R}^{n}$
	is $\lambda_1 + \lambda_2$ if $i = j$, and is $\lambda_2$ if $i \neq j$,
	for weights $\lambda_1, \lambda_2 \in \mathbb{R}$.
	This is precisely the Deep Sets characterisation given in \citet{zaheer2017}.
\end{example}


\section{Map Label Implementation and Examples} \label{maplabelimpl}

At the end of the paper, we present complete implementations of the five subprocedures (I to V)
that appear in the procedure for generating the transformation map labels that correspond to 
a $(k,l)$--bipartition diagram for a given value of $n$.
We provide illustrative examples of each of the subprocedures below.

\subsection{Subprocedure I}

\begin{example} \label{subI32ex}

	Consider the $(3,2)$--bipartition diagram $d_{\pi}$ given by
\begin{equation} \label{ProcIex10}
	\begin{aligned}
		\scalebox{0.7}{\input{ProcIex10.tikz}}
	\end{aligned}
\end{equation}
	We follow Subprocedure I to obtain all possible grouped outputs for $d_{\pi}$
	and, for each one, its associated set of partially labelled diagrams.
	(In this case, the resulting diagrams will also be fully labelled, that is, there would be
	no need to apply Subprocedure II to the outcome).

	\textbf{Step 1}: We sort the spiders in decreasing order of output wires, breaking
	ties by sorting them in decreasing order of input wires. This gives
\begin{equation} \label{ProcIex11}
	\begin{aligned}
		\scalebox{0.7}{\input{ProcIex11.tikz}}
	\end{aligned}
\end{equation}

	\textbf{Step 2}: We partition the spiders into groupings where spiders are in the
	same grouping if and only if they have the same number of output wires.
	Labelling each spider with $O_x$, where $x$ denotes the number of output wires that
	the spider has, we see that,
	for the diagram given in (\ref{ProcIex11}), 
	the partition $\pi$ is $\{O_1, O_1\}$, which consists of only one grouping, $\pi_{B} = \{O_1, O_1\}$.

	\textbf{Step 3}: For $\pi_{B}$, we form a new partition of the spiders such
	that spiders are in the same grouping in the new partition $\pi_{B,I}$ if and
	only if they have the same number of input wires. 
	Labelling each spider in $\pi_{B}$ with $I_x$, where $x$ denotes the number of input wires that
	the spider has, we see that $\pi_{B,I}$ is $\{I_2 \mid I_1\}$.

	\textbf{Step 4}: We calculate all integer partitions of $t$,
	where $t$ is the number of spiders in $d_{\pi}$ that have output wires.
	Here $t=2$, since both spiders in (\ref{ProcIex11}) have output wires. 
	The integer partitions of $2$ are $1 + 1$ and $2$ only.

	\textbf{Step 5}: Consider each integer partition $\lambda$ in turn.

	\textbf{a)}: For $\lambda = 1 + 1$, we form the multiset $\{i,j\}$ and compute
	all partitions of this multiset that correspond to the original partition
	structure of $\pi$, modulo isomorphism. 
	There is only one non-isomorphic partition of the multiset, namely $L = \{i,j\}$ itself,
	which has only one grouping, $L_B = \{i,j\}$.

	We now compute all partitions of the grouping $L_B$ that correspond to the partition
	structure of $\pi_{B,I}$.
	This gives $\{i \mid j\}$ and $\{j \mid i\}$, which we convert into ordered tuples,
	that is, $T_{L, B} = \{(i,j), (j,i)\}$.

	Since $L$ had only one grouping, we see that $T_{1+1, L} = \{(i,j), (j,i)\}$.
	Hence the set $D_{1+1,L}$ consists only of the following two diagrams
\begin{equation} \label{ProcIex1213}
	\begin{aligned}
		\scalebox{0.7}{\input{ProcIex12.tikz}}
	\end{aligned}
	\quad\quad\text{ and }\quad\quad
	\begin{aligned}
		\scalebox{0.7}{\input{ProcIex13.tikz}}
	\end{aligned}
\end{equation}

	where, in each diagram, we are unable to merge any spiders.
	Hence $R_{1+1,L} = \varnothing$, and so $P_{1+1,L} = D_{1+1,L}$.

	Moreover, since there was only one non-isomorphic partition of the multiset $\{i,j\}$,
	we see that $P_{1+1} = P_{1+1,L}$.

	\textbf{b)}: For $\lambda = 2$, we form the multiset $\{i,i\}$ and compute
	all partitions of this multiset that correspond to the original partition
	structure of $\pi$, modulo isomorphism. 
	There is only one non-isomorphic partition of the multiset, namely $L = \{i,i\}$ itself,
	which has only one grouping, $L_B = \{i,i\}$.

	We now compute all partitions of the grouping $L_B$ that correspond to the partition
	structure of $\pi_{B,I}$.
	This gives $\{i \mid i\}$ only, which we convert into an ordered tuple,
	that is, $T_{L, B} = \{(i,i)\}$.

	Since $L$ had only one grouping, we see that $T_{2, L} = \{(i,i)\}$.
	Hence the set $D_{2,L}$ consists only of the following diagram

\begin{equation} \label{ProcIex14}
	\begin{aligned}
		\scalebox{0.7}{\input{ProcIex14.tikz}}
	\end{aligned}
\end{equation}

	We merge spiders that are labelled with the same label to form:

\begin{equation} \label{ProcIex15}
	\begin{aligned}
		\scalebox{0.7}{\input{ProcIex15.tikz}}
	\end{aligned}
\end{equation}

	It is clear that, despite the new grouping, no new partially labelled diagrams
	can appear from (\ref{ProcIex15}), and so we get that $R_{2,L} = \varnothing$.
	Hence $P_{2,L} = D_{2,L}$.

	Moreover, since there was only one non-isomorphic partition of the multiset $\{i,i\}$,
	we see that $P_{2} = P_{2,L}$.

	\textbf{Step 6}: We now take the set union of $P_{\lambda}$ over all integer partitions
	$\lambda$ of $t=2$, and call the new set $P$.
	Hence $P$ has three elements, namely the two diagrams given in (\ref{ProcIex1213})
	and the one diagram given in (\ref{ProcIex15}).
	By reordering the labelled spiders in lexicographical order, 
	and partitioning $P$ into subsets such that diagrams are in the same subset 
	if and only if their tuples of output labels are the same, 
	we obtain the following two sets as output of Subprocedure I for the $(3,2)$--bipartition diagram $d_{\pi}$:
\begin{equation} \label{exD1result}
	\begin{aligned}
		(i,j)
		=
		\left\{
			\;
		\scalebox{0.7}{\input{ProcIex12.tikz}},
		\;
		\scalebox{0.7}{\input{ProcIex13a.tikz}}
			\;
		\right\}
	\end{aligned}
		\quad\quad\text{ and }\quad\quad
	\begin{aligned}
		(i,i)
		=
		\left\{
			\;
		\scalebox{0.7}{\input{ProcIex15.tikz}}
		\;
		\right\}
	\end{aligned}
\end{equation}
\end{example}


\begin{example}
	Consider instead the $(1,4)$--bipartition diagram $d_{\pi}$ given by
\begin{equation} \label{ProcIex1}
	\begin{aligned}
		\scalebox{0.7}{\input{ProcIex1.tikz}}
	\end{aligned}
\end{equation}
	We follow Subprocedure I to obtain all possible grouped outputs for $d_{\pi}$
	and, for each one, its associated set of partially labelled diagrams.
	(Once again, the resulting diagrams will also be fully labelled, that is, there would be
	no need to apply Subprocedure II to the outcome).

	\textbf{Step 1}: We sort the spiders in decreasing order of output wires, breaking
	ties by sorting them in decreasing order of input wires. We see that the spiders
	in (\ref{ProcIex1}) are already sorted.

	\textbf{Step 2}: We partition the spiders into groupings where spiders are in the
	same grouping if and only if they have the same number of output wires.
	Labelling each spider with $O_x$, where $x$ denotes the number of output wires that
	the spider has, we see that,
	for the diagram given in (\ref{ProcIex1}), 
	the partition $\pi$ is $\{O_2 \mid O_1, O_1\}$, which consists of two groupings, 
	$\pi_{B_1} = \{O_2\}$ and $\pi_{B_2} = \{O_1, O_1\}$.

	\textbf{Step 3}: For each grouping $\pi_{B_i}$, we form a new partition of the spiders such
	that spiders are in the same grouping in the new partition $\pi_{B_i,I}$ if and
	only if they have the same number of input wires. 
	Labelling each spider in $\pi_{B_i}$ with $I_x$, where $x$ denotes the number of input wires that
	the spider has, we see that 
	$\pi_{B_1,I} = \{I_1\}$ and
	$\pi_{B_2,I} = \{I_0, I_0\}$.
	
	\textbf{Step 4}: We calculate all integer partitions of $t$,
	where $t$ is the number of spiders in $d_{\pi}$ that have output wires.
	Here $t=3$, since all three spiders in (\ref{ProcIex1}) have output wires. 
	The integer partitions of $3$ are $1 + 1 + 1$, $2 + 1$ and $3$.

	\textbf{Step 5}: Consider each integer partition $\lambda$ in turn.

	\textbf{a)}: For $\lambda = 1 + 1 + 1$, we form the multiset $\{i,j,k\}$ and compute
	all partitions of this multiset that correspond to the original partition
	structure of $\pi$, modulo isomorphism. 
	Despite there being three partitions, namely
	$\{i \mid j, k\}$,
	$\{j \mid i, k\}$, and
	$\{k \mid i, j\}$,
	all three are equivalent under a relabelling of the labels.
	Hence, there is only one non-isomorphic partition of the multiset, namely $L = \{i \mid j, k\}$,
	which has two groupings: $L_{B_1} = \{i\}$ and $L_{B_2} = \{j,k\}$.

	For each grouping $L_{B_i}$, we now compute all partitions of the grouping $L_{B_i}$ that correspond to the partition
	structure of $\pi_{B_i,I}$.
	For $L_{B_1}$, this gives $\{i\}$ only, and for $L_{B_2}$, this gives $\{j,k\}$ only.
	Converting these into ordered tuples, we get that
	$T_{L, B_1} = \{(i)\}$ and $T_{L, B_2} = \{(j,k)\}$.
	Hence $T_{1+1+1, L}$ is $\{(i,j,k)\}$.
	Consequently, the set $D_{1+1+1,L}$ consists only of the diagram
\begin{equation} \label{ProcIex2}
	\begin{aligned}
		\scalebox{0.7}{\input{ProcIex2.tikz}}
	\end{aligned}
\end{equation}

	Since we are unable to merge any spiders, we have that $R_{1+1+1,L} = \varnothing$,
	and so $P_{1+1+1,L} = D_{1+1+1,L}$.

	Moreover, since there was only one non-isomorphic partition of the multiset $\{i,j,k\}$,
	we see that $P_{1+1+1} = P_{1+1+1,L}$.

	\textbf{b)}: For $\lambda = 2 + 1$, we form the multiset $\{i,i,j\}$ and compute
	all partitions of this multiset that correspond to the original partition
	structure of $\pi$, modulo isomorphism. 
	There are two non-isomorphic partitions of the multiset,
	namely $L^1 = \{i \mid i, j\}$
	and $L^2 = \{j \mid i, i\}$.
	$L^1$ has two groupings, $L_{B_1}^1 = \{i\}$ and $L_{B_2}^1 = \{i,j\}$,
	and 
	$L^2$ has two groupings, $L_{B_1}^2 = \{j\}$ and $L_{B_2}^2 = \{i,i\}$.

	i): Consider $L^1$.
	For each grouping $L_{B_i}^1$, we now compute all partitions of the grouping $L_{B_i}^1$ that correspond to the partition
	structure of $\pi_{B_i,I}$.
	For $L_{B_1}^1$, this gives $\{i\}$ only, and for $L_{B_2}^1$, this gives $\{i,j\}$ only.
	Converting these into ordered tuples, we get that
	$T_{L^1, B_1} = \{(i)\}$ and $T_{L^1, B_2} = \{(i,j)\}$.
	Hence $T_{2+1, L^1}$ is $\{(i,i,j)\}$.
	Consequently, the set $D_{2+1,L^1}$ consists only of the diagram
\begin{equation} \label{ProcIex3}
	\begin{aligned}
		\scalebox{0.7}{\input{ProcIex3.tikz}}
	\end{aligned}
\end{equation}
	Merging spiders that are labelled with the same label gives
\begin{equation} \label{ProcIex5}
	\begin{aligned}
		\scalebox{0.7}{\input{ProcIex5.tikz}}
	\end{aligned}
\end{equation}
	It is clear that, despite the new grouping, no new partially labelled diagrams
	can appear from (\ref{ProcIex5}), since there is only one spider with three output wires
	and one spider with one output wire.
	Hence, we get that $R_{2+1,L^1} = \varnothing$, and so $P_{2+1,L^1} = D_{2+1,L^1}$.

	ii): Now consider $L^2$.
	For each grouping $L_{B_i}^2$, we now compute all partitions of the grouping $L_{B_i}^2$ that correspond to the partition
	structure of $\pi_{B_i,I}$.
	For $L_{B_1}^2$, this gives $\{j\}$ only, and for $L_{B_2}^2$, this gives $\{i,i\}$ only.
	Converting these into ordered tuples, we get that
	$T_{L^2, B_1} = \{(j)\}$ and $T_{L^2, B_2} = \{(i,i)\}$.
	Hence $T_{2+1, L^2}$ is $\{(j,i,i)\}$.
	Consequently, the set $D_{2+1,L^2}$ consists only of the diagram
\begin{equation} \label{ProcIex4}
	\begin{aligned}
		\scalebox{0.7}{\input{ProcIex4.tikz}}
	\end{aligned}
\end{equation}
	Merging spiders that are labelled with the same label gives
\begin{equation} \label{ProcIex6}
	\begin{aligned}
		\scalebox{0.7}{\input{ProcIex6.tikz}}
	\end{aligned}
\end{equation}
	However, this time, with the new grouping, we will form new partially labelled diagrams from
	(\ref{ProcIex6}). 
	Let the newly merged bipartition diagram in (\ref{ProcIex6}) be $d_{\theta}$. 
	We retain but ignore its labels for now.
	We see that $d_{\theta}$ is already sorted in decreasing order on output wires, with ties broken by input wires,
	and so we partition the spiders into groupings based on their output wires. Hence $\theta = \{O_2, O_2\}$.
	Moreover, for the single individual grouping in $\theta$, $\theta_B$, we form a new partition of $\theta_B$
	such that spiders are in the same grouping if and only if they have the same number of input wires.
	Hence we have that $\theta_{B,I} = \{I_1 \mid I_0\}$.
	Removing the labels from $\theta$, we obtain a partition of labels $\tilde{L}$ into individual groupings $\tilde{L}_B$.
	Hence $\tilde{L}_B = \{j,i\}$.
	We now compute all partitions of $\tilde{L}_B$ that correspond to the partition structure of $\theta_{B,I}$.
	This gives $\{j \mid i\}$ and $\{i \mid j\}$, which we convert into ordered tuples.
	Hence $T_{\tilde{L}, B} = \{(j,i), (i,j)\}$, and since there is only one grouping in $\tilde{L}$, we have that
	$T_{1+1, \tilde{L}} =  \{(j,i), (i,j)\}$.
	Creating a copy of (\ref{ProcIex6}) without its labels for each element $l$ of $T_{1+1, \tilde{L}}$, we label
	it with $l$. This gives back (\ref{ProcIex6}) and a new diagram
\begin{equation} \label{ProcIex7}
	\begin{aligned}
		\scalebox{0.7}{\input{ProcIex7.tikz}}
	\end{aligned}
\end{equation}
	both of which we add to $R_{2+1, L^2}$.

	Taking the set union of $D_{2+1,L^2}$ and $R_{2+1, L^2}$ gives $P_{2+1, L^2}$, which in this case is equal to $R_{2+1, L^2}$.

	As there were two non-isomorphic partitions of the multiset $\{i,i,j\}$, we see that $P_{2+1} = P_{2+1, L^1} \cup P_{2+1, L^2}$.

	\textbf{c)}: For $\lambda = 3$, we form the multiset $\{i,i,i\}$ and compute
	all partitions of this multiset that correspond to the original partition
	structure of $\pi$, modulo isomorphism. 
	Hence, there is only one non-isomorphic partition of the multiset, namely $L = \{i \mid i, i\}$,
	which has two groupings: $L_{B_1} = \{i\}$ and $L_{B_2} = \{i,i\}$.

	For each grouping $L_{B_i}$, we now compute all partitions of the grouping $L_{B_i}$ that correspond to the partition
	structure of $\pi_{B_i,I}$.
	For $L_{B_1}$, this gives $\{i\}$ only, and for $L_{B_2}$, this gives $\{i,i\}$ only.
	Converting these into ordered tuples, we get that
	$T_{L, B_1} = \{(i)\}$ and $T_{L, B_2} = \{(i,i)\}$.
	Hence $T_{3, L}$ is $\{(i,i,i)\}$.
	Consequently, the set $D_{3,L}$ consists only of the diagram
\begin{equation} \label{ProcIex8}
	\begin{aligned}
		\scalebox{0.7}{\input{ProcIex8.tikz}}
	\end{aligned}
\end{equation}
	Merging spiders that are labelled with the same label gives
\begin{equation} \label{ProcIex9}
	\begin{aligned}
		\scalebox{0.7}{\input{ProcIex9.tikz}}
	\end{aligned}
\end{equation}
	It is clear that, despite the new grouping, no new partially labelled diagrams
	can appear from (\ref{ProcIex9}), since there is only one spider with four output wires.
	Hence, we get that $R_{3,L} = \varnothing$, and so $P_{3,L} = D_{3,L}$.

	Moreover, since there was only one non-isomorphic partition of the multiset $\{i,i,i\}$,
	we see that $P_{3} = P_{3,L}$.

	\textbf{Step 6}: We now take the set union of $P_{\lambda}$ over all integer partitions
	$\lambda$ of $t=3$, and call the new set $P$.
	Hence $P$ has five elements, namely the diagrams given in 
	(\ref{ProcIex2}),
	(\ref{ProcIex5}),
	(\ref{ProcIex6}),
	(\ref{ProcIex7}), and
	(\ref{ProcIex9}).
	By reordering the labelled spiders in lexicographical order, 
	and partitioning $P$ into subsets such that diagrams are in the same subset 
	if and only if their tuples of output labels are the same, 
	we obtain the following four sets as output of Subprocedure I for the $(1,4)$--bipartition diagram $d_{\pi}$:
	\begin{equation} \label{exD2resulti}
	\begin{aligned}
		(i,i,j,k)
		=
		\left\{
			\;
		\scalebox{0.7}{\input{ProcIex2.tikz}}
			\;
		\right\}
	\end{aligned}
	\end{equation}
	\begin{equation} \label{exD2resultii}
	\begin{aligned}
		(i,i,i,j)
		=
		\left\{
			\;
		\scalebox{0.7}{\input{ProcIex5.tikz}}
		\;
		\right\}
	\end{aligned}
	\end{equation}
	\begin{equation} \label{exD2resultiii}
	\begin{aligned}
		(i,i,j,j)
		=
		\left\{
			\;
		\scalebox{0.7}{\input{ProcIex7.tikz}},
			\;
		\scalebox{0.7}{\input{ProcIex6a.tikz}}
		\;
		\right\}
	\end{aligned}
	\end{equation}
	\begin{equation} \label{exD2resultiv}
	\begin{aligned}
		(i,i,i,i)
		=
		\left\{
			\;
		\scalebox{0.7}{\input{ProcIex9.tikz}}
		\;
		\right\}
	\end{aligned}
	\end{equation}
\end{example}

\begin{figure*}[tb!]
	\begin{tcolorbox}[colback=sblue!10, colframe=sblue!40, coltitle=black, 
		title={\bfseries 
		Subprocedure I: 
		How to Obtain All Possible Grouped Outputs of a $(k,l)$--Bipartition Diagram.
		},
		fonttitle=\bfseries
		]
		\textbf{Input:} A $(k,l)$--bipartition diagram $d_{\pi}$.
	\begin{enumerate}
		\item \label{step:sortspiders}
			Sort the spiders in decreasing order of output wires. 
			For spiders that have the same number of output wires, 
			break ties by sorting them in decreasing order of input wires.
			For now, ignore in the following all spiders that do not have output wires.
		\item \label{step:firstpartition}
			Partition the spiders having output wires into groupings where spiders are in the
			same grouping if and only if they have the same number of output wires.
			Call this partition $\pi$ and the individual groupings $\{\pi_B\}$.
		\item \label{step:firstpartitionsplit} 
				For each grouping $\pi_B$, 	
				form a new partition of the spiders in $\pi_B$
				such that
				spiders are in the same grouping in the new partition
				if and only if they have the same number of input wires.
				Let $\pi_{B,I}$ denote the new partition that corresponds to the grouping $\pi_B$.
		\item
			Suppose that the number of spiders in $d_{\pi}$ that have output wires is $t$.
			Calculate all integer partitions of $t$.  
			
		\item
			Then, for each integer partition $\lambda$ of $t$:
		\begin{enumerate}
			\item
				Assign a distinct label ($i$, $j$, $k$, $\dots$) to each part of $\lambda$
				such that the label appears as many times as the size of the part.
				Treating the labels as a multiset, compute all partitions
				of this multiset that correspond to the original partition
				structure of $\pi$, modulo isomorphism: that is, identifying
				partitions that are equivalent under a relabelling of the labels.
				Denote the non-isomorphic partitions of the multiset of labels by
				$\{L\}$, and for each $L$, denote its individual groupings by $\{L_B\}$.
			\item \label{step:secondpartition}
				For each non-isomorphic partition $L$:
				\begin{enumerate}
					\item 
						For each grouping $L_B$, 
						compute all partitions of $L_B$ that 
						correspond to the partition structure of $\pi_{B,I}$.
						Convert each partition of $L_B$ into an ordered tuple of labels
						by replacing curly brackets with curved brackets
						and the vertical bar that splits groupings with a comma. 
						Let $T_{L,B}$ denote the resulting set of tuples corresponding to $L_B$.
					\item \label{step:enumerate}
						Form the Cartesian product of the sets $T_{L,B}$, maintaining the order of the
						individual groupings $\pi_B$ in $\pi$. 
						Denote the Cartesian product set of ordered labels by $T_{\lambda, L}$.

					\item
						For each tuple $l \in T_{\lambda, L}$,
						create a copy of the sorted diagram of $d_{\pi}$
						and label the output spiders with $l$,
						resulting in a set $D_{\lambda, L}$ of partially labelled diagrams 
						that are indexed by $T_{\lambda, L}$.
					\item
						Create an empty set $R_{\lambda, L}$.
						For each partially labelled diagram in $D_{\lambda, L}$, merge spiders that are labelled
						with the same label. If no merging took place, continue, otherwise:
						\begin{enumerate}
							\item Reconsider the newly merged bipartition diagram $d_{\theta}$ 
								(retaining but ignoring its labels for now) and
								apply Steps \ref{step:sortspiders}, \ref{step:firstpartition}, 
								and \ref{step:firstpartitionsplit} to $d_{\theta}$.
								This gives a partition $\theta$ whose individual groupings are $\{\theta_B\}$,
								and for each grouping $\theta_B$, it is further partitioned into $\theta_{B,I}$.
							\item For each grouping $\theta_B$, form a set from the labels of the grouping
								and remove the labels from the diagram.
								This gives a partition of labels $\tilde{L}$ into individual groupings $\{\tilde{L}_B\}$.
							\item
								Apply Step \ref{step:secondpartition} to this partition of labels $\tilde{L}$ for $\theta_{B,I}$
								up to and including Step \ref{step:enumerate}.
							\item 
								Then, for each tuple $l$ in the newly created $T_{\tilde{\lambda}, \tilde{L}}$,
								create a copy of the sorted diagram of $d_{\theta}$ with its labels removed and
								label the output spiders with $l$. 
								Add the resulting diagram for each tuple $l$ to $R_{\lambda, L}$.
						\end{enumerate}
					\item Take the set union of $D_{\lambda, L}$ and $R_{\lambda, L}$. 
						This gives a full set of partially labelled diagrams for 
						each non-isomorphic partition of the multiset of labels $L$.
						Denote this set by $P_{\lambda, L}$.
				\end{enumerate}
			\item Take the set union of the $P_{\lambda, L}$ over all non-isomorphic partitions of the multiset of labels $L$.
				This gives a full set of partially labelled diagrams for each integer partition $\lambda$ of $t$. 
				Denote this set by $P_{\lambda}$.
		\end{enumerate}

	\item Take the set union of $P_{\lambda}$ over all integer partitions $\lambda$ of $t$.
		Denote this set by $P$.
		For each diagram in $P$, reorder the labelled spiders in lexicographical order.
		Propagate the labels to the ends of the output wires.
		Partition the set $P$ into subsets such that diagrams are in the same subset
		if and only if their tuples of output indices are the same.
		Label each subset with its shared tuple of output labels; each label
		represents a grouped output for $d_{\pi}$.
	\end{enumerate}	
	\textbf{Output:} All possible grouped outputs for $d_{\pi}$ and, for each one, its associated set
	of partially labelled diagrams.
	\end{tcolorbox}
  	\label{subprocedureI}
\end{figure*}



\subsection{Subprocedure II}

\begin{example}

	Consider the partially labelled diagram for the grouped output $(i,i,j)$ of the $(4,3)$--bipartition diagram $d_{\pi}$
\begin{equation} \label{ProcIex16}
	\begin{aligned}
		\scalebox{0.7}{\input{ProcIex16.tikz}}
	\end{aligned}
\end{equation}
	We follow Subprocedure II to obtain all possible fully labelled diagrams for this particular partially labelled diagram.

	\textbf{Step 1}: We see that the partially labelled diagram has $t=2$ labelled spiders, and also $s=3$ spiders with no label.
	Hence $L_X$, the set of fixed labels, is $\{i,j\}$.
	
	\textbf{Step 2}: 
	We see that the unlabelled spiders in (\ref{ProcIex16}) are already sorted in decreasing order of input wires.

	\textbf{Step 3}: 
	We partition the sorted unlabelled spiders into groupings where spiders are in the same grouping if and only if
	they have the same number of input wires. We see that $\pi = \{I_2 \mid I_1, I_1\}$, which has two individual groupings
	$\pi_{B_1} = \{I_2\}$ and
	$\pi_{B_2} = \{I_1, I_1\}$.

	\textbf{Step 4}: 
	We consider all possible ways of labelling the unlabelled spiders, where 
			$k$ of them are fixed and the remaining $s - k$ of them are free, for $k = 0, 1, 2, 3$.

	Consider each value of $k$ in turn:

	i) $k = 0$: $0$ fixed labels and $3$ free labels.

	First, we create the multiset $\{F, F, F\}$, where $F$ denotes a free symbol.
	We compute all partitions of this multiset that correspond to the original partition structure of $\pi$.
	In this case, there is only one such partition, and so $\Pi_{X,F} = \{\{F \mid F, F\}\}$.

	Next, we calculate all possible size $0$ multisets choosing labels from $L_X$, allowing repeats,
	and all possible size $3$ multisets choosing free labels, introducing them as needed, and also allowing repeats.
	This gives $\{k,l,m\}$, $\{k,k,l\}$ and $\{k,k,k\}$.
	We consider each of these in turn:

	\begin{itemize}
		\item we compute all partitions of $\{k,l,m\}$ that correspond to the original partition structure of the only
	partition in $\Pi_{X,F}$, namely $\{F \mid F, F\}$, modulo isomorphism. 
	Despite there being three partitions, namely
	$\{k \mid l, m\}$,
	$\{l \mid k, m\}$, and
	$\{m \mid k, l\}$,
	all three are equivalent under a relabelling of the free labels.
	Hence there is only one non-isomorphic partition of the multiset, namely $\{k \mid l,m\}$ only.
	Creating a copy of (\ref{ProcIex16}), we label the spiders in order from left to right to create the fully labelled diagram
\begin{equation} \label{ProcIex17}
	\begin{aligned}
		\scalebox{0.7}{\input{ProcIex17.tikz}}
	\end{aligned}
\end{equation}
	
		\item we compute all partitions of $\{k,k,l\}$ that correspond to the original partition structure of the only
	partition in $\Pi_{X,F}$, namely $\{F \mid F, F\}$, modulo isomorphism. 
	This gives $\{k \mid k, l\}$ and $\{l \mid k, k\}$.
	As before, we label the spiders in order from left to right to create the following two fully labelled diagrams
\begin{equation} \label{ProcIex18}
	\begin{aligned}
		\scalebox{0.7}{\input{ProcIex18.tikz}}
	\end{aligned}
	\quad\quad\text{ and }\quad\quad
	\begin{aligned}
		\scalebox{0.7}{\input{ProcIex18a.tikz}}
	\end{aligned}
\end{equation}

		\item we compute all partitions of $\{k,k,k\}$ that correspond to the original partition structure of the only
	partition in $\Pi_{X,F}$, namely $\{F \mid F, F\}$, modulo isomorphism. 
	This gives $\{k \mid k, k\}$ only. 
	As before, we label the spiders in order from left to right to create the following two fully labelled diagrams
\begin{equation} \label{ProcIex18b}
	\begin{aligned}
		\scalebox{0.7}{\input{ProcIex18b.tikz}}
	\end{aligned}
\end{equation}
	\end{itemize}

	ii) $k = 1$: $1$ fixed label and $2$ free labels.

	First, we create the multiset $\{X, F, F\}$, where $X$ denotes a fixed symbol and $F$ denotes a free symbol.
	We compute all partitions of this multiset that correspond to the original partition structure of $\pi$.
	There are two such partitions, and so $\Pi_{X,F} = \{\{X \mid F, F\}, \{F \mid X, F\}\}$.

	Next, we calculate all possible size $1$ multisets choosing labels from $L_X$, allowing repeats,
	and all possible size $1$ multisets choosing free labels, introducing them as needed, and also allowing repeats.
	Taking their Cartesian product gives all possible ways of choosing $1$ fixed label and $1$ free label.
	This gives 
	$\{i,k,l\}$,
	$\{i,k,k\}$,
	$\{j,k,l\}$, and
	$\{j,k,k\}$.
	We consider each of these in turn:
	\begin{itemize}
		\item for each partition $\pi_{X,F} \in \Pi_{X,F}$, we 
			compute all partitions of $\{i,k,l\}$ that correspond to the original partition structure of $\pi_{X,F}$,
			modulo isomorphism.
			For $\{X \mid F, F\}$, this gives $\{i \mid k,l\}$ only.
			For $\{F \mid X, F\}$, this initially gives $\{k \mid i,l\}$ and $\{l \mid i,k\}$, but 
			these two are equivalent under a relabelling of the free labels, hence there is only $\{k \mid i,l\}$.

			For each partition, we create a copy of (\ref{ProcIex16}) and label the spiders 
			in order from left to right to create the two fully labelled diagrams
			\begin{equation} \label{ProcIex19}
				\begin{aligned}
					\scalebox{0.7}{\input{ProcIex19.tikz}}
				\end{aligned}
				\quad\quad\text{ and }\quad\quad
				\begin{aligned}
					\scalebox{0.7}{\input{ProcIex19a.tikz}}
				\end{aligned}
			\end{equation}

			A similar process occurs for $\{j,k,l\}$, giving
			\begin{equation} \label{ProcIex19b}
				\begin{aligned}
					\scalebox{0.7}{\input{ProcIex19b.tikz}}
				\end{aligned}
				\quad\quad\text{ and }\quad\quad
				\begin{aligned}
					\scalebox{0.7}{\input{ProcIex19c.tikz}}
				\end{aligned}
			\end{equation}
		\item for each partition $\pi_{X,F} \in \Pi_{X,F}$, we 
			compute all partitions of $\{i,k,k\}$ that correspond to the original partition structure of $\pi_{X,F}$,
			modulo isomorphism.
			For $\{X \mid F, F\}$, this gives $\{i \mid k,k\}$ only.
			For $\{F \mid X, F\}$, this gives $\{k \mid i,k\}$ only.

			As before,  we create a copy of (\ref{ProcIex16}) and label the spiders 
			in order from left to right to create the two fully labelled diagrams
			\begin{equation} \label{ProcIex20}
				\begin{aligned}
					\scalebox{0.7}{\input{ProcIex20.tikz}}
				\end{aligned}
				\quad\quad\text{ and }\quad\quad
				\begin{aligned}
					\scalebox{0.7}{\input{ProcIex20a.tikz}}
				\end{aligned}
			\end{equation}
			A similar process occurs for $\{j,k,k\}$, giving
			\begin{equation} \label{ProcIex20b}
				\begin{aligned}
					\scalebox{0.7}{\input{ProcIex20b.tikz}}
				\end{aligned}
				\quad\quad\text{ and }\quad\quad
				\begin{aligned}
					\scalebox{0.7}{\input{ProcIex20c.tikz}}
				\end{aligned}
			\end{equation}
	\end{itemize}

	iii) $k = 2$: $2$ fixed labels and $1$ free label.

	First, we create the multiset $\{X, X, F\}$, where $X$ denotes a fixed symbol and $F$ denotes a free symbol.
	We compute all partitions of this multiset that correspond to the original partition structure of $\pi$.
	There are two such partitions, and so $\Pi_{X,F} = \{\{X \mid X, F\}, \{F \mid X, X\}\}$.

	Next, we calculate all possible size $2$ multisets choosing labels from $L_X$, allowing repeats,
	and all possible size $1$ multisets choosing free labels, introducing them as needed, and also allowing repeats.
	Taking their Cartesian product gives all possible ways of choosing $2$ fixed labels and $1$ free label.
	This gives 
	$\{i,j,k\}$,
	$\{i,i,k\}$, and
	$\{j,j,k\}$.
	We consider each of these in turn:
	\begin{itemize}
		\item for each partition $\pi_{X,F} \in \Pi_{X,F}$, we 
			compute all partitions of $\{i,j,k\}$ that correspond to the original partition structure of $\pi_{X,F}$,
			modulo isomorphism.
			For $\{X \mid X, F\}$, this gives $\{i \mid j,k\}$ and $\{j \mid i,k\}$ (which are not isomorphic).
			For $\{F \mid X, X\}$, this gives $\{k \mid i,j\}$ only.
			
			As usual, we create a copy of (\ref{ProcIex16}) and label the spiders 
			in order from left to right to create the three fully labelled diagrams
			\begin{equation} \label{ProcIex21}
				\begin{aligned}
					\scalebox{0.7}{\input{ProcIex21.tikz}}
				\end{aligned}
				\quad\quad\text{ and }\quad\quad
				\begin{aligned}
					\scalebox{0.7}{\input{ProcIex21a.tikz}}
				\end{aligned}
			\end{equation}
			and
			\begin{equation} \label{ProcIex21b}
				\begin{aligned}
					\scalebox{0.7}{\input{ProcIex21b.tikz}}
				\end{aligned}
			\end{equation}
		\item for each partition $\pi_{X,F} \in \Pi_{X,F}$, we 
			compute all partitions of $\{i,i,k\}$ that correspond to the original partition structure of $\pi_{X,F}$,
			modulo isomorphism.
			For $\{X \mid X, F\}$, this gives $\{i \mid i,k\}$ only.
			For $\{F \mid X, X\}$, this gives $\{k \mid i,i\}$ only.
			
			As usual, we create a copy of (\ref{ProcIex16}) and label the spiders 
			in order from left to right to create the two fully labelled diagrams
			\begin{equation} \label{ProcIex22}
				\begin{aligned}
					\scalebox{0.7}{\input{ProcIex22.tikz}}
				\end{aligned}
				\quad\quad\text{ and }\quad\quad
				\begin{aligned}
					\scalebox{0.7}{\input{ProcIex22b.tikz}}
				\end{aligned}
			\end{equation}
		\item a similar process to $\{i,i,k\}$ occurs for $\{j,j,k\}$, hence we obtain
			\begin{equation} \label{ProcIex22c}
				\begin{aligned}
					\scalebox{0.7}{\input{ProcIex22c.tikz}}
				\end{aligned}
				\quad\quad\text{ and }\quad\quad
				\begin{aligned}
					\scalebox{0.7}{\input{ProcIex22d.tikz}}
				\end{aligned}
			\end{equation}
	\end{itemize}

	iv) $k = 3$: $3$ fixed labels and $0$ free labels.

	Despite there only being two fixed labels in $L_X$, we can still label the $s=3$ unlabelled spiders in (\ref{ProcIex16})
	with them.
	First, we create the multiset $\{X, X, X\}$, where $X$ denotes a fixed symbol.
	We compute all partitions of this multiset that correspond to the original partition structure of $\pi$.
	There is only one such partition, and so $\Pi_{X,F} = \{\{X \mid X, X\}\}$.

	Next, we calculate all possible size $3$ multisets choosing labels from $L_X$, allowing repeats,
	and all possible size $0$ multisets choosing free labels, introducing them as needed, and also allowing repeats.
	Taking their Cartesian product gives all possible ways of choosing $3$ fixed labels and $0$ free labels.
	This gives 
	$\{i,i,j\}$,
	$\{i,i,i\}$, 
	$\{j,j,i\}$, and
	$\{j,j,j\}$.
	We consider each of these in turn:
	\begin{itemize}
		\item we compute all partitions of $\{i,i,j\}$ that correspond to the original partition structure of the only
			partition in $\Pi_{X,F}$, namely $\{X \mid X, X\}$, modulo isomorphism. 
			This gives $\{i \mid i, j\}$ and $\{j \mid i, i\}$.
			As before, we label the spiders in order from left to right to create the following two fully labelled diagrams
			\begin{equation} \label{ProcIex22e}
				\begin{aligned}
					\scalebox{0.7}{\input{ProcIex22e.tikz}}
				\end{aligned}
				\quad\quad\text{ and }\quad\quad
				\begin{aligned}
					\scalebox{0.7}{\input{ProcIex22f.tikz}}
				\end{aligned}
			\end{equation}
			A similar process occurs for $\{j,j,i\}$, giving
			\begin{equation} \label{ProcIex22g}
				\begin{aligned}
					\scalebox{0.7}{\input{ProcIex22g.tikz}}
				\end{aligned}
				\quad\quad\text{ and }\quad\quad
				\begin{aligned}
					\scalebox{0.7}{\input{ProcIex22h.tikz}}
				\end{aligned}
			\end{equation}
		\item we compute all partitions of $\{i,i,i\}$ that correspond to the original partition structure of the only
			partition in $\Pi_{X,F}$, namely $\{X \mid X, X\}$, modulo isomorphism. 
			This gives $\{i \mid i, i\}$ only.
			A similar process for $\{j,j,j\}$ gives $\{j \mid j, j\}$ only.
			Hence we obtain the following two fully labelled diagrams
			\begin{equation} \label{ProcIex22i}
				\begin{aligned}
					\scalebox{0.7}{\input{ProcIex22i.tikz}}
				\end{aligned}
				\quad\quad\text{ and }\quad\quad
				\begin{aligned}
					\scalebox{0.7}{\input{ProcIex22j.tikz}}
				\end{aligned}
			\end{equation}
	\end{itemize}

	\textbf{Step 5}: To finish, we merge all of the spiders in each of the diagrams, remove duplicates, and form a set as a result.
	This is the set of fully labelled diagrams corresponding to the partially labelled diagram (\ref{ProcIex16})
	for the grouped output $(i,i,j)$ of the $(4,3)$--bipartition diagram $d_{\pi}$.

	To ensure brevity, we will not reproduce the final set of merged spiders here.
	However, we can see that only three duplicates are removed:
	the right hand diagram in (\ref{ProcIex22}) is removed since it is the same as the left hand diagram in (\ref{ProcIex20});
	the right hand diagram in (\ref{ProcIex22c}) is removed since it is the same as the left hand diagram in (\ref{ProcIex20b});
	and
	the right hand diagram in (\ref{ProcIex22g}) is removed since it is the same as the right hand diagram in (\ref{ProcIex22e}).
\end{example}


\begin{figure*}[tb!]
	\begin{tcolorbox}[colback=lplum!10, colframe=lplum!40, coltitle=black, 
		title={\bfseries 
		Subprocedure II: 
		How to Determine All Fully Labelled Diagrams for a Partially Labelled Diagram of a Grouped Output.
		},
		fonttitle=\bfseries
		]
		\textbf{Input:} A partially labelled diagram for a grouped output of a
		$(k,l)$--bipartition diagram $d_{\pi}$.
	\begin{enumerate}
		\item 
			Suppose that the partially labelled diagram has $t$ labelled spiders, each with a separate label,
			and $s$ spiders with no label. 
			Call the labels in the partially labelled diagram fixed. Let $L_X$ denote the set of fixed labels.
			Any other labels that are created in the following are said to be free.

		\item	\label{step:sortunlabelled}
					Leaving the labelled spiders as they are, sort the unlabelled spiders in decreasing order of input wires.
		\item
					Partition the sorted unlabelled spiders into groupings where spiders are in the same grouping
					if and only if they have the same number of input wires. 
					Call this partition $\pi$ and the individual groupings $\{\pi_B\}$.
		\item
			Consider all possible ways of labelling the unlabelled spiders, where 	
			$k$ of them are fixed and the remaining $s - k$ of them are free, for $k = 0, 1, \dots, s$.
			Then, for each value of $k$:


			\begin{enumerate}
				\item
					Create a multiset of $k$ $X$ symbols (to denote fixed) and $s-k$ $F$ symbols (to denote free).
					Compute all partitions of this multiset that correspond to the original partition structure of $\pi$.
					Denote this set by $\Pi_{X,F}$.
				\item 
					Now calculate all possible $k$-size multisets choosing labels from $L_X$, allowing repeats, and
					all possible $(s-k)$-size multisets choosing free labels (by introducing them as needed), also allowing repeats.
					Taking their Cartesian product gives all possible ways of choosing $k$ fixed indices and $s-k$ free indices.
				\item
					For each pair of $k$ fixed indices and $s-k$ free indices in this set:
					\begin{enumerate}
						\item For each partition $\pi_{X,F} \in \Pi_{X,F}$, compute all partitions of this pair of 
							$k$ fixed indices and $s-k$ free indices that correspond to the original partition structure
							of $\pi_{X,F}$, modulo isomorphism: that is, identifying partitions that are equivalent
							under a relabelling of the free indices only.
						\item
							For each non-isomorphic partition, create a copy of the sorted partially labelled diagram
							from Step \ref{step:sortunlabelled} and label the output spiders in order from left to right
							to create a fully labelled diagram.
					\end{enumerate}
				\item 
					Take the set union of the set of fully labelled diagrams for all pairs of $k$ fixed indices and $s-k$ free indices.
			\end{enumerate}
		\item
			Finally, take the set union of the set of the fully labelled diagrams for all values of $k$.
			Merge all of the spiders in each fully labelled diagram, and remove any duplicates.
			If it is desired, propagate the labels to the ends of the wires.
	\end{enumerate}
	\textbf{Output:} A set of fully labelled diagrams for the partially labelled diagram of a grouped output.
	\end{tcolorbox}
  	\label{subprocedureII}
\end{figure*}


\subsection{Subprocedures III, IV, and V}

\begin{example} \label{subIII32ex}
	In Example \ref{subI32ex}, we considered 
	the $(3,2)$--bipartition diagram (\ref{ProcIex10})
	and, by applying Subprocedure I, obtained two sets of partially labelled diagrams that were given in (\ref{exD1result}), 
	one for each possible grouped output for (\ref{ProcIex10}).
	In fact, we noted that the diagrams in each set were actually fully labelled, hence we could skip applying Subprocedure II
	to each of them and proceed to applying Subprocedure III to each fully labelled diagram.

	Following Subprocedure III, we see that the grouped map labels that correspond to the diagrams in the set $(i,j)$ of (\ref{exD1result})
	are
	\begin{equation}
		i, j
		\xleftarrow{g}
		i, i, j
	\end{equation}
	and
	\begin{equation}
		i, j
		\xleftarrow{g}
		j, j, i
	\end{equation}
	where $i \neq j$,
	and the grouped map label that corresponds to the diagram in the set $(i,i)$ of (\ref{exD1result}) is
	\begin{equation}
		i, i
		\xleftarrow{g}
		i, i, i
	\end{equation}
	We now apply Subprocedure IV to each of these grouped map labels --- that is,
	we unroll the right hand side of each grouped map label to obtain its left grouped map label:
	\begin{equation}
		i, j
		\xleftarrow{lg}
		i, i, j
		+
		i, j, i
		+
		j, i, i
	\end{equation}
	\begin{equation}
		i, j
		\xleftarrow{lg}
		j, j, i
		+
		j, i, j
		+
		i, j, j
	\end{equation}
	\begin{equation}
		i, i
		\xleftarrow{lg}
		i, i, i
	\end{equation}

	For each grouped output, we can add together the right hand side of every left grouped map label
	where this grouped output appears on the left hand side to obtain a single left grouped map label
	for each grouped output:
	\begin{equation} \label{exD4conc}
		i, j
		\xleftarrow{lg}
		i, i, j
		+
		i, j, i
		+
		j, i, i
		+
		j, j, i
		+
		j, i, j
		+
		i, j, j
	\end{equation}
	and
	\begin{equation}
		i, i
		\xleftarrow{lg}
		i, i, i
	\end{equation}
	Hence, by applying Subprocedure V and converting the resulting map labels to transformation map labels, 
	we see that, for an input symmetric tensor $T \in (\mathbb{R}^{n})^{\otimes 3}$, where $n \geq 2$,
	the output symmetric tensor $D_{\pi}(T) \in (\mathbb{R}^{n})^{\otimes 2}$, where $d_\pi$ is given in (\ref{ProcIex10}),
	is given by
	\begin{equation} \label{exD4conc1}
		{D_{\pi}(T)}_{i,j}
		\leftarrow
		T_{i, i, j}
		+
		T_{i, j, i}
		+
		T_{j, i, i}
		+
		T_{j, j, i}
		+
		T_{j, i, j}
		+
		T_{i, j, j}
	\end{equation}
	for all $i \neq j \in [n]$,
	and 
	\begin{equation}
		{D_{\pi}(T)}_{i, i}
		\leftarrow
		T_{i, i, i}
	\end{equation}
	for all $i \in [n]$.

	As a trivial case, note that if $n = 1$, then we would not consider the fully labelled diagrams in the set $(i,j)$ of (\ref{exD1result}),
	and so the only transformation map label would be
	\begin{equation}
		{D_{\pi}(T)}_{1, 1}
		\leftarrow
		T_{1, 1, 1}
	\end{equation}
\end{example}

\begin{example} \label{subIII34ex}


	We can compare Example \ref{subIII32ex}
	directly with the $(3,4)$--bipartition diagram $d_{\pi}$
	\begin{equation} \label{ProcIex29}
		\begin{aligned}
			\scalebox{0.7}{\input{ProcIex29.tikz}}
		\end{aligned}
	\end{equation}
	Subprocedure I will give two sets:
	\begin{equation} \label{remarkD5Iresult}
	\begin{aligned}
		(i,i,j,j)
		=
		\left\{
			\;
		\scalebox{0.7}{\input{ProcIex29a.tikz}},
		\;
		\scalebox{0.7}{\input{ProcIex29b.tikz}}
			\;
		\right\}
	\end{aligned}
		\quad\quad\text{ and }\quad\quad
	\begin{aligned}
		(i,i,i,i)
		=
		\left\{
			\;
		\scalebox{0.7}{\input{ProcIex29c.tikz}}
		\;
		\right\}
	\end{aligned}
	\end{equation}
	Once again we can skip over Subprocedure II since each diagram in each set is fully labelled.
	Applying Subprocedure III and then Subprocedure IV,
	we obtain the following left grouped map labels
	\begin{equation} \label{exD4remark}
		i, i, j, j
		\xleftarrow{lg}
		i, j, j
		+
		j, i, j
		+
		j, j, i
		+
		i, i, j
		+
		i, j, i
		+
		j, i, i
	\end{equation}
	where $i \neq j$,
	and
	\begin{equation} \label{exD4remarki}
		i, i, i, i
		\xleftarrow{lg}
		i, i, i
	\end{equation}
	We now apply Subprocedure V to obtain a set of map labels for each left grouped map label:
	(\ref{exD4remark}) becomes
		\begin{equation} \label{D5mapequ1}
		i, i, j, j
		\leftarrow
		i, j, j
		+
		j, i, j
		+
		j, j, i
		+
		i, i, j
		+
		i, j, i
		+
		j, i, i
	\end{equation}
	\begin{equation} \label{D5mapequ2}
		i, j, i, j
		\leftarrow
		i, j, j
		+
		j, i, j
		+
		j, j, i
		+
		i, i, j
		+
		i, j, i
		+
		j, i, i
	\end{equation}
	\begin{equation} \label{D5mapequ3}
		i, j, j, i
		\leftarrow
		i, j, j
		+
		j, i, j
		+
		j, j, i
		+
		i, i, j
		+
		i, j, i
		+
		j, i, i
	\end{equation}
	and (\ref{exD4remarki}) becomes
	\begin{equation} 
		i, i, i, i
		\leftarrow
		i, i, i
	\end{equation}
	Hence, by converting the resulting map labels to transformation map labels, 
	we see that, for an input symmetric tensor $T \in (\mathbb{R}^{n})^{\otimes 3}$, where $n \geq 2$,
	the output symmetric tensor $D_{\pi}(T) \in (\mathbb{R}^{n})^{\otimes 4}$, where $d_\pi$ is given in (\ref{ProcIex29}),
	is given by
	\begin{equation} \label{D5equ1}
		{D_{\pi}(T)}_{i, i, j, j}
		\leftarrow
		T_{i, j, j}
		+
		T_{j, i, j}
		+
		T_{j, j, i}
		+
		T_{i, i, j}
		+
		T_{i, j, i}
		+
		T_{j, i, i}
	\end{equation}
	\begin{equation} \label{D5equ2}
		{D_{\pi}(T)}_{i, j, i, j}
		\leftarrow
		T_{i, j, j}
		+
		T_{j, i, j}
		+
		T_{j, j, i}
		+
		T_{i, i, j}
		+
		T_{i, j, i}
		+
		T_{j, i, i}
	\end{equation}
	\begin{equation} \label{D5equ3}
		{D_{\pi}(T)}_{i, j, j, i}
		\leftarrow
		T_{i, j, j}
		+
		T_{j, i, j}
		+
		T_{j, j, i}
		+
		T_{i, i, j}
		+
		T_{i, j, i}
		+
		T_{j, i, i}
	\end{equation}
	for all $i \neq j \in [n]$, and
	\begin{equation}
		{D_{\pi}(T)}_{i, i, i, i}
		\leftarrow
		T_{i, i, i}
	\end{equation}
	for all $i \in [n]$.

	Once again, as a trivial case, note that if $n = 1$, then we would not consider the fully labelled diagrams 
	in the set $(i,i,j,j)$ of (\ref{remarkD5Iresult}),
	and so the only transformation map label would be
	\begin{equation}
		{D_{\pi}(T)}_{1,1,1,1}
		\leftarrow
		T_{1, 1, 1}
	\end{equation}
\end{example}


\begin{example}
	We obtain the left grouped map labels for the $(3,1)$--bipartition diagram $d_{\pi}$
	\begin{equation} \label{ProcIex23}
		\begin{aligned}
			\scalebox{0.7}{\input{ProcIex23.tikz}}
		\end{aligned}
	\end{equation}
	as follows.

	We apply Subprocedure I to (\ref{ProcIex23}) to obtain 
\begin{equation} \label{ProcIex23i}
	(i)
	=
	\left\{
		\;
	\begin{aligned}
		\scalebox{0.7}{\input{ProcIex23i.tikz}}
	\end{aligned}
	\;
	\right\}
\end{equation}	
	as the only possible grouped output for (\ref{ProcIex23}) with its associated set of partially labelled diagrams.

	We then apply Subprocedure II to the only partially labelled diagram in $(i)$,
	giving five fully labelled diagrams for this grouped output:
	\begin{equation}
		\begin{aligned}
			\scalebox{0.7}{\input{ProcIex24.tikz}}
		\end{aligned}
	\quad\quad\quad\quad
		\begin{aligned}
			\scalebox{0.7}{\input{ProcIex25.tikz}}
		\end{aligned}
	\quad\quad\quad\quad
		\begin{aligned}
			\scalebox{0.7}{\input{ProcIex26.tikz}}
		\end{aligned}
	\end{equation}
	and
	\begin{equation}
		\begin{aligned}
			\scalebox{0.7}{\input{ProcIex27.tikz}}
		\end{aligned}
		\quad\quad\quad\quad
		\begin{aligned}
			\scalebox{0.7}{\input{ProcIex28.tikz}} 
		\end{aligned}
	\end{equation}

	Following Subprocedure III, we see that the grouped map labels that correspond to these fully labelled diagrams are
	\begin{equation}
		i 
		\xleftarrow{g}
		\sum_{\substack{j,\,k \\ j \ne k \\ j \ne i \\ k \ne i}} j, j, k
		\quad\quad\quad\quad
		i 
		\xleftarrow{g}
		\sum_{\substack{j \\ j \ne i}} j, j, j
		\quad\quad\quad\quad
		i 
		\xleftarrow{g}
		\sum_{\substack{j \\ j \ne i}} i, i, j
	\end{equation}
	and
	\begin{equation}
		i 
		\xleftarrow{g}
		\sum_{\substack{j \\ j \ne i}} i, j, j
		\quad\quad\quad\quad
		i 
		\xleftarrow{g}
		i, i, i
	\end{equation}

	We now apply Subprocedure IV to each of these grouped map labels --- 
	unrolling the right hand side of each grouped map label, we obtain its left grouped map label:
	\begin{equation}
		i 
		\xleftarrow{lg}
		\sum_{\substack{j,\,k \\ j \ne k \\ j \ne i \\ k \ne i}} 
		[
		j, j, k
		+
		j, k, j
		+
		k, j, j
		]
		\quad\quad\quad\quad
		i 
		\xleftarrow{lg}
		\sum_{\substack{j \\ j \ne i}} 
		j, j, j
		\quad\quad\quad\quad
		i 
		\xleftarrow{lg}
		\sum_{\substack{j \\ j \ne i}} 
		[
		i, i, j
		+
		i, j, i
		+
		j, i, i
		]
	\end{equation}
	and
	\begin{equation}
		i 
		\xleftarrow{lg}
		\sum_{\substack{j \\ j \ne i}} 
		[
		i, j, j
		+
		j, i, j
		+
		j, j, i
		]
		\quad\quad\quad\quad
		i 
		\xleftarrow{lg}
		i, i, i
	\end{equation}
	Given that there is only one grouped output, we can add together the right hand side of every left grouped map label
	above to obtain a single left grouped map label for this grouped output:
	\begin{equation} \label{exD3IIIIVlg}
		i 
		\xleftarrow{lg}
		\sum_{\substack{j,\,k \\ j \ne k \\ j \ne i \\ k \ne i}} 
		[
		j, j, k
		+
		j, k, j
		+
		k, j, j	
		]
		+
		\sum_{\substack{j \\ j \ne i}} 
		j, j, j
		+
		\sum_{\substack{j \\ j \ne i}} 
		[
		i, i, j
		+
		i, j, i
		+
		j, i, i
		]
		+
		\sum_{\substack{j \\ j \ne i}} 
		[
		i, j, j
		+
		j, i, j
		+
		j, j, i
		]
		+
		i, i, i
	\end{equation}
	Tidying the right hand side of (\ref{exD3IIIIVlg}) gives
	\begin{equation} 
		i 
		\xleftarrow{lg}
		\sum_{j,\,k = 1}^{n}
		[
		j, j, k
		+
		j, k, j
		+
		k, j, j	
		]
		-
		2\sum_{j = 1}^{n}
		j, j, j	
	\end{equation}
	Hence, by applying Subprocedure V and converting the resulting map label to a transformation map label,
	we see that, for an input symmetric tensor $T \in (\mathbb{R}^{n})^{\otimes 3}$, where $n \geq 3$,
	the output (symmetric) tensor $D_{\pi}(T) \in \mathbb{R}^{n}$, where $d_\pi$ is given in (\ref{ProcIex23}),
	is given by
	\begin{equation}
		{D_{\pi}(T)}_i 
		\leftarrow
		\sum_{j,\,k = 1}^{n}
		[
		T_{j, j, k}
		+
		T_{j, k, j}
		+
		T_{k, j, j}
		]
		-
		2\sum_{j = 1}^{n}
		T_{j, j, j}
	\end{equation}
	for all $i \in [n]$.
	
	Finally, in a similar way to the preceding examples, it is possible to obtain the transformation map label
	in the trivial cases $n=1$ and $n=2$.
\end{example}


\subsection{Map Label Calculations for Example \ref{mainexmaplabels}.}

\begin{example} [Calculations for Example \ref{mainexmaplabels}]
	\label{mainexmaplabelscalcs}
	We calculate the transformation map labels for each of the four $(2,1)$--bipartition diagrams
	that were given in (\ref{bipartdiagsex}) for $n=3$. We reproduce them below, for ease:
	\begin{equation} 
		\begin{aligned}
			\scalebox{0.7}{\input{diagram21diags.tikz}}
		\end{aligned}
	\end{equation}
	In the following, we refer to these diagrams as 
	$d_{\pi_1}$,
	$d_{\pi_2}$,
	$d_{\pi_3}$,
	$d_{\pi_4}$, respectively.
	We follow the procedure that was given in the main paper to obtain the transformation map labels.

	For $d_{\pi_1}$, it is clear that the only grouped output, and its associated set of partially labelled diagrams, is
	\begin{equation} 
	(i)
	=
	\left\{
		\;
	\begin{aligned}
		\scalebox{0.7}{\input{ProcIex30.tikz}}
	\end{aligned}
	\;
	\right\}
	\end{equation}	
	Since there is only one diagram in the set, and it is fully labelled, we skip over Subprocedure II to apply
	Subprocedure III, obtaining the grouped map label
	\begin{equation}
		i
		\xleftarrow{g}
		i,i
	\end{equation}
	Applying Subprocedure IV to unroll its right hand side does nothing, hence we obtain
	\begin{equation}
		i
		\xleftarrow{lg}
		i,i
	\end{equation}
	Finally, there is nothing to unroll on its left hand side either,
	hence we obtain the map label
	\begin{equation}
		i
		\leftarrow
		i,i
	\end{equation}
	which we can convert into the transformation map label
	\begin{equation}
		{D_{\pi_1}(T)}_i
		\leftarrow
		T_{i,i}
	\end{equation}
	for $i \in [3]$.
	This is (\ref{tensormaplabelex1}) in the main paper.

	For $d_{\pi_2}$, once again there is only one grouped output, whose associated set of partially labelled diagrams is
	\begin{equation} 
	(i)
	=
	\left\{
		\;
	\begin{aligned}
		\scalebox{0.7}{\input{ProcIex31.tikz}}
	\end{aligned}
	\;
	\right\}
	\end{equation}	
	As there is only one diagram in the set, we now apply Subprocedure II to it to obtain two fully labelled diagrams that are associated
	with the grouped output
	\begin{equation} 
	\begin{aligned}
		(i)
		=
		\left\{
			\;
		\scalebox{0.7}{\input{ProcIex31b.tikz}},
		\;
		\scalebox{0.7}{\input{ProcIex31a.tikz}}
			\;
		\right\}
	\end{aligned}
	\end{equation}	
	From these fully labelled diagrams, we obtain the following grouped map labels
	\begin{equation}
		i
		\xleftarrow{g}
		\sum_{\substack{j=1 \\ j \ne i}}^{3} j,j
		\quad\quad\text{ and }\quad\quad
		i
		\xleftarrow{g}
		i,i
	\end{equation}
	Applying Subprocedure IV to unroll the right hand side for each grouped map label does nothing, hence we obtain
	\begin{equation}
		i
		\xleftarrow{lg}
		\sum_{\substack{j=1 \\ j \ne i}}^{3} j,j
		\quad\quad\text{ and }\quad\quad
		i
		\xleftarrow{lg}
		i,i
	\end{equation}
	Adding together the right hand side of the resulting left grouped map labels gives
	\begin{equation}
		i
		\xleftarrow{lg}
		\sum_{j=1}^{3} j,j
	\end{equation}
	Finally, there is nothing to unroll on its left hand side either,
	hence we obtain the map label
	\begin{equation}
		i
		\leftarrow
		\sum_{j=1}^{3} j,j
	\end{equation}
	which we can convert into the transformation map label
	\begin{equation}
		{D_{\pi_2}(T)}_i
		\leftarrow
		\sum_{j=1}^{3} T_{j,j}
	\end{equation}
	for $i \in [3]$.
	This is (\ref{tensormaplabelex2}) in the main paper.

	For $d_{\pi_3}$, once again there is only one grouped output, whose associated set of partially labelled diagrams is
	\begin{equation} 
	(i)
	=
	\left\{
		\;
	\begin{aligned}
		\scalebox{0.7}{\input{ProcIex32.tikz}}
	\end{aligned}
	\;
	\right\}
	\end{equation}	
	As there is only one diagram in the set, we now apply Subprocedure II to it to obtain two fully labelled diagrams that are associated
	with the grouped output
	\begin{equation} 
	\begin{aligned}
		(i)
		=
		\left\{
			\;
		\scalebox{0.7}{\input{ProcIex32a.tikz}},
		\;
		\scalebox{0.7}{\input{ProcIex32b.tikz}}
			\;
		\right\}
	\end{aligned}
	\end{equation}	
	From these, we obtain the following grouped map labels
	\begin{equation}
		i
		\xleftarrow{g}
		\sum_{\substack{j=1 \\ j \ne i}}^{3} i,j
		\quad\quad\text{ and }\quad\quad
		i
		\xleftarrow{g}
		i,i
	\end{equation}
	Applying Subprocedure IV to unroll the right hand side for each grouped map label gives
	\begin{equation}
		i
		\xleftarrow{lg}
		\sum_{\substack{j=1 \\ j \ne i}}^{3} [i,j + j,i]
		\quad\quad\text{ and }\quad\quad
		i
		\xleftarrow{lg}
		i,i
	\end{equation}
	Adding together the right hand side of the resulting left grouped map labels 
	and tidying gives
	\begin{equation} 
		i
		\xleftarrow{lg}
		\sum_{j=1}^{3} [i,j + j,i] - i,i
	\end{equation}
	Finally, there is nothing to unroll on its left hand side either,
	hence we obtain the map label
	\begin{equation}
		i
		\leftarrow
		\sum_{j=1}^{3} [i,j + j,i] - i,i
	\end{equation}
	which we can convert into the transformation map label
	\begin{equation}
		{D_{\pi_3}(T)}_i
		\leftarrow
		\sum_{j=1}^{3} [T_{i,j} + T_{j,i}] - T_{i,i}
	\end{equation}
	for $i \in [3]$.
	This is (\ref{tensormaplabelex3}) in the main paper.

	Finally, for $d_{\pi_4}$, once again there is only one grouped output, whose associated set of partially labelled diagrams is
	\begin{equation} 
	(i)
	=
	\left\{
		\;
	\begin{aligned}
		\scalebox{0.7}{\input{ProcIex33.tikz}}
	\end{aligned}
	\;
	\right\}
	\end{equation}	
	As there is only one diagram in the set, we now apply Subprocedure II to it to obtain four fully labelled diagrams that are associated
	with the grouped output
	\begin{equation} 
	\begin{aligned}
		(i)
		=
		\left\{
			\;
		\scalebox{0.7}{\input{ProcIex33a.tikz}},
		\quad
		\scalebox{0.7}{\input{ProcIex33b.tikz}},
		\quad
		\scalebox{0.7}{\input{ProcIex33c.tikz}},
		\quad
		\scalebox{0.7}{\input{ProcIex33d.tikz}}
		\;
		\right\}
	\end{aligned}
	\end{equation}	
	From these, we obtain the following grouped map labels
	\begin{equation}
		i
		\xleftarrow{g}
		\sum_{\substack{j,\,k \\ j \ne k \\ j \ne i \\ k \ne i}}^{3} j,k
		\quad\quad\quad\quad
		i
		\xleftarrow{g}
		\sum_{\substack{j=1 \\ j \ne i}}^{3} j,j
		\quad\quad\quad\quad
		i
		\xleftarrow{g}
		\sum_{\substack{j=1 \\ j \ne i}}^{3} i,j 
		\quad\quad\quad\quad
		i
		\xleftarrow{g}
		i,i
	\end{equation}
	Applying Subprocedure IV to unroll the right hand side for each grouped map label gives
	\begin{equation}
		i
		\xleftarrow{lg}
		\sum_{\substack{j,\,k \\ j \ne k \\ j \ne i \\ k \ne i}}^{3} j,k
		\quad\quad\quad\quad
		i
		\xleftarrow{lg}
		\sum_{\substack{j=1 \\ j \ne i}}^{3} j,j
		\quad\quad\quad\quad
		i
		\xleftarrow{lg}
		\sum_{\substack{j=1 \\ j \ne i}}^{3} [i,j + j,i]
		\quad\quad\quad\quad
		i
		\xleftarrow{lg}
		i,i
	\end{equation}
	Adding together the right hand side of the resulting left grouped map labels 
	and tidying gives
	\begin{equation}
		i
		\xleftarrow{lg}
		\sum_{j,k = 1}^{3} j,k
	\end{equation}
	Finally, there is nothing to unroll on its left hand side either,
	hence we obtain the map label
	\begin{equation}
		i
		\leftarrow
			\sum_{j,k = 1}^{3} j,k
	\end{equation}
	which we can convert into the transformation map label
	\begin{equation}
		{D_{\pi_4}(T)}_i
		\leftarrow
		\sum_{j,k = 1}^{3} T_{j,k}
	\end{equation}
	for $i \in [3]$.
	This is (\ref{tensormaplabelex4}) in the main paper.
\end{example}


\begin{figure*}[tb!]
	\begin{tcolorbox}[colback=terracotta!10, colframe=terracotta!40, coltitle=black, 
		title={\bfseries 
		Subprocedure III: How to Convert a Fully Labelled Diagram into a Grouped Map Label.
		},
		fonttitle=\bfseries
		]
		\textbf{Input:} A fully labelled diagram of a $(k,l)$--bipartition diagram $d_{\pi}$. \\[0.4em]
		Convert the fully labelled diagram into a grouped map label as follows:
	\begin{enumerate}
		\item 
			Create the arrow $\xleftarrow{g}$. 
		\item
			On its left, write the tuple that is formed from 
			propagating the labels of the fully labelled diagram
			to the output wires.
			In particular, distinct labels that appear on the left hand side of the arrow represent distinct values.
		\item
			On its right, write the tuple that is formed from
			propagating the labels of the fully labelled diagram
			to the input wires,
			summing over any labels that do not appear on the left hand side,
			and not allowing any summation labels to equal each other nor 
			any label that appears on the left hand side.
	\end{enumerate}
		\textbf{Output:} A grouped map label.
	\end{tcolorbox}
  	\label{subprocedureIII}
\end{figure*}



\begin{figure*}[tb!]
	\begin{tcolorbox}[colback=melon!10, colframe=melon!40, coltitle=black, 
		title={\bfseries 
		Subprocedure IV: 
		How to Convert the Grouped Map Label for a Fully Labelled Diagram into a Left Grouped Map Label.
		},
		fonttitle=\bfseries
		]
		\textbf{Input:} A grouped map label for a fully labelled diagram of a
		$(k,l)$--bipartition diagram $d_{\pi}$.
	\begin{enumerate}
		\item Consider the right hand side of the grouped map label. 
			The summation indices correspond to free indices in Subprocedure II, 
			and all other indices correspond to fixed indices in Subprocedure II.
			(Each fixed index should also appear on the left hand side of the grouped map label.)
		\item
			Replace the tuple of indices by a sum of all permutations of the tuple, modulo isomorphism
			under a relabelling of the free indices only.
			Also replace $\xleftarrow{g}$ with $\xleftarrow{lg}$.
	\end{enumerate}
		\textbf{Output:} A left grouped map label for a fully labelled diagram of a
		$(k,l)$--bipartition diagram $d_{\pi}$.
	\end{tcolorbox}
  	\label{subprocedureIV}
\end{figure*}
\begin{figure*}[tb!]
	\begin{tcolorbox}[colback=dustyblue!10, colframe=dustyblue!40, coltitle=black, 
		title={\bfseries 
		Subprocedure V: 
		How to Unroll the Left Hand Side of a Left Grouped Map Label.
		},
		fonttitle=\bfseries
		]
		\textbf{Input:} A left grouped map label. \\[0.3em]
		Create a set of map labels for the left grouped map label by unrolling the left hand side as follows:
	\begin{enumerate}
		\item 
			Calculate all permutations of the left hand side of the left grouped map label
			modulo isomorphism: any two permutations are said to be equivalent if the first can be bijectively relabelled
			to become the second, using the multiset of letters of letters formed from the left hand side of the left grouped map label.
		\item
			For each resulting permutation, create a map label by letting it be the left hand side of the arrow $\leftarrow$.
			To its right is exactly the right hand side of the left grouped map label, unchanged.
	\end{enumerate}
		\textbf{Output:} A set of map labels corresponding to the left grouped map label.
	\end{tcolorbox}
  	\label{subprocedureV}
\end{figure*}


\section{Implementation Details}

We describe the architecture and training details that were used for each of the two tasks.
In implementing\footnote{The code is available at \url{https://github.com/epearcecrump/symmetrictensors}.}
the linear permutation equivariant functions, we used batch vectorised
operations for each of the unrolled basis transformations that came from the corresponding transformation map labels.

\textbf{$S_{12}$-Invariant Task:}
We randomly generated a synthetic data set consisting of $5000$ symmetric tensors,
split into $90\%$ training and $10\%$ test.
We trained two models which had the following specifications:
\begin{itemize}
	\item SymmPermEquiv, the $S_{12}$-invariant model 
		on symmetric tensors in $(\mathbb{R}^{12})^{\otimes 3}$,
		consisted of a single, trainable, linear permutation equivariant function
		from symmetric tensors in $(\mathbb{R}^{12})^{\otimes 3}$ to $\mathbb{R}$.
		The total number of parameters was $3$ and the training time was $2.26$ seconds.
	\item MLP consisted of a single standard linear layer from
		$\mathbb{R}^{12 \times 12 \times 12}$ to $\mathbb{R}$.
		The total number of parameters was $1728$ and the training time 
		was $1.46$ seconds.
\end{itemize}
Both models were optimised with stochastic gradient descent with 
a learning rate of $0.0001$.
We trained both models for $50$ epochs with a batch size of $50$.



\textbf{$S_{8}$-Equivariant Task:}
We randommly generated a synthetic data set consisting of $10000$ symmetric tensors,
split into $90\%$ training and $10\%$ test.
We trained three models which had the following specifications:
\begin{itemize}
	\item SymmPermEquiv, the $S_{8}$-equivariant model 
		on symmetric tensors in $(\mathbb{R}^{8})^{\otimes 3}$,
		consisted of a single, trainable, linear permutation equivariant function
		from symmetric tensors in $(\mathbb{R}^{8})^{\otimes 3}$ 
		to symmetric tensors in $\mathbb{R}^{8}$.
		The total number of parameters was $7$ and the training time was $12.01$ seconds.
	\item PermEquiv, the standard $S_{8}$-equivariant model 
		\citep{maron2019, pearcecrump, godfrey2023}, 
		consisted of a single, trainable, linear permutation equivariant function
		from $(\mathbb{R}^{8})^{\otimes 3}$ 
		to $\mathbb{R}^{8}$.
		The total number of parameters was $15$ and the training time was $12.17$ seconds.
	\item MLP consisted of a single standard linear layer from 
		$\mathbb{R}^{8 \times 8 \times 8}$ to $\mathbb{R}^{8}$.
		The total number of parameters was $4096$ and the training time 
		was $4.38$ seconds.
\end{itemize}
Both models were optimised with stochastic gradient descent with 
a learning rate of $0.0001$.
We trained both models for $50$ epochs with a batch size of $50$.
In Table \ref{tab:diag}, we used a batch size of $50$ to obtain the test mean squared error (MSE).
In Table \ref{tab:generalisation},
we generated a test set of $1000$ tensors, one for each of $n = 8, 16, 32$,
using the parameters of the model that was trained in the task itself for each of these models.

\textbf{Comparison between Map Label and Weight Matrix Implementations:}
We are grateful to the reviewers for the following suggestion:
they recommended that we should provide empirical evidence 
to justify the time and space requirements 
of our map label notation method compared with the standard implementation 
that uses unrolled weight matrices. 

For the $S_{12}$-invariant task, the training time of SymmPermEquiv using 
the map label implementation
was $2.26$ seconds, compared to $127.55$ seconds using a weight matrix implementation,
resulting in a $60\times$ speedup.

For the $S_8$-equivariant task, the training time of SymmPermEquiv using 
the map label implementation
was $12.01$ seconds, compared to $2451.45$ seconds using a weight matrix implementation,
resulting in a $200\times$ speedup.

These results demonstrate that 
what we have said in theory is true in practice, 
namely that the standard weight matrix approach becomes unfeasible 
as $n, k$ and $l$ become larger, 
owing to constraints on storing the weight matrices in memory,
particularly in training the networks.





\end{document}

%% file: mystyle.tikzstyles
\tikzstyle{node}=[fill=white, draw=black, shape=circle, minimum size=1mm, ultra thick]
\tikzstyle{red node}=[fill=white, draw=red, shape=circle, minimum size=1mm, ultra thick]
\tikzstyle{green node}=[fill=white, draw={rgb,255: red,0; green,138; blue,0}, shape=circle, minimum size=1mm, ultra thick]
\tikzstyle{small box}=[fill=white, draw=black, shape=rectangle, minimum height=0.5cm, minimum width=0.5cm, ultra thick]
\tikzstyle{grey box}=[fill=gray!40, draw=black, shape=rectangle, minimum height=0.5cm, minimum width=0.5cm, ultra thick]
\tikzstyle{red box}=[fill=red!40, draw=black, shape=rectangle, minimum height=0.5cm, minimum width=0.5cm, ultra thick]
\tikzstyle{weyl}=[fill=white, draw={rgb,255: red,0; green,0; blue,109}, shape=rectangle, minimum height=0.5cm, minimum width=0.5cm]
\tikzstyle{filled_node}=[fill=black, draw=black, fill opacity=1, shape=circle, minimum size=1mm, ultra thick]
\tikzstyle{large box}=[fill=white, draw=black, shape=rectangle, minimum height=0.5cm, minimum width=1cm, ultra thick]
\tikzstyle{blue node}=[fill=white, draw=blue, shape=circle, minimum size=1mm, ultra thick]
\tikzstyle{filled_green_node}=[fill={rgb,255: red,0; green,138; blue,0}, draw={rgb,255: red,0; green,138; blue,0}, fill opacity=1, shape=circle, minimum size=1mm, ultra thick]

\tikzstyle{pattern_node}=[
    fill=black,
    pattern={Hatch[angle=45,distance=1.05mm, line width=0.2mm]},
    draw=black,
    text=white,
    shape=circle,
    minimum size=1mm,
    ultra thick,
    inner sep=5pt,
    rounded corners
]

\tikzstyle{thick}=[-, ultra thick]
\tikzstyle{blue_thick}=[-, ultra thick, draw=blue]
\tikzstyle{dashes}=[-, dashed, draw={rgb,255: red,191; green,191; blue,191}, dash pattern=on 2mm off 1mm, fill={rgb,255: red,244; green,228; blue,0}]
\tikzstyle{thick_arrow}=[ultra thick, ->]
\tikzstyle{dash_1}=[-, dashed]
\tikzstyle{dash_2}=[-, dashed, fill={rgb,255: red,246; green,235; blue,255}]
\tikzstyle{dash_3}=[-, dashed, fill={rgb,255: red,227; green,255; blue,209}]
\tikzstyle{dash_4}=[-, dashed, fill={rgb,255: red,255; green,209; blue,153}]
\tikzstyle{red_thick}=[-, ultra thick, draw=red]
\tikzstyle{dash_5}=[-, dashed, fill={rgb,255: red,225; green,255; blue,254}]
\tikzstyle{dash_6}=[-, dashed, fill={rgb,255: red,251; green,245; blue,160}]
\tikzstyle{jellyfish}=[-, ultra thick, fill={rgb,255: red,34; green,48; blue,255}]
\tikzstyle{arrow}=[thick, ->]
\tikzstyle{dashed_arrow}=[dashed, ->]
\tikzstyle{dashed_thick_arrow}=[ultra thick, dashed, ->]
\tikzstyle{red_thick_arrow}=[ultra thick, ->, draw=red]
\tikzstyle{thick_blue_double_arrow}=[ultra thick, <->, draw=blue]
\tikzstyle{thick_green_double_arrow}=[ultra thick, <->, draw={rgb,255: red,0; green,138; blue,0}]

%% file: example_paper.bbl
\begin{thebibliography}{95}
\providecommand{\natexlab}[1]{#1}
\providecommand{\url}[1]{\texttt{#1}}
\expandafter\ifx\csname urlstyle\endcsname\relax
  \providecommand{\doi}[1]{doi: #1}\else
  \providecommand{\doi}{doi: \begingroup \urlstyle{rm}\Url}\fi

\bibitem[Anandkumar et~al.(2014)Anandkumar, Ge, Hsu, Kakade, and
  Telgarsky]{anandkumar2014}
Anandkumar, A., Ge, R., Hsu, D.~J., Kakade, S.~M., and Telgarsky, M.
\newblock {Tensor Decompositions for Learning Latent Variable Models}.
\newblock \emph{Journal of Machine Learning Research}, 15:\penalty0 2773--2832,
  2014.

\bibitem[Andrews(1977)]{andrews1977}
Andrews, G.~E.
\newblock {An extension of Carlitz’s bipartition identity}.
\newblock \emph{Proceedings of the American Mathematical Society}, 63\penalty0
  (1):\penalty0 180--184, 1977.

\bibitem[Andrews(1998)]{andrews1998}
Andrews, G.~E.
\newblock \emph{{The Theory of Partitions}}, volume~2.
\newblock Cambridge University Press, 1998.

\bibitem[Andrews \& Paule(2012)Andrews and Paule]{andrews2012}
Andrews, G.~E. and Paule, P.
\newblock {MacMahon’s Dream}.
\newblock In \emph{Partitions, q-Series, and Modular Forms}, volume~23 of
  \emph{Developments in Mathematics}, pp.\  1--12. Springer, New York, NY,
  2012.

\bibitem[Auluck(1953)]{auluck1953}
Auluck, F.~C.
\newblock {On Partitions of Bipartite Numbers}.
\newblock \emph{Mathematical Proceedings of the Cambridge Philosophical
  Society}, 49\penalty0 (1):\penalty0 72--83, 1953.

\bibitem[Bast{\'\i}as et~al.(2024)Bast{\'\i}as, Martin, and
  Ryom-Hansen]{bastias2024}
Bast{\'\i}as, K.~O., Martin, P., and Ryom-Hansen, S.
\newblock {On the spherical partition algebra}, 2024.
\newblock \texttt{arXiv:2402.01890}.

\bibitem[Batatia et~al.(2022)Batatia, Kovacs, Simm, Ortner, and
  Csanyi]{batatia2022}
Batatia, I., Kovacs, D.~P., Simm, G., Ortner, C., and Csanyi, G.
\newblock {MACE: Higher Order Equivariant Message Passing Neural Networks for
  Fast and Accurate Force Fields}.
\newblock In \emph{Advances in Neural Information Processing Systems},
  volume~35, pp.\  11423--11436, 2022.

\bibitem[Batzner et~al.(2022)Batzner, Musaelian, Lixin~Sun, Mailoa, Kornbluth,
  Molinari, Smidt, and Kozinsky]{batzner2022}
Batzner, S., Musaelian, A., Lixin~Sun, M.~G., Mailoa, J.~P., Kornbluth, M.,
  Molinari, N., Smidt, T.~E., and Kozinsky, B.
\newblock E(3)-equivariant graph neural networks for data-efficient and
  accurate interatomic potentials.
\newblock \emph{Nature Communications}, 13:\penalty0 2453, 2022.

\bibitem[Bekkers et~al.(2018)Bekkers, Lafarge, Veta, Eppenhof, Pluim, and
  Duits]{bekkers2018}
Bekkers, E.~J., Lafarge, M.~W., Veta, M., Eppenhof, K. A.~J., Pluim, J. P.~W.,
  and Duits, R.
\newblock {Roto-Translation Covariant Convolutional Networks for Medical Image
  Analysis}.
\newblock In \emph{Medical Image Computing and Computer Assisted
  Intervention–MICCAI 2018: 21st International Conference, Granada, Spain,
  September 16-20, 2018, Proceedings, Part I}, pp.\  440--448. Springer
  International Publishing, 2018.

\bibitem[Blum-Smith et~al.(2024)Blum-Smith, Huang, Cuturi, and
  Villar]{blumsmith2024}
Blum-Smith, B., Huang, N., Cuturi, M., and Villar, S.
\newblock Learning functions on symmetric matrices and point clouds via
  lightweight invariant features, 2024.
\newblock \texttt{arXiv:2405.08097}.

\bibitem[Bogatskiy et~al.(2020)Bogatskiy, Anderson, Offermann, Roussi, Miller,
  and Kondor]{bogatskiy2020}
Bogatskiy, A., Anderson, B., Offermann, J., Roussi, M., Miller, D., and Kondor,
  R.
\newblock {Lorentz Group Equivariant Neural Network for Particle Physics}.
\newblock In \emph{International Conference on Machine Learning}, pp.\
  992--1002. PMLR, 2020.

\bibitem[Carlitz(1963)]{carlitz1963}
Carlitz, L.
\newblock {A Problem in Partitions}.
\newblock \emph{Duke Mathematical Journal}, 30\penalty0 (2):\penalty0 203--213,
  1963.

\bibitem[Carlitz \& Roselle(1966)Carlitz and Roselle]{carlitz1966}
Carlitz, L. and Roselle, D.~P.
\newblock {Restricted Bipartite Partitions}.
\newblock \emph{Pacific Journal of Mathematics}, 19\penalty0 (2):\penalty0
  221--228, 1966.

\bibitem[Chatzipantazis et~al.(2023)Chatzipantazis, Pertigkiozoglou, Dobriban,
  and Daniilidis]{chatzipantazis2023}
Chatzipantazis, E., Pertigkiozoglou, S., Dobriban, E., and Daniilidis, K.
\newblock {SE(3)-Equivariant Attention Networks for Shape Reconstruction in
  Function Space}.
\newblock In \emph{The Eleventh International Conference on Learning
  Representations}, 2023.

\bibitem[Cohen \& Welling(2016)Cohen and Welling]{cohenc16}
Cohen, T. and Welling, M.
\newblock {Group Equivariant Convolutional Networks}.
\newblock In \emph{Proceedings of The 33rd International Conference on Machine
  Learning}, volume~48 of \emph{Proceedings of Machine Learning Research}, pp.\
   2990--2999, New York, USA, 20--22 Jun 2016. PMLR.

\bibitem[Cohen \& Welling(2017)Cohen and Welling]{cohen2017steerable}
Cohen, T. and Welling, M.
\newblock Steerable {CNN}s.
\newblock In \emph{International Conference on Learning Representations}, 2017.

\bibitem[Cohen et~al.(2018)Cohen, Geiger, Köhler, and Welling]{cohen2018}
Cohen, T., Geiger, M., Köhler, J., and Welling, M.
\newblock Spherical {CNN}s.
\newblock In \emph{International Conference on Learning Representations}, 2018.

\bibitem[Deng et~al.(2021)Deng, Litany, Duan, Poulenard, Tagliasacchi, and
  Guibas]{deng2021}
Deng, C., Litany, O., Duan, Y., Poulenard, A., Tagliasacchi, A., and Guibas,
  L.~J.
\newblock {Vector Neurons: A General Framework for SO(3)-Equivariant Networks}.
\newblock In \emph{Proceedings of the IEEE/CVF International Conference on
  Computer Vision}, pp.\  12200--12209, 2021.

\bibitem[Esteves et~al.(2018)Esteves, Allen-Blanchette, Makadia, and
  Daniilidis]{esteves2018}
Esteves, C., Allen-Blanchette, C., Makadia, A., and Daniilidis, K.
\newblock {Learning SO(3) Equivariant Representations with Spherical CNNs}.
\newblock In \emph{Proceedings of the European Conference on Computer Vision
  (ECCV)}, pp.\  52--68, 2018.

\bibitem[Esteves et~al.(2019)Esteves, Xu, Allen-Blanchette, and
  Daniilidis]{esteves2019}
Esteves, C., Xu, Y., Allen-Blanchette, C., and Daniilidis, K.
\newblock {Equivariant Multi-View Networks}.
\newblock In \emph{Proceedings of the IEEE/CVF International Conference on
  Computer Vision (ICCV)}, October 2019.

\bibitem[Finzi et~al.(2020)Finzi, Stanton, Izmailov, and Wilson]{finzi20a}
Finzi, M., Stanton, S., Izmailov, P., and Wilson, A.~G.
\newblock {Generalizing Convolutional Neural Networks for Equivariance to Lie
  Groups on Arbitrary Continuous Data}.
\newblock In \emph{Proceedings of the 37th International Conference on Machine
  Learning}, volume 119 of \emph{Proceedings of Machine Learning Research},
  pp.\  3165--3176. PMLR, 13--18 Jul 2020.

\bibitem[Finzi et~al.(2021)Finzi, Welling, and Wilson]{finzi2021}
Finzi, M., Welling, M., and Wilson, A.~G.
\newblock A {P}ractical {M}ethod for {C}onstructing {E}quivariant {M}ultilayer
  {P}erceptrons for {A}rbitrary {M}atrix {G}roups.
\newblock In \emph{Proceedings of the 38th International Conference on Machine
  Learning}, volume 139 of \emph{Proceedings of Machine Learning Research},
  pp.\  3318--3328. PMLR, 18--24 Jul 2021.

\bibitem[Fuchs et~al.(2020)Fuchs, Worrall, Fischer, and Welling]{fuchs2020}
Fuchs, F.~B., Worrall, D.~E., Fischer, V., and Welling, M.
\newblock {SE(3)-transformers: 3D roto-translation equivariant attention
  networks}.
\newblock In \emph{Proceedings of the 34th International Conference on Neural
  Information Processing Systems}, 2020.

\bibitem[Gao et~al.(2022)Gao, Du, Li, and Lin]{gao2022}
Gao, L., Du, Y., Li, H., and Lin, G.
\newblock {RotEqNet: Rotation-equivariant network for fluid systems with
  symmetric high-order tensors}.
\newblock \emph{Journal of Computational Physics}, 461:\penalty0 111205, 2022.

\bibitem[Garanger et~al.(2024)Garanger, Kraus, and Rimoli]{garanger2024}
Garanger, K., Kraus, J., and Rimoli, J.~J.
\newblock Symmetry-enforcing neural networks with applications to constitutive
  modeling.
\newblock \emph{Extreme Mechanics Letters}, 71:\penalty0 102188, 2024.

\bibitem[Georgii et~al.(2011)Georgii, Tsuda, and Schölkopf]{georgii2011}
Georgii, E., Tsuda, K., and Schölkopf, B.
\newblock Multi-way set enumeration in weight tensors.
\newblock \emph{Machine Learning}, 82:\penalty0 123--155, 2011.

\bibitem[Ghoshdastidar \& Dukkipati(2017)Ghoshdastidar and
  Dukkipati]{ghoshdastidar2017}
Ghoshdastidar, D. and Dukkipati, A.
\newblock {Uniform Hypergraph Partitioning: Provable Tensor Methods and
  Sampling Techniques}.
\newblock \emph{Journal of Machine Learning Research}, 18\penalty0
  (50):\penalty0 1--41, 2017.

\bibitem[Godfrey et~al.(2023)Godfrey, Rawson, Brown, and Kvinge]{godfrey2023}
Godfrey, C., Rawson, M.~G., Brown, D., and Kvinge, H.
\newblock {F}ast computation of permutation equivariant layers with the
  partition algebra.
\newblock In \emph{ICLR 2023 Workshop on Physics for Machine Learning}, 2023.

\bibitem[Gong et~al.(2022)Gong, Meng, Zhang, Qu, Li, Qian, Du, Ma, and
  Liu]{gong2022}
Gong, S., Meng, Q., Zhang, J., Qu, H., Li, C., Qian, S., Du, W., Ma, Z.-M., and
  Liu, T.-Y.
\newblock {An Efficient Lorentz Equivariant Graph Neural Network for Jet
  Tagging}.
\newblock \emph{Journal of High Energy Physics}, 2022\penalty0 (7):\penalty0
  1--22, 2022.

\bibitem[Gordon et~al.(2020)Gordon, Lopez-Paz, Baroni, and
  Bouchacourt]{gordon2020}
Gordon, J., Lopez-Paz, D., Baroni, M., and Bouchacourt, D.
\newblock {Permutation Equivariant Models for Compositional Generalization in
  Language}.
\newblock In \emph{International Conference on Learning Representations}, 2020.

\bibitem[Goulart et~al.(2022)Goulart, Couillet, and Comon]{goulart2022}
Goulart, J. H.~M., Couillet, R., and Comon, P.
\newblock {A Random Matrix Perspective on Random Tensors}.
\newblock \emph{Journal of Machine Learning Research}, 23\penalty0
  (264):\penalty0 1--36, 2022.

\bibitem[Guha \& Guhaniyogi(2020)Guha and Guhaniyogi]{guha2020}
Guha, S. and Guhaniyogi, R.
\newblock {Bayesian Generalized Sparse Symmetric Tensor-on-Vector Regression}.
\newblock \emph{Technometrics}, 63\penalty0 (2):\penalty0 160--170, 2020.

\bibitem[Hartford et~al.(2018)Hartford, Graham, Leyton-Brown, and
  Ravanbakhsh]{hartford2018}
Hartford, J.~S., Graham, D.~R., Leyton-Brown, K., and Ravanbakhsh, S.
\newblock {Deep {M}odels of {I}nteractions {A}cross {S}ets}.
\newblock In \emph{{Proceedings of the 35th International Conference on Machine
  Learning}}, pp.\  1914--1923. {PMLR}, 2018.

\bibitem[Heilman et~al.(2024)Heilman, Schlesinger, and Yan]{heilman2024}
Heilman, A., Schlesinger, C., and Yan, Q.
\newblock {Equivariant Graph Neural Networks for Prediction of Tensor Material
  Properties of Crystals}, 2024.
\newblock \texttt{arXiv:2406.03563}.

\bibitem[Hoogeboom et~al.(2022)Hoogeboom, Satorras, Vignac, and
  Welling]{hoogeboom22a}
Hoogeboom, E., Satorras, V.~G., Vignac, C., and Welling, M.
\newblock {Equivariant Diffusion for Molecule Generation in 3{D}}.
\newblock In \emph{Proceedings of the 39th International Conference on Machine
  Learning}, volume 162 of \emph{Proceedings of Machine Learning Research},
  pp.\  8867--8887. PMLR, 17--23 Jul 2022.

\bibitem[Jaffe et~al.(2018)Jaffe, Weiss, Nadler, Carmi, and Kluger]{jaffe2018}
Jaffe, A., Weiss, R., Nadler, B., Carmi, S., and Kluger, Y.
\newblock {Learning Binary Latent Variable Models: A Tensor Eigenpair
  Approach}.
\newblock In \emph{International Conference on Machine Learning}, pp.\
  2196--2205. PMLR, 2018.

\bibitem[Kaba et~al.(2023)Kaba, Mondal, Zhang, Bengio, and
  Ravanbakhsh]{kaba23a}
Kaba, S.-O., Mondal, A.~K., Zhang, Y., Bengio, Y., and Ravanbakhsh, S.
\newblock {Equivariance with Learned Canonicalization Functions}.
\newblock In \emph{Proceedings of the 40th International Conference on Machine
  Learning}, volume 202 of \emph{Proceedings of Machine Learning Research},
  pp.\  15546--15566. PMLR, 23--29 Jul 2023.

\bibitem[Keriven \& Peyr\'{e}(2019)Keriven and Peyr\'{e}]{keriven2019}
Keriven, N. and Peyr\'{e}, G.
\newblock {Universal Invariant and Equivariant Graph Neural Networks}.
\newblock In \emph{Advances in Neural Information Processing Systems},
  volume~32, 2019.

\bibitem[Khouja(2022)]{khouja2022}
Khouja, R.
\newblock \emph{Optimization algorithms for the tensor rank approximation
  problem: application to clustering in machine learning}.
\newblock PhD thesis, Université Côte d'Azur; Université Libanaise, 2022.

\bibitem[Kim \& Hahn(1997)Kim and Hahn]{kim1997}
Kim, J.~K. and Hahn, S.~G.
\newblock Partitions of bipartite numbers.
\newblock \emph{Graphs and Combinatorics}, 13:\penalty0 73--78, 1997.

\bibitem[Kondor \& Trivedi(2018)Kondor and Trivedi]{kondor18a}
Kondor, R. and Trivedi, S.
\newblock {On the Generalization of Equivariance and Convolution in Neural
  Networks to the Action of Compact Groups}.
\newblock In \emph{Proceedings of the 35th International Conference on Machine
  Learning}, volume~80 of \emph{Proceedings of Machine Learning Research}, pp.\
   2747--2755. PMLR, 10--15 Jul 2018.

\bibitem[Kondor et~al.(2018)Kondor, Lin, and Trivedi]{clebschgordan}
Kondor, R., Lin, Z., and Trivedi, S.
\newblock Clebsch\textendash {G}ordan {N}ets: a {F}ully {F}ourier {S}pace
  {S}pherical {C}onvolutional {N}eural {N}etwork.
\newblock In \emph{Advances in Neural Information Processing Systems},
  volume~31, 2018.

\bibitem[Kunisky et~al.(2024)Kunisky, Moore, and Wein]{kunisky2024}
Kunisky, D., Moore, C., and Wein, A.~S.
\newblock Tensor cumulants for statistical inference on invariant
  distributions, 2024.
\newblock \texttt{arXiv:2404.18735}.

\bibitem[Lafarge et~al.(2021)Lafarge, Bekkers, Pluim, Duits, and
  Veta]{lafarge2021}
Lafarge, M.~W., Bekkers, E.~J., Pluim, J. P.~W., Duits, R., and Veta, M.
\newblock {Roto-translation equivariant convolutional networks: Application to
  histopathology image analysis}.
\newblock \emph{Medical Image Analysis}, 68:\penalty0 101849, 2021.

\bibitem[Landman et~al.(1992)Landman, Brown, and Portier]{landman1992}
Landman, B.~M., Brown, E.~A., and Portier, F.~J.
\newblock {Partitions of bi-partite numbers into at most j parts}.
\newblock \emph{Graphs and Combinatorics}, 8:\penalty0 65--73, 1992.

\bibitem[Liao \& Smidt(2023)Liao and Smidt]{liao2023}
Liao, Y.-L. and Smidt, T.
\newblock {Equiformer: Equivariant Graph Attention Transformer for 3D Atomistic
  Graphs}.
\newblock In \emph{The Eleventh International Conference on Learning
  Representations}, 2023.

\bibitem[Lou \& Ganose(2024)Lou and Ganose]{lou2024}
Lou, Y. and Ganose, A.~M.
\newblock Discovery of highly anisotropic dielectric crystals with equivariant
  graph neural networks, 2024.
\newblock \texttt{arXiv:2405.07915}.

\bibitem[MacMahon(1893)]{macmahon1893}
MacMahon, P.~A.
\newblock {XVII. Memoir on the Theory of the Compositions of Numbers}.
\newblock \emph{Philosophical Transactions of the Royal Society of London.
  (A.)}, 184:\penalty0 835--901, 1893.

\bibitem[MacMahon(1896)]{macmahon1896}
MacMahon, P.~A.
\newblock {XVI. Memoir on the Theory of the Partition of Numbers. Part I}.
\newblock \emph{Philosophical Transactions of the Royal Society of London.
  Series A, Containing Papers of a Mathematical or Physical Character},
  187:\penalty0 619--673, 1896.

\bibitem[MacMahon(1899)]{macmahon1899}
MacMahon, P.~A.
\newblock {Memoir on the Theory of the Partitions of Numbers. Part II}.
\newblock \emph{Philosophical Transactions of the Royal Society of London.
  Series A, Containing Papers of a Mathematical or Physical Character},
  192:\penalty0 351--401, 1899.

\bibitem[MacMahon(1900)]{macmahon1900}
MacMahon, P.~A.
\newblock {Combinatorial Analysis. The Foundations of a New Theory}.
\newblock \emph{Philosophical Transactions of the Royal Society of London.
  Series A, Containing Papers of a Mathematical or Physical Character},
  194:\penalty0 361--386, 1900.

\bibitem[MacMahon(1906)]{macmahon1906}
MacMahon, P.~A.
\newblock {Memoir on the Theory of the Partitions of Numbers. Part III}.
\newblock \emph{Philosophical Transactions of the Royal Society of London.
  Series A, Containing Papers of a Mathematical or Physical Character},
  205:\penalty0 37--59, 1906.

\bibitem[MacMahon(1908)]{macmahon1908}
MacMahon, P.~A.
\newblock {II. Second Memoir on the Compositions of Numbers}.
\newblock \emph{Philosophical Transactions of the Royal Society of London.
  Series A, Containing Papers of a Mathematical or Physical Character},
  207\penalty0 (413-426):\penalty0 65--134, 1908.

\bibitem[MacMahon(1912)]{macmahon1912}
MacMahon, P.~A.
\newblock {III. Memoir on the Theory of the Partitions of Numbers. Part V.
  Partitions in Two-Dimensional Space}.
\newblock \emph{Philosophical Transactions of the Royal Society of London.
  Series A, Containing Papers of a Mathematical or Physical Character},
  211\penalty0 (471-483):\penalty0 75--110, 1912.

\bibitem[MacMahon(1918)]{macmahon1918}
MacMahon, P.~A.
\newblock {Seventh Memoir on the Partition of Numbers. A Detailed Study of the
  Enumeration of the Partitions of Multipartite Numbers}.
\newblock \emph{Philosophical Transactions of the Royal Society of London.
  Series A, Containing Papers of a Mathematical or Physical Character},
  217:\penalty0 81--113, 1918.

\bibitem[Marcos et~al.(2017)Marcos, Volpi, Komodakis, and Tuia]{marcos2017}
Marcos, D., Volpi, M., Komodakis, N., and Tuia, D.
\newblock {Rotation Equivariant Vector Field Networks}.
\newblock In \emph{Proceedings of the IEEE International Conference on Computer
  Vision}, pp.\  5048--5057, 2017.

\bibitem[Maron et~al.(2019{\natexlab{a}})Maron, Ben-Hamu, Shamir, and
  Lipman]{maron2019}
Maron, H., Ben-Hamu, H., Shamir, N., and Lipman, Y.
\newblock Invariant and {E}quivariant {G}raph {N}etworks.
\newblock In \emph{International Conference on Learning Representations},
  2019{\natexlab{a}}.

\bibitem[Maron et~al.(2019{\natexlab{b}})Maron, Fetaya, Segol, and
  Lipman]{maron19a}
Maron, H., Fetaya, E., Segol, N., and Lipman, Y.
\newblock On the {U}niversality of {I}nvariant {N}etworks.
\newblock In \emph{Proceedings of the 36th International Conference on Machine
  Learning}, volume~97 of \emph{Proceedings of Machine Learning Research}, pp.\
   4363--4371. PMLR, 09--15 Jun 2019{\natexlab{b}}.

\bibitem[Maron et~al.(2020)Maron, Litany, Chechik, and Fetaya]{maron20a}
Maron, H., Litany, O., Chechik, G., and Fetaya, E.
\newblock On {L}earning {S}ets of {S}ymmetric {E}lements.
\newblock In \emph{Proceedings of the 37th International Conference on Machine
  Learning}, volume 119 of \emph{Proceedings of Machine Learning Research},
  pp.\  6734--6744. PMLR, 13--18 Jul 2020.

\bibitem[Mathews(1896)]{mathews1896}
Mathews, G.~B.
\newblock On the partition of numbers.
\newblock \emph{Proceedings of the London Mathematical Society}, s1-28\penalty0
  (1):\penalty0 486--490, 1896.

\bibitem[McCullagh(2018)]{mccullagh2018}
McCullagh, P.
\newblock \emph{{Tensor Methods in Statistics}}.
\newblock Courier Dover Publications, 2018.

\bibitem[Miller et~al.(2020)Miller, Geiger, Smidt, and No{\'e}]{miller2020}
Miller, B.~K., Geiger, M., Smidt, T.~E., and No{\'e}, F.
\newblock {Relevance of rotationally equivariant convolutions for predicting
  molecular properties}.
\newblock \emph{arXiv preprint arXiv:2008.08461}, 2020.

\bibitem[Morris et~al.(2022)Morris, Rattan, Kiefer, and Ravanbakhsh]{morris22a}
Morris, C., Rattan, G., Kiefer, S., and Ravanbakhsh, S.
\newblock {S}peq{N}ets: {S}parsity-aware {P}ermutation-equivariant graph
  networks.
\newblock In \emph{Proceedings of the 39th International Conference on Machine
  Learning}, volume 162, pp.\  16017--16042, 17--23 Jul 2022.

\bibitem[Nanda(1957)]{nanda1957}
Nanda, V.~S.
\newblock {Bipartite partitions}.
\newblock \emph{Mathematical Proceedings of the Cambridge Philosophical
  Society}, 53\penalty0 (2):\penalty0 273--277, 1957.

\bibitem[Pan \& Kondor(2022)Pan and Kondor]{pan22}
Pan, H. and Kondor, R.
\newblock Permutation {E}quivariant {L}ayers for {H}igher {O}rder
  {I}nteractions.
\newblock In \emph{Proceedings of The 25th International Conference on
  Artificial Intelligence and Statistics}, volume 151 of \emph{Proceedings of
  Machine Learning Research}, pp.\  5987--6001. PMLR, 28--30 Mar 2022.

\bibitem[Pearce-Crump(2023)]{pearcecrumpB}
Pearce-Crump, E.
\newblock {B}rauer's {G}roup {E}quivariant {N}eural {N}etworks.
\newblock In \emph{Proceedings of the 40th International Conference on Machine
  Learning}, volume 202 of \emph{Proceedings of Machine Learning Research},
  pp.\  27461--27482. PMLR, 23--29 Jul 2023.

\bibitem[Pearce-Crump(2024)]{pearcecrump}
Pearce-Crump, E.
\newblock {Connecting Permutation Equivariant Neural Networks and Partition
  Diagrams}.
\newblock In \emph{Proceedings of the 27th European Conference on Artificial
  Intelligence}, volume 392, pp.\  1511--1518. IOS Press Ebooks, 2024.

\bibitem[Pearce-Crump \& Knottenbelt(2024)Pearce-Crump and
  Knottenbelt]{pearcecrumpG}
Pearce-Crump, E. and Knottenbelt, W.~J.
\newblock {Graph Automorphism Group Equivariant Neural Networks}.
\newblock In \emph{Proceedings of the 41st International Conference on Machine
  Learning}, volume 235 of \emph{Proceedings of Machine Learning Research},
  pp.\  40051--40077. PMLR, 21--27 Jul 2024.

\bibitem[Petrache \& Trivedi(2024)Petrache and Trivedi]{petrache2024}
Petrache, M. and Trivedi, S.
\newblock {Position Paper: Generalized grammar rules and structure-based
  generalization beyond classical equivariance for lexical tasks and
  transduction}, 2024.
\newblock \texttt{arXiv:2402.01629}.

\bibitem[Puny et~al.(2023)Puny, Lim, Kiani, Maron, and Lipman]{puny2023}
Puny, O., Lim, D., Kiani, B., Maron, H., and Lipman, Y.
\newblock {Equivariant Polynomials for Graph Neural Networks}.
\newblock In \emph{International Conference on Machine Learning}, pp.\
  28191--28222. PMLR, 2023.

\bibitem[Qi et~al.(2017)Qi, Su, Mo, and Guibas]{qi2017}
Qi, C.~R., Su, H., Mo, K., and Guibas, L.~J.
\newblock {PointNet: Deep Learning on Point Sets for 3D Classification and
  Segmentation}.
\newblock In \emph{{Proceedings of the IEEE Conference on Computer Vision and
  Pattern Recognition}}, pp.\  652--660, 2017.

\bibitem[Ravanbakhsh et~al.(2017)Ravanbakhsh, Schneider, and
  P{\'o}czos]{ravanbakhsh17a}
Ravanbakhsh, S., Schneider, J., and P{\'o}czos, B.
\newblock {Equivariance Through Parameter-Sharing}.
\newblock In \emph{Proceedings of the 34th International Conference on Machine
  Learning}, volume~70, pp.\  2892--2901, 06--11 Aug 2017.

\bibitem[Satorras et~al.(2021)Satorras, Hoogeboom, and Welling]{satorras21a}
Satorras, V.~G., Hoogeboom, E., and Welling, M.
\newblock {E(n) Equivariant Graph Neural Networks}.
\newblock In \emph{Proceedings of the 38th International Conference on Machine
  Learning}, volume 139 of \emph{Proceedings of Machine Learning Research},
  pp.\  9323--9332. PMLR, 18--24 Jul 2021.

\bibitem[Sch{\"u}tt et~al.(2021)Sch{\"u}tt, Unke, and Gastegger]{schutt21a}
Sch{\"u}tt, K., Unke, O., and Gastegger, M.
\newblock {Equivariant Message Passing for the Prediction of Tensorial
  Properties and Molecular Spectra}.
\newblock In \emph{Proceedings of the 38th International Conference on Machine
  Learning}, volume 139 of \emph{Proceedings of Machine Learning Research},
  pp.\  9377--9388. PMLR, 18--24 Jul 2021.

\bibitem[Segol \& Lipman(2020)Segol and Lipman]{segol2019}
Segol, N. and Lipman, Y.
\newblock {On universal equivariant set networks}.
\newblock In \emph{International Conference on Learning Representations}, 2020.

\bibitem[Shashua et~al.(2005)Shashua, Zass, and Hazan]{shashua2005}
Shashua, A., Zass, R., and Hazan, T.
\newblock {Multi-way Clustering Using Super-Symmetric Non-Negative Tensor
  Factorization}.
\newblock In \emph{European Conference on Computer Vision}, pp.\  595--608,
  2005.

\bibitem[Stanley(1997)]{stanley1997}
Stanley, R.~P.
\newblock \emph{Enumerative Combinatorics: Volume 1}.
\newblock Cambridge University Press, 1997.

\bibitem[Suk et~al.(2024)Suk, {de Haan}, Lippe, Brune, and Wolterink]{suk24}
Suk, J., {de Haan}, P., Lippe, P., Brune, C., and Wolterink, J.~M.
\newblock {Mesh neural networks for SE(3)-equivariant hemodynamics estimation
  on the artery wall}.
\newblock \emph{Computers in Biology and Medicine}, 173, 2024.

\bibitem[Sverdlov et~al.(2024)Sverdlov, Springer, and Dym]{sverdlov2024}
Sverdlov, Y., Springer, I., and Dym, N.
\newblock {Revisiting Multi-Permutation Equivariance through the Lens of
  Irreducible Representations}, 2024.
\newblock \texttt{arXiv:2410.06665}.

\bibitem[Thiede et~al.(2020)Thiede, Hy, and Kondor]{thiede2020}
Thiede, E.~H., Hy, T.~S., and Kondor, R.
\newblock The general theory of permutation equivarant neural networks and
  higher order graph variational encoders, 2020.
\newblock \texttt{arXiv:2004.03990}.

\bibitem[Thomas et~al.(2018)Thomas, Smidt, Kearnes, Yang, Li, Kohlhoff, and
  Riley]{thomas2018}
Thomas, N., Smidt, T., Kearnes, S., Yang, L., Li, L., Kohlhoff, K., and Riley,
  P.
\newblock Tensor field networks: {R}otation-and translation-equivariant neural
  networks for 3d point clouds, 2018.
\newblock \texttt{arXiv:1802.08219}.

\bibitem[van~der Pol et~al.(2020)van~der Pol, Worrall, van Hoof, Oliehoek, and
  Welling]{vanderpol}
van~der Pol, E., Worrall, D., van Hoof, H., Oliehoek, F., and Welling, M.
\newblock {MDP Homomorphic Networks: Group Symmetries in Reinforcement
  Learning}.
\newblock In \emph{Advances in Neural Information Processing Systems},
  volume~33, pp.\  4199--4210, 2020.

\bibitem[Villar et~al.(2021)Villar, Hogg, Storey-Fisher, Yao, and
  Blum-Smith]{villar2021scalars}
Villar, S., Hogg, D.~W., Storey-Fisher, K., Yao, W., and Blum-Smith, B.
\newblock Scalars are universal: {E}quivariant machine learning, structured
  like classical physics.
\newblock In \emph{Advances in Neural Information Processing Systems}, 2021.

\bibitem[Wang et~al.(2022{\natexlab{a}})Wang, Jia, Zhu, Walters, and
  Platt]{wang2022r}
Wang, D., Jia, M., Zhu, X., Walters, R., and Platt, R.
\newblock {On-Robot Learning with Equivariant Models}.
\newblock In \emph{6th Annual Conference on Robot Learning},
  2022{\natexlab{a}}.

\bibitem[Wang et~al.(2022{\natexlab{b}})Wang, Walters, and Platt]{wang2022}
Wang, D., Walters, R., and Platt, R.
\newblock {SO(2)-Equivariant Reinforcement Learning}.
\newblock In \emph{International Conference on Learning Representations},
  2022{\natexlab{b}}.

\bibitem[Wang et~al.(2022{\natexlab{c}})Wang, Walters, Zhu, and Platt]{wang22j}
Wang, D., Walters, R., Zhu, X., and Platt, R.
\newblock {Equivariant $Q$ Learning in Spatial Action Spaces}.
\newblock In \emph{Proceedings of the 5th Conference on Robot Learning}, volume
  164 of \emph{Proceedings of Machine Learning Research}, pp.\  1713--1723.
  PMLR, 08--11 Nov 2022{\natexlab{c}}.

\bibitem[Weiler \& Cesa(2019)Weiler and Cesa]{weiler2019}
Weiler, M. and Cesa, G.
\newblock {General E(2)-Equivariant Steerable CNNs}.
\newblock In \emph{Advances in Neural Information Processing Systems},
  volume~32, 2019.

\bibitem[Weiler et~al.(2018)Weiler, Geiger, Welling, Boomsma, and
  Cohen]{weiler2018}
Weiler, M., Geiger, M., Welling, M., Boomsma, W., and Cohen, T.
\newblock {3D Steerable CNNs: Learning Rotationally Equivariant Features in
  Volumetric Data}.
\newblock In \emph{Advances in Neural Information Processing Systems},
  volume~31, 2018.

\bibitem[Wen et~al.(2024)Wen, Horton, Munro, Huck, and Persson]{wen2024}
Wen, M., Horton, M.~K., Munro, J.~M., Huck, P., and Persson, K.~A.
\newblock {An equivariant graph neural network for the elasticity tensors of
  all seven crystal systems}.
\newblock \emph{Digital Discovery}, 3\penalty0 (5):\penalty0 869--882, 2024.

\bibitem[Worrall et~al.(2017)Worrall, Garbin, Turmukhambetov, and
  Brostow]{worrall2017}
Worrall, D.~E., Garbin, S.~J., Turmukhambetov, D., and Brostow, G.~J.
\newblock {Harmonic Networks: Deep Translation and Rotation Equivariance}.
\newblock In \emph{Proceedings of the IEEE Conference on Computer Vision and
  Pattern Recognition (CVPR)}, 2017.

\bibitem[Wright(1957)]{wright1957}
Wright, E.~M.
\newblock The number of partitions of a large bi-partite number.
\newblock \emph{Proceedings of the London Mathematical Society}, 3\penalty0
  (1):\penalty0 150--160, 1957.

\bibitem[Wright(1958)]{wright1958}
Wright, E.~M.
\newblock Partitions of large bipartites.
\newblock \emph{American Journal of Mathematics}, 80\penalty0 (3):\penalty0
  643--658, 1958.

\bibitem[Wright(1964)]{wright1964}
Wright, E.~M.
\newblock Partition of multipartite numbers into k parts.
\newblock \emph{Journal für die reine und angewandte Mathematik},
  1964\penalty0 (216):\penalty0 101--112, 1964.

\bibitem[Yarotsky(2022)]{yarotsky2022}
Yarotsky, D.
\newblock Universal approximations of invariant maps by neural networks.
\newblock \emph{Constructive Approximation}, 55\penalty0 (1):\penalty0
  407--474, 2022.

\bibitem[Zaheer et~al.(2017)Zaheer, Kottur, Ravanbakhsh, Poczos, Salakhutdinov,
  and Smola]{zaheer2017}
Zaheer, M., Kottur, S., Ravanbakhsh, S., Poczos, B., Salakhutdinov, R.~R., and
  Smola, A.~J.
\newblock Deep {S}ets.
\newblock In \emph{Advances in Neural Information Processing Systems},
  volume~30, 2017.

\end{thebibliography}
